\pdfoutput=1
\documentclass[a4paper,10pt]{article}

\usepackage[margin=1in]{geometry}
\usepackage[utf8]{inputenc}
\usepackage[T1]{fontenc}
\usepackage{amsmath}
\usepackage{amsfonts}
\usepackage[boxruled,algo2e]{algorithm2e}
\SetCommentSty{mycommfont}
\SetProcNameSty{texttt}
\usepackage[dvipsnames]{xcolor}
\usepackage{url}
\usepackage{multirow}
\usepackage{tikz}
\usepackage{appendix}
\usepackage{mathtools}
\newcommand{\eqcomment}[1]{%
  \text{\phantom{#1}} \tag*{#1}
}
\usepackage{url}

\usepackage{hyperref}
\usepackage{graphics}
\usepackage{subfig}

\definecolor{darkgreen}{rgb}{0,0.5,0}

\newcommand{\hght}{5.5cm}

\usepackage{amsthm}
\newtheorem{remark}{Remark}
\newtheorem{lemma}{Lemma}
\newtheorem{theorem}{Theorem}
\newtheorem{corollary}{Corollary}
\newtheorem{definition}{Definition}

\newtheorem{innerassumption}{Assumption}
\newenvironment{assumption}[1]
  {\renewcommand\theinnerassumption{#1}\innerassumption}
  {\endinnerassumption}
  
\newtheorem{innerhypothesis}{Hypothesis}
\newenvironment{hypothesis}[1]
  {\renewcommand\theinnerhypothesis{#1}\innerhypothesis}
  {\endinnerhypothesis}
  
\newtheorem{innermythm}{Theorem}
\newenvironment{mythm}[1]
  {\renewcommand\theinnermythm{#1}\innermythm}
  {\endinnermythm}

\def\bM{\bfM}
\def\bx{\bfx}
\def\bz{\bfz}
\def\order{\mathcal{O}}

\newcommand{\review}[1]{{\color{black} #1}}

\def\u{\linewidth/30} 
\def\hrh{2.5\u} 
\def\hrl{10\u} 
\def\vrh{10\u} 
\def\vrl{2.5\u} 

\def\eg{\textit{e.g.}}
\def\etal{\textit{et al.}}
\def\ie{\textit{i.e.}}

\newcommand{\bfX}{\mathbf{X}}
\newcommand{\bfW}{\mathbf{W}}

\newcommand{\bfM}{\mathbf{M}}

\newcommand{\bfz}{\mathbf{z}}
\newcommand{\bfx}{\mathbf{x}}
\newcommand{\bfr}{\mathbf{r}}
\newcommand{\bfe}{\mathbf{e}}
\newcommand{\bfy}{\mathbf{y}}

\newcommand{\bfI}{\mathbf{I}}
\newcommand{\bfU}{\mathbf{U}}

\newcommand{\bfV}{\mathbf{V}}
\newcommand{\bfv}{\mathbf{v}}

\newcommand{\sumk}{\sum_{k=1}^K}

\newcommand{\sumn}{\sum_{i=1}^n}

\newcommand{\dataset}{\mathcal{X}}

\newcommand{\neucl}[1]{\left\lVert#1\right\rVert_2}
\newcommand{\ipeucl}[2]{\left\langle #1,#2 \right\rangle_2}

\newcommand{\pmeas}{\mathcal{P}}
\newcommand{\gaussset}{\mathcal{G}}
\newcommand{\gaussmix}[1]{{\mathcal{G}_{#1}}}

\newcommand{\fracsqm}{\frac{1}{\sqrt{m}}}

\newcommand{\argmin}[1]{\underset{#1}{\text{argmin}}}

\newcommand{\freqs}{\Omega} 
\newcommand{\freq}{{\boldsymbol\omega}} 

\newcommand{\hyppar}{\Theta} 
\newcommand{\chrc}{\psi} 
\newcommand{\skop}{\mathcal{A}} 
\newcommand{\malpha}{{\boldsymbol\alpha}} 
\newcommand{\mmu}{{\boldsymbol\mu}} 
\newcommand{\mSigma}{{\boldsymbol\Sigma}} 
\newcommand{\msigma}{{\boldsymbol\sigma}} 
\newcommand{\mrho}{{\boldsymbol\varphi}} 
\newcommand{\mtheta}{{\boldsymbol\theta}} 
\newcommand{\thetaset}{\mathcal{T}}

\newcommand{\gmmcover}{\Gamma}
\newcommand{\pp}{P}
\newcommand{\PP}{\mathbb{P}}
\newcommand{\qq}{Q}
\newcommand{\dens}{p}
\newcommand{\cstdom}{\mathrm{A}}

\newcommand{\rkhs}{\mathcal{H}}
\newcommand{\ipH}[2]{\left\langle #1,#2 \right\rangle_\rkhs}

\newcommand{\freqdist}{\Lambda}

\newcommand{\model}{\Sigma}
\newcommand{\Xspace}{X}
\newcommand{\TIk}{\mathbf{K}}
\newcommand{\supp}{\text{\upshape supp}}
\newcommand{\secant}{\mathcal{S}}
\newcommand{\decod}{\Delta}
\newcommand{\enet}{\mathcal{N}}
\newcommand{\kernel}{\kappa}
\newcommand{\meas}{\nu} 
\newcommand{\mmap}{\varphi}
\newcommand{\sigmaker}{\sigma_\freqdist}
\newcommand{\sigkersmall}{a}

\newcommand{\imaginaryi}{\mathsf{i}}

\definecolor{light-gray}{gray}{0.3}

\newcommand{\pproj}{\pp_\text{proj}}

\newcommand{\distfun}[3]{\gamma_{#3}\left(#1,#2\right)}
\newcommand{\distfunsq}[3]{\gamma^2_{#3}\left(#1,#2\right)}
\newcommand{\distfuns}[1]{\gamma_{#1}}

\newcommand{\normK}[2]{\distfun{#1}{#2}{\freqdist}}
\newcommand{\normKs}{\distfuns{\freqdist}}
\newcommand{\normKsq}[2]{\distfunsq{#1}{#2}{\freqdist}}

\newcommand{\rad}[1]{\text{\upshape radius}\left(#1\right)}

\newcommand{\embd}{\phi}

\makeatletter
\newcommand{\nosemic}{\renewcommand{\@endalgocfline}{\relax}}
\newcommand{\dosemic}{\renewcommand{\@endalgocfline}{\algocf@endline}}
\makeatother

\begin{document}


\author{{
\sc Nicolas Keriven}$^*$,\\[2pt]
Universit\'e Rennes 1, INRIA Rennes-Bretagne Atlantique \\
Campus de Beaulieu, Rennes, France\\
$^*${{Corresponding author: Email: nicolas.keriven@inria.fr}}\\[2pt]
{\sc Anthony Bourrier}\\[2pt]
Gipsa-Lab, 11 rue des Mathématiques, St-Martin-d'H\'er\`es, France\\
{{Email: anthony.bourrier@gmail.com}}\\[6pt]
{\sc R\'emi Gribonval}\\[2pt]
INRIA Rennes-Bretagne Atlantique, Campus de Beaulieu, Rennes, France\\
{{Email: remi.gribonval@inria.fr}}\\[6pt]
{\sc and}\\[6pt]
{\sc Patrick P\'erez} \\[2pt]
Technicolor, 975 Avenue des Champs Blancs, Cesson S\'evign\'e, France\\
{{Email: patrick.perez@technicolor.fr}}}

\title{Sketching for Large-Scale Learning of Mixture Models}

\maketitle

\begin{abstract}
{Learning parameters from voluminous data can be prohibitive in terms of memory and computational requirements. We propose a ``compressive learning'' framework where we estimate model parameters from a \emph{sketch} of the training data. This sketch is a collection of generalized moments of the underlying probability distribution of the data. It can be computed in a single pass on the training set, and is easily computable on streams or distributed datasets. The proposed framework shares similarities with compressive sensing, which aims at 
drastically reducing the dimension of high-dimensional signals while preserving the ability to reconstruct them.

To perform the estimation task, we derive an iterative algorithm analogous to sparse reconstruction algorithms in the context of linear inverse problems. We exemplify our framework with the compressive estimation of a Gaussian Mixture Model (GMM), providing heuristics on the choice of the sketching procedure and theoretical guarantees of reconstruction. We experimentally show on synthetic data that the proposed algorithm yields results comparable to the classical Expectation-Maximization (EM) technique while requiring significantly less memory and fewer computations when the number of database elements is large. We further demonstrate the potential of the approach on real large-scale data (over $10^{8}$ training samples) for the task of model-based speaker verification.

Finally, we draw some connections between the proposed framework and approximate Hilbert space embedding of probability distributions using random features. We show that the proposed sketching operator can be seen as an innovative method to design translation-invariant kernels adapted to the analysis of GMMs. We also use this theoretical framework to derive information preservation guarantees, in the spirit of infinite-dimensional compressive sensing.}
\\
{Compressive Sensing, Compressive Learning, Database Sketching, Gaussian Mixture Model}
\end{abstract}

\section{Introduction}
\label{sec:intro}

An essential challenge in machine learning is to develop scalable techniques able to cope with large-scale training collections comprised of numerous elements of high dimension. To achieve this goal, it is necessary to come up with approximate learning schemes which perform the learning task with a fair precision while reducing the memory and computational requirements compared to standard techniques. A possible way to achieve such savings is to compress data beforehand, as illustrated in Figure~\ref{fig:comp_data}. Instead of the more classical individual compression of each element of the database with random projections \cite{Achlioptas2001, Fradkin2003,Calderbank2009,Reboredo2013,Maillard2009} the framework we consider in this paper corresponds to the top right scheme: a large collection of training data is compressed into a fixed-size representation called \emph{sketch}. The dimension of the sketch --and therefore the cost of infering/learning the parameters of interest from this sketch-- does not depend on the number of data items in the initial collection.
A complementary path to handling large-scale collections is online learning \cite{Cappe2009}. Sketching, which leverages ideas originating from streaming algorithms \cite{Cormode2005}, can trivially be turned into an online version and is amenable to distributed computing.

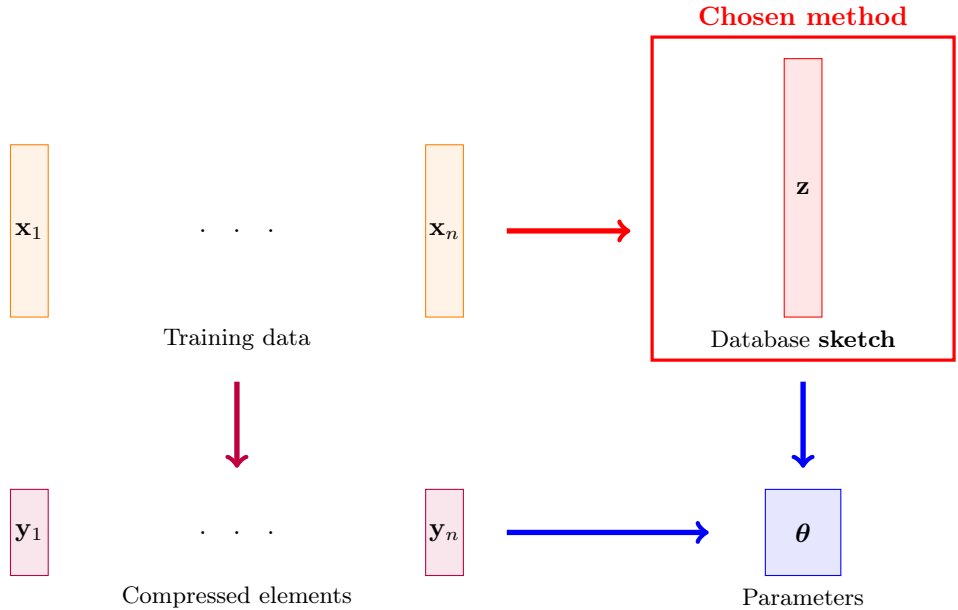
\begin{figure}
\centering
\begin{tikzpicture}
\def\u{13pt} 
\def\hrh{2.5*\u} 
\def\hrl{10*\u} 
\def\vrh{10*\u} 
\def\vrl{2.5*\u} 
 \def\hv{0.5*\vrh}
 \def\lv{\vrl/2.3}
 \def\hs{\hv}
 \def\gap{\hv/2}
 \def\gapt{\hv/3.33}
 
  \draw[orange,fill=orange!10] (-9*\lv,0) rectangle (-8*\lv,-\hv);
  \draw (-8.5*\lv,-\hv/2) node{ $\bfx_1$};
  \draw (-3*\lv,-\hv/2) node{.\quad .\quad .};
  \draw[orange,fill=orange!10] (2*\lv,0) rectangle (3*\lv,-\hv);
  \draw (2.5*\lv,-\hv/2) node{ $\bfx_n$};
  \draw (-3*\lv,-\hv-\hv/8) node{\small Training data};
  
  \draw[->, line width=2pt, color=purple] (-3*\lv,-\hv-\hv/4-\hv/8) -- (-3*\lv,-2*\hv+\hv/8);
  \draw[purple,fill=purple!10] (-9*\lv,-2*\hv) rectangle (-8*\lv,-2*\hv-\hv/2);
  \draw (-8.5*\lv,-2*\hv-\hv/4) node{ $\bfy_1$};
  \draw[purple,fill=purple!10] (2*\lv,-2*\hv) rectangle (3*\lv,-2*\hv-\hv/2);
  \draw (2.5*\lv,-2*\hv-\hv/4) node{ $\bfy_n$};
  \draw (-3*\lv,-7*\hv/4-\hv/2) node{.\quad .\quad .};
  \draw (-3*\lv,-2*\hv-\hv/2-\hv/8) node{\small Compressed elements};

  \draw[->, line width=2pt, color=red] (3*\lv+\hv/4,-\hv/2) -- (8*\lv-\hv/8,-\hv/2);
  \draw[red,fill=red!10] (11.5*\lv,\hv/2) rectangle (12.5*\lv,-\hv);
  \draw (12*\lv,-\hv/4) node{ $\bfz$};
  \draw (12*\lv,-\hv-\hv/8) node{\small Database \emph{\bf\small sketch}};
  \draw[red,very thick] (8*\lv,\hv/2+\hv/8) rectangle (16*\lv,-\hv-\hv/4);
\draw[red] (12*\lv,\hv/2+\hv/4) node{\bf Chosen method};

 \draw[->, line width=2pt, color=blue] (3*\lv+\hv/4,-2*\hv-\hv/4) -- (9.5*\lv,-2*\hv-\hv/4);
  \draw[->, line width=2pt, color=blue] (12*\lv,-\hv-\hv/4-\hv/8) -- (12*\lv,-2*\hv+\hv/8);
\draw[blue,fill=blue!10] (11*\lv,-2*\hv) rectangle (13*\lv,-2*\hv-\hv/2);
\draw (12*\lv,-2*\hv-\hv/4) node{ $\mtheta$};
\draw (12*\lv,-2*\hv-\hv/2-\hv/8) node{\small Parameters};
\end{tikzpicture}

\caption{{\bf Paths to compressive learning}.  The training data is compressed into a smaller representation, typically through random projections. This can either consist in reducing the dimensions of each individual entry (left bottom) or in computing a more global compressed representation of the data, called \emph{sketch} (top right). Parameters are then inferred from such a compressed representation by a learning algorithm adapted to it. The proposed approach explores the second sketch-based option.
}
\label{fig:comp_data}
\end{figure}

\subsection{From compressive sensing to sketched learning.}
As we will see, compressing a training collection into a sketch before learning is reminiscent of --and indeed inspired by-- compressive sensing (CS)~\cite{Foucart2013} and streaming algorithms \cite{Cormode2004,Cormode2005}. The main goal of CS is to find a dimensionality-reducing linear operator $\bM$ such that certain high-dimensional vectors (or signals) can be reconstructed from their observations by $\bM$. Initial work on CS~\cite{Candes2006a,Donoho2006} showed that such a reconstruction is possible for $k$-sparse signals of dimension $d$; from only $\order(k)$ linear measurements by (theoretically) solving an intractable NP-complete problem (\cite{Foucart2013}, chap. 2), and in practice from $\order(k\ln(d))$ linear measurements by solving a convex relaxation of this problem. Matrices $\bM$ with such reconstruction guarantees can be obtained as typical draws of certain random matrix ensembles. This reconstruction paradigm from few random measurements has subsequently been considered and proven to work for more general signal models \cite{Bourrier2014}. Examples of such models include low-rank matrices~\cite{Candes2011}, cosparse vectors~\cite{Nam2013} and dictionary models~\cite{Candes2011a}. Reconstruction from compressive measurements for these models is made possible, among other properties, by the fact that they correspond to unions of subspaces~\cite{Blumensath2011} which have a much lower dimension than the ambient dimension.

Low-dimensional models also intervene in learning procedures, where one aims at fitting a model of moderate ``dimension'' to some training data $\{\bx_{1} \ldots \bx_{n}\} \subset X$ in order to prevent overfitting and ensure good generalization properties. In this paper, we consider \emph{mixture models} comprising probability distributions on the set $X$ of the form
\begin{equation}\label{eq:MixtureModel}
\pp = \sum_{k=1}^K \alpha_k \pp_k,
\end{equation}
where the $\pp_k$'s are probability distributions taken in a certain set and the $\alpha_k$'s, $\alpha_{k}\geq 0$, $\sum_{k}\alpha_{k}=1$, are the weights of the mixture. Such a model 
can be considered as a generalized sparse model in the linear space $E$ of finite measures over the set $X$.

Similarly to compressive sensing, one can define a linear compressive operator $\skop: E\rightarrow \mathbb{C}^m$ which computes \emph{generalized moments}~\cite{Hall2005} of a measure $\meas$:
\begin{equation}\label{eq:generalized_moments}
\skop: \meas \mapsto \skop\meas := \fracsqm \left[\int_{X} M_1 d\meas, \ldots , \int_{X} M_m d\meas\right],
\end{equation}
where the $M_j$'s are well-chosen functions on $X$ and the constant $\fracsqm$ is used for normalization purposes. In the case where $\meas$ is a probability measure $\pp$, the integrals are the expectations of $M_j(\bx)$ with $\bx\sim \pp$.

Given some training data $\dataset=\left\{\bx_1,\ldots,\bx_n\right\}$ drawn from $X$, the corresponding empirical distribution is
\begin{equation}\label{eq:empirical_sketch}
\hat{\pp}=\frac{1}{n}\sum_{i=1}^n \delta_{\bx_i},
\end{equation}
where $\delta_{\bx_i}$ is a unit mass at $\bx_i$. A practical sketch of the data can then be defined\footnote{Any other unbiased empirical estimator of the moments, for example using the empirical median, can be envisioned.} and computed as 
\begin{equation}\label{eq:EmpSketch}
\hat{\bz} = \skop \hat{\pp}=\frac{1}{n\sqrt{m}}\left[\sum_{i=1}^n M_j(\bfx_i)\right]_{j=1,...,m}.
\end{equation}
Denoting $\Sigma \subset E$ the set of probability distributions satisfying~\eqref{eq:MixtureModel}, fitting a probability mixture to the training collection $\dataset$ in a compressive fashion can be expressed as the following optimization problem
\begin{equation}\label{eq:problemset}
\argmin{\pp\in \Sigma}~\neucl{\hat{\bz}-\skop \pp},
\end{equation}
which corresponds to the search for the probability mixture in the model set $\Sigma$ whose sketch is closest to the empirical data sketch $\hat{\bz}$. By analogy with sparse reconstruction, we propose an iterative greedy reconstruction algorithm to empirically address this problem, and exemplify our framework on the estimation of GMMs.

\subsection{Related works}
\label{sec:related_works}

A large set of the existing literature on random projections for dimension reduction in the context of learning focuses on the scheme represented on the bottom left of Figure~\ref{fig:comp_data}\,: each item of the training collection is individually compressed with random projections \cite{Achlioptas2001, Fradkin2003} prior to learning for classification \cite{Calderbank2009,Reboredo2013} or regression \cite{Maillard2009}, or to fitting a GMM \cite{Dasgupta1999}. In contrast, we consider here a framework where the whole training collection is compressed to a fixed-size sketch, corresponding to the top right scheme in Figure~\ref{fig:comp_data}. This framework builds on work initiated in \cite{Bourrier2013}\cite{Bourrier2015}. Compared to \cite{Bourrier2013}\cite{Bourrier2015}, the algorithms we propose here: a) are inspired by Orthogonal Matching Pursuits rather than Iterative Hard Thresholding; b) can handle GMMs with arbitrary diagonal variances; c) yield much better empirical performance 
(see Section~\ref{sec:results} for a thorough experimental comparison).

\review{Our approach bears connections with the Generalized Method of Moments (GeMM) \cite{Hall2005}, where parameters are estimated by matching empirical generalized moments computed from the data with theoretical generalized moments of the distribution. Typically used in practice when the considered probability models do not yield explicit likelihoods, the GeMM also provides mathematical tools to study the identifiability and learnability of parametric models such as GMMs~\cite{Belkin2010,Belkin2010a,Hsu2013}. 
Using the empirical characteristic function is a natural way of obtaining moment conditions \cite{Feuerverger1977,Feuerverger1981,Tran1998}. Following developments of GeMM with a continuum of moments instead of a finite number of them \cite{Carrasco2000}, powerful estimators can be derived when the characteristic function is available at all frequencies simultaneously \cite{Carrasco2002,Xu2010,Carrasco2014}. Yet, these estimators are rarely implementable in practice.
}

This is naturally connected with a formulation of mixture model estimation as a linear inverse problem. In~\cite{Bunea2010,Bertin2011} for example, this is addressed by considering a finite and incoherent dictionary of densities and the unknown density is reconstructed from its scalar products with every density of the dictionary. These scalar products can be interpreted as generalized moments of the density of interest. Under these assumptions, the authors provide reconstruction guarantees for their cost functions. In our framework, we consider possibly infinite and even uncountable dictionaries, and only collect a limited number of ``measurements'', much smaller than the number of elements in the dictionary. 

Sketching is a classical technique in the database literature \cite{Cormode2005}. A sketch is a fixed-size data structure which is updated with each element of a data stream, allowing one to perform some tasks without keeping the data stored. Applications include the detection of frequent elements, called \emph{heavy hitters} \cite{Cormode2009} and simple statistical estimations on the data \cite{Gilbert2002}. The sketches used in these works typically involve quantization steps which we do not perform in our work. We also consider the density estimation problem, which is closer to machine learning than the typical application of sketches. Closer to our work, the authors in~\cite{Thaper2002} consider the estimation of 2-dimensional histograms from random linear sketches. Even though this last method is theoretically applicable in higher dimensions, the complexity would grow exponentially with the dimension of the problem. Such a ``curse of dimensionality'' is also found in \cite{Bunea2010,Bertin2011}: the size of the dictionary grows exponentially with the dimension of the data vectors, and naturally impacts the cost of the estimation procedure. In our work, we rather consider parametric dictionaries that are described by a finite number of parameters. This enables us to empirically leverage the structure of iterative algorithms from sparse reconstruction and compressive sensing to optimize with respect to these parameters, offering better scalability to higher dimensions. This is reminiscent of generalized moments methods for the reconstruction of measures supported on a finite subset of the real line~\cite{Castro2012}, and can be applied to much more general families of probability measures.

\review{
The particular sketching operator that we propose to apply on GMMs (see Section \ref{sec:freq} and further) is obtained by randomly sampling the (empirical) characteristic function of the distribution of the training data. This can be seen as a combination between two techniques from the Reproducing Kernel Hilbert Space (RKHS) literature, namely embedding of probability distributions in RKHS using a feature map referred to as the ``Mean Map'' \cite{Borgwardt2006, Smola2007,Sriperumbudur2010}, and Random Fourier Features (RFFs) for approximating translation-invariant reproducing kernels \cite{Rahimi2007}.

Embedding of probability distributions in RKHS with the Mean Map has been successfully used for a large variety of tasks, such as two-sample test \cite{Gretton2006}, classification \cite{Muandet2012} or even performing algebraic operations on distributions \cite{Scholkopf2015}. In \cite{Sriperumbudur2011}, the  estimation of a mixture model with respect to the metric of the RKHS is considered with a greedy algorithm. The proposed algorithm is however designed to approximate the target distribution by a large mixture with many components, resulting in an approximation error that decreases as the number of components increases, while our approach considers \eqref{eq:MixtureModel} as a ``sparse'' combination of a fixed, limited number of components which we aim at identifying. Furthermore, unlike our method that uses RFFs to obtain an efficient algorithm, the algorithm proposed in \cite{Sriperumbudur2011} does not seem to be directly implementable.

Many kernel methods can be performed efficiently using finite-dimensional, nonlinear embeddings that approximate the feature map of the RKHS \cite{Rahimi2007,Vedaldi2012}. A popular method approximates translation-invariant kernels with RFFs \cite{Rahimi2007}. There has been since many variants of RFFs that are faster \cite{Le2013,Yang2015,Choromanski2016}, more precise \cite{Xinnan2016}, or designed for a different type of kernel \cite{Vedaldi2012}. Similar to our sketching operator, structures combining RFFs with Mean Map embedding of probability distributions have been recently used by the kernel community \cite{Bo2009,Sutherland2015,Oliva2015} to accelerate methods such as classification with the so-called Support Measure Machine \cite{Muandet2012,Sutherland2015,Oliva2015} or two-sample test \cite{Zhao2014,Chwialkowski2015,Jitkrittum2016,Paige2016}.

Our point of view, \ie~that of generalized compressive sensing, is sensibly different: we consider the sketch as a \emph{compressed} representation of the probability distribution, and demonstrate 
that it contains enough information to robustly recover the distribution from it, resulting in an effective ``compressive learning'' alternative to usual mixture estimation algorithms.
}

\subsection{Contributions and outline}
Our main contributions can be summarized as follows:
\begin{itemize}
\item Inspired by Orthogonal Matching Pursuit (OMP) and its variant OMP with Replacement (OMPR), we design in Section~\ref{sec:algo} an algorithmic framework for general compressive mixture estimation.
\item In the specific context of GMMs, we design in Section \ref{sec:BS} an algorithm that scales better with the number $K$ of mixture components. 
\item Inspired by random Fourier sampling for compressive sensing, we consider sketching operators $\skop$ defined in terms of random sampling of the characteristic function \cite{Bourrier2013,Bourrier2015}. 
However we show that the sampling pattern of \cite{Bourrier2013,Bourrier2015} is not adapted in high dimension. In Section \ref{sec:freqdist}, in the specific case of GMMs we propose a new heuristic and devise a practical scheme for randomly drawing the frequencies that define $\skop$. This is empirically demonstrated to yield significantly improved performance in Section~\ref{sec:results}. 
\item \review{We establish in Section~\ref{sec:kernel} a connection between the choice of the proposed sampling pattern and the design of a reproducing kernel on probability distributions. Compared to existing literature \cite{Gretton2012,Oliva2015}, our method is relatively simpler, faster to perform and fully unsupervised.}
\item Extensive tests on synthetic data in Section \ref{sec:results} demonstrate that our approach matches the estimation precision of a state-of-the-art C++ implementation of the EM algorithm while enabling significant savings in time and memory.
\item In the context of hypothesis testing-based speaker verification, we also report in Section \ref{sec:speaker_verification}
results on real data, where we exploit a corpus of 1000 hours of speech at scales inaccessible to traditional methods, and match  using a very limited number of measurements the results obtained with EM.
\item \review{We provide preliminary theoretical results (Theorem~\ref{thm:appli} in Section~\ref{sec:theory}) on the information preservation guarantees of the sketching operator. 
The proofs of these results (Appendices \ref{sec:proof} and \ref{sec:add_proofs}) introduce a new variant of the Restricted Isometry Property (RIP) \cite{Candes2005}, connected here to kernel mean embedding and Random Features. Compared to usual guarantees in the GeMM literature, our results have less of a ``statistics'' flavor and more of a ``signal processing'' one, such as robustness to modeling error, \ie the true distribution of the data is not exactly a GMM but close to one.}
\end{itemize}


\section{A Compressive Mixture Estimation Framework}
\label{sec:algo}


%
%


In classical compressive sensing \cite{Foucart2013}, a signal $\bfx \in \mathbb{R}^d$ is encoded with a measurement matrix $\bfM \in \mathbb{R}^{m \times d}$ into a compressed representation $\bfz \in \mathbb{R}^m$:
\begin{equation}\label{eq:CS}
\bfz=\bfM\bfx
\end{equation}
and the goal is to recover $\bfx$ from those linear measurements. Often the system is underdetermined ($m<d$) and recovery can only be done with additional assumptions, typically sparsity. The vector $\bfx=[x_\ell]_{\ell=1}^d$ is said to be \emph{sparse} if it has only a limited number $k<d$ of non-zero coefficients. Its \emph{support} is the set of indices of non-zero entries: $\Gamma(\bfx)=\{\ell\ |\ x_\ell\neq 0\}$. The notation $\bfM_{\Gamma}$ (resp. $\bfx_{\Gamma}$) denotes the restriction of matrix $\bfM$ (resp. vector $\bfx$) to columns (resp. entries) with indices in $\Gamma$.

A sparse vector can be seen as a combination of few basis elements: $\bfx=\sum_{\ell \in \Gamma} x_\ell \bfe_\ell$, where $\left\lbrace\bfe_\ell\right\rbrace_{\ell=1,...,d}$ is the canonical basis of $\mathbb{R}^d$. The measurement vector $\bz$ is thus expressed as a combination of few \emph{atoms}, corresponding to the columns of the measurement matrix: $\bfz=\sum_{\ell \in \Gamma}x_\ell\bfM\bfe_\ell$. The set of all atoms is referred to as a \emph{dictionary} $\left\lbrace\bfM\bfe_\ell\right\rbrace_{\ell=1,...,d}$.

\subsection{Mixture model and generalized compressive sensing} Let $E$ = $E(X)$ be a space of signed finite measures over a measurable space $(X,\mathcal{B})$, and $\pmeas$ the set of probability distributions over $X$, $\pmeas := \left\lbrace \pp \in E; \pp\geq 0, \int_X d\pp=1\right\rbrace$. In our framework, a distribution $\pp \in \pmeas$ is encoded with a linear measurement operator (or \emph{sketching} operator) $\skop: \pmeas \rightarrow \mathbb{C}^m$:
\begin{equation}
\bfz=\skop \pp.
\end{equation}
As in classical compressive sensing, we define a ``sparse'' model in $\pmeas$. As mentioned in the introduction, it is here assimilated to a mixture model \eqref{eq:MixtureModel}, generated by combining elements from some given set $\gaussset=\left\lbrace \pp_\mtheta\right\rbrace_{\mtheta \in \thetaset} \subset \pmeas$ of basic distributions indexed by a parameter $\mtheta \in \thetaset$. For some finite $K \in \mathbb{N}^*$, a distribution 
is thus said to be \emph{K-sparse} (in $\gaussset$) if it is a convex combination of $K$ elements from $\gaussset$:
\begin{equation}
\pp_{\hyppar,\malpha}=\sumk \alpha_k \pp_{\mtheta_k},
\end{equation}
with $\pp_{\mtheta_k} \in \gaussset$, $\alpha_k \geq 0$ for all $k$'s, and $\sumk \alpha_k=1$.
We name \emph{support} of the representation\footnote{Note that this representation might not be unique.} $(\hyppar,\malpha)$ of such a sparse distribution the set of parameters $\hyppar=\left\lbrace\mtheta_1,...,\mtheta_K\right\rbrace$.


The measurement vector $\bfz=\skop \pp_{\hyppar,\malpha}=\sumk \alpha_k \skop \pp_{\mtheta_k}$ of a sparse distribution is a combination of \emph{atoms} selected from the \emph{dictionary} $\left\lbrace\skop \pp_\mtheta\right\rbrace_{\mtheta \in \thetaset}$ indexed by the parameter $\mtheta$. Table \ref{tab:notation} summarizes the notations used in the context of compressive mixture estimation and their correspondence with more classical notions from compressive sensing.

\begin{table}[t]
\centering
\begin{tabular}{c|c|c|}
\cline{2-3}
& Usual compressive sensing & Compressive mixture estimation \\ \hline
\multicolumn{1}{|c|}{Signal} & $\bfx \in \mathbb{R}^d$ & $\pp \in \pmeas$ \\ \hline
\multicolumn{1}{|c|}{Model} & $K$-sparse vectors & $K$-sparse mixtures $\pp_{\hyppar,\malpha} = \sumk \alpha_k \pp_{\mtheta_k} $ \\ \hline
\multicolumn{1}{|c|}{Measurement operator} & $\bfM \in \mathbb{R}^{m \times d}$ & $\skop: \pmeas \rightarrow \mathbb{C}^m$ \\ \hline
\multicolumn{1}{|c|}{Support 
} & $\Gamma(\bfx)=\{\ell\ |\ x_\ell \neq 0\}$ & $\Gamma(\pp_{\hyppar,\malpha})=\hyppar=\lbrace\mtheta_1,...,\mtheta_K\rbrace$ \\ \hline
\multicolumn{1}{|c|}{Dictionary of atoms} & $\left\lbrace\bfM\bfe_\ell\right\rbrace_{\ell=1,...,d}$ & $\left\lbrace\skop \pp_\mtheta\right\rbrace_{\mtheta \in \thetaset}$ \\ \hline
\end{tabular}
\caption{Correspondance between objects manipulated in usual compressive sensing of finite-dimensional signals and in the proposed compressive mixture estimation framework.}
\label{tab:notation}
\end{table}

\subsection{Principle for reconstruction: moment matching}
As mentioned in Section \ref{sec:intro}, usual reconstruction algorithms (also known as {\em decoders}~\cite{Cohen2009,Bourrier2014}) in generalized compressive sensing are designed with the purpose of minimizing the measurement error while enforcing sparsity \cite{Blumensath2008}, as formulated in equation \eqref{eq:problemset}. Here it also corresponds to traditional parametric optimization in the Generalized Method of Moments (GeMM) \cite{Hall2005}:
\begin{equation}\label{eq:costfun}
\argmin{\hyppar,\malpha} \neucl{\hat\bfz-\skop \pp_{\hyppar,\malpha}},
\end{equation}
where $\hat\bfz=\skop\hat\pp=\tfrac{1}{n}\sum_{i=1}^n\skop\delta_{\bfx_i}$ is the empirical sketch. This problem is usually highly non-convex and does not allow for an efficient direct optimization, nevertheless we show in Section \ref{sec:theory} that in some cases it yields a decoder robust to modeling errors and empirical estimation errors, with high probability.

Convex relaxations of~\eqref{eq:costfun} based on sparsity-promoting penalization terms \cite{Bunea2010,Bertin2011,Castro2012,Pilanci2012} can be envisioned in certain settings, however their direct adaptation to general uncountable dictionaries of atoms (\eg, with GMMs) seems difficult. The main alternative is greedy algorithms. Using an algorithm inspired by Iterative Hard Thresholding (IHT) \cite{Blumensath2009}, Bourrier \etal~\cite{Bourrier2013} estimate mixtures of isotropic Gaussian distributions with fixed variance, using a sketch formed by sampling the empirical characteristic function. As will be shown in Section \ref{sec:results_N}, this IHT-like algorithm often yields an unsatisfying local minimum of \eqref{eq:costfun} when the variance is estimated. Instead, we propose a greedy approach similar to Orthogonal Matching Pursuit (OMP) \cite{Mallat1993,Pati1993} and its extension OMP with Replacement (OMPR) \cite{Jain2011}.

\review{Another intuitive solution would be to discretize the space of the parameter $\mtheta$ to obtain a finite dictionary of atoms and apply the classic convex relaxation or greedy methods mentioned above. However, one quickly encounters the well-known curse of dimensionality: for a grid with precision $\varepsilon$ and a parameter of dimension $p$, the size of the dictionary is as $\mathcal{O}\left(\varepsilon^{-p}\right)$, which is intractable even for moderate $p$. Initial experiments for learning GMMs in dimension $d=2$ with diagonal covariance (\ie the dimension of the parameter $\mtheta$ is $p=4$) show that this approach is extremely long and has a very limited precision. Instead, in the next section we propose an adaptation of OMPR directly in the continuous domain.}

\subsection{Inspiration: OMPR for classical compressive sensing}\label{sec:OMPR}

Matching Pursuit \cite{Mallat1993} and Orthonormal Matching Pursuit \cite{Pati1993} deal with general sparse approximation problems. They gradually extend the sparse support by selecting atoms most correlated with the residual signal, until the desired sparsity is attained.

An efficient variation of OMP called OMP with Replacement (OMPR) \cite{Jain2011} exhibits better reconstruction guarantees. Inspired by IHT\cite{Blumensath2009}, and similar to CoSAMP or Subspace Pursuit \cite{Foucart2013}, it increases the number of iterations of OMP and extends the size of the support \emph{further than} the desired sparsity before reducing it with Hard Thresholding to suppress spurious atoms.

\subsection{Proposed algorithms: Compressive Learning OMP/OMPR}\label{sec:OMPRC}
Adapting OMPR to the considered compressive mixture estimation framework requires several modifications. We detail them below, and summarize them in Algorithm \ref{algo:CL-OMP}.

\begin{algorithm2e*}[htbp]
\SetKwBlock{uBegin}{Step}{end}
\KwData{Empirical sketch $\hat{\bfz}$, sketching operator $\skop$, sparsity $K$, number of iterations $T\geq K$}
\KwResult{Support $\hyppar$, weights $\malpha$}
$\hat\bfr \leftarrow \hat\bfz$; $\hyppar \leftarrow \emptyset$ \;
\For{$t \leftarrow 1$ \KwTo $T$}{
\uBegin({\bf 1}: Find a normalized atom highly correlated with the residual with a gradient descent){$\mtheta \leftarrow \texttt{maximize}_\mtheta \left(\mathrm{Re}\ipeucl{\frac{\skop \pp_\mtheta}{\neucl{\skop \pp_\mtheta}}}{\hat\bfr}, \texttt{init}=\texttt{rand}\right)$ \;
}
\uBegin({\bf 2}: Expand support){$\hyppar \leftarrow \hyppar\cup \{\mtheta\}$ \;
}
\uBegin({\bf 3}: Enforce sparsity by Hard Thresholding if needed){\If{$|\hyppar|>K$}{$\boldsymbol\beta \leftarrow \arg\min_{\boldsymbol\beta \geq 0} \neucl{\hat\bfz - \sum_{k=1}^{|\hyppar|} \beta_k \frac{\skop \pp_{\mtheta_k}}{\neucl{\skop \pp_{\mtheta_k}}}
}$\; 
Select $K$ largest entries $\beta_{i_1},...,\beta_{i_K}$\;
Reduce the support $\hyppar \leftarrow \{\mtheta_{i_1},...,\mtheta_{i_K}\}$\;
}
}
\uBegin({\bf 4}: Project to find weights){$\malpha \leftarrow \arg\min_{\malpha \geq 0} \neucl{\hat\bfz - \sum_{k=1}^{|\hyppar|} \alpha_k \skop \pp_{\mtheta_k}}$\;
}
\uBegin({\bf 5}: Perform a gradient descent \emph{initialized with current parameters}){$\hyppar,\malpha \leftarrow \texttt{minimize}_{\hyppar,\malpha}\left( \neucl{\hat\bfz - \sum_{k=1}^{|\hyppar|} \alpha_k \skop \pp_{\mtheta_k}}, \texttt{init}=\left({\hyppar},{\malpha}\right), \texttt{constraint}=\{\malpha \geq 0\}\right)$\;
}
\nosemic Update residual: $\hat\bfr \leftarrow \hat\bfz - \sum_{k=1}^{|\hyppar|} {\alpha}_k \skop \pp_{\mtheta_k}$\;
}
\nosemic Normalize $\malpha$ such that $\sumk \alpha_k=1$\;
\caption{Compressive mixture learning {\em {\`a} la} OMP: CL-OMP ($T=K$) and CL-OMPR ($T=2K$)}
 \label{algo:CL-OMP}
\end{algorithm2e*}

Several aspects of this framework must be highlighted:
\begin{itemize}
\item
{\bf Non-negativity.}
Unlike classical compressive sensing, the compressive mixture estimation framework imposes a non-negativity constraint on the weights $\malpha$, that we enforce at each iteration. Thus Step 1 is modified compared to classical OMPR by replacing the modulus of the correlation by its real part, to avoid negative correlation between atom and residual. Similarly, in Step 4 we perform a Non-Negative Least-Squares (NNLS) \cite{Lawson1995} instead of a classical Least-Squares.
\item
{\bf Normalization.}
Note that we do not enforce normalization $\sumk \alpha_k=1$ at each iteration. Instead, a normalization of $\malpha$ is performed at the end of the algorithm to obtain a valid distribution. Enforcing the normalization constraint at each iteration was found on initial experiments to have a negligible effect while increasing computation time.
\item
{\bf Continuous dictionary.}
The set $\gaussset$ of elementary distributions is often continuously indexed (as with GMMs, see Section~\ref{sec:BS}) and cannot be exhaustively searched. Instead we propose to replace the maximization in Step 1 of classical OMPR with a randomly initialized gradient descent, denoted by a call to a sub-routine $\texttt{maximize}_\mtheta$, leading to a -- local -- maximum of the correlation between atom and residual. Note that the atoms are normalized during the search, as is often the case with OMP.
\item
{\bf Global optimization step to handle coherent dictionaries.}
Compared to classical OMPR, the proposed algorithm includes a new step at each iteration (Step 5), which further reduces the cost function with a gradient descent initialized with the \emph{current} parameters $(\hyppar,\malpha)$. This is denoted by a call to the sub-routine $\texttt{minimize}_{\hyppar,\malpha}$. The need for this additional step stems from the lack of incoherence between the elements of the uncountable dictionary. 
For instance, in the case of GMM estimation, a $(K+1)$-GMM approximation of a distribution cannot be directly derived from a $K$-GMM by simply adding a Gaussian. 
This is reminiscent of a similar problem handled in High Resolution Matching Pursuit (HRMP) \cite{Gribonval1996}, which uses a multi-scale decomposition of atoms, while we handle here the more general case of a continuous dictionary using a global gradient descent that adjusts all atoms. \review{Experiments show that this step is the most time-consuming of the algorithm, but that it is necessary.}
%
\end{itemize}
Similar to classical OMPR, Algorithm \ref{algo:CL-OMP} yields two algorithms depending on the number of iterations:

\begin{enumerate}
\item {\bf Compressive Learning OMP} (CL-OMP) if run \emph{without} Hard Thresholding (\ie~ with $T=K$);
\item {\bf CL-OMPR} (with Replacement) if run with $T=2K$ iterations.
\end{enumerate}

\review{
\paragraph{Learning the number of components ?} In the proposed framework, the number of components $K$ is known in advance and provided by the user. However, it is known that greedy approaches such as OMP are convenient to derive stopping conditions, that could be readily applied to CL-OMP: when the residual falls below a fixed (or adaptive) threshold, stop the algorithm (additional strategies would be required for CL-OMPR however). In this paper however, we only compare the proposed method with classical approaches such as EM for Gaussian Mixture Models, that consider the number of components $K$ {\em known in advance}. We leave for future work the implementation of a stopping condition for CL-OMP(R) and comparison with existing methods for model selection.
}

\subsection{Complexity of CL-OMP(R).}

Just as OMP, which complexity scales quadratically with the sparsity parameter $K$, proposed greedy approaches CL-OMP or CL-OMPR have a computational cost of the order of $\order(mdKT)$, where $T \geq K$ is the number of iterations, resulting in a quadratic cost with respect to the number of components $K$.

This is potentially a limiting factor for the estimation of mixtures of many basic distributions (large $K$). In classical sparse approximation, approximate least squares approaches such as Gradient Pursuit \cite{Blumensath2008} or LocOMP \cite{Mailhe2009} have been developed to overcome this computational bottleneck.  One could probably get inspiration from these approaches to further scale up compressive mixture estimation, however in the context of GMMs we propose in Section~\ref{sec:BS} to rather leverage ideas from existing fast Expectation-Maximization (EM) algorithms that are \emph{specific to GMMs}.

\subsection{Sketching by randomly sampling the characteristic function}\label{sec:sampling}

Let us now assume $X=\mathbb{R}^d$. {\em The reader will notice that in classical compressive sensing, the compressed object is a vector $\bfx \in \mathbb{R}^{d}$, while in this context, a training collection of vectors $\{\bfx_{1},\ldots,\bfx_{n}\} \subset \mathbb{R}^{d}$ is considered as an empirical version of some probability distribution $\pp \in E(\mathbb{R}^{d})$ which is the compressed object.}

The proposed algorithms CL-OMP(R) are suitable for any sketching operator $\skop$ and any mixture model of parametric densities $\pp_\mtheta$, {\em as long as the optimization schemes in Steps 1 and 5 can be performed}. 
In the case of a continuous dictionary the optimization steps can be performed with simple gradient descents implicitly represented by calls to the sub-routines $\texttt{maximize}_\mtheta$ and $\texttt{minimize}_{\hyppar,\malpha}$, provided $\skop \pp_\mtheta$ and its gradient with respect to $\mtheta$ have a closed-form expression. 

In many important applications such as medical imaging (MRI and tomography), astrophysics or geophysical exploration, one wishes to reconstruct a signal from incomplete samples of its discrete Fourier transform. Random Fourier sampling was therefore one of the first problems to give rise to the classical notions of compressive sensing \cite{Candes2005,Candes2006,Candes2006a}. Indeed, a random uniform selection of rows of the full Fourier matrix, \ie~a random selection of frequencies, forms a partial Fourier matrix that satisfies a certain Restricted Isometry Property (RIP) with overwhelming probability \cite{Candes2006a}, and is therefore appropriate for the encoding and recovery of sparse signals. For more details, we refer the reader to \cite{Foucart2013,Candes2006a,Candes2005,Candes2006,Baraniuk2007} and references therein. Inspired by this methodology, we form the sketch by sampling the characteristic function (\ie~the Fourier transform) of the probability distribution $\pp$.

The characteristic function $\chrc_\meas$ of a finite measure $\meas \in E(\mathbb{R}^d)$ is defined as:
\begin{equation}\label{eq:general_chrc}
\chrc_\nu(\freq)=\int\left(e^{-\imaginaryi\freq^T\bfx}\right)d\meas(\bfx)\quad \forall \freq \in \mathbb{R}^{d}.
\end{equation}
For a sparse distribution $\pp_{\hyppar,\malpha}$ (in some given set of basic distributions $\gaussset \subset \pmeas$), we also denote $\chrc_{\hyppar,\malpha}=\chrc_{\pp_{\hyppar,\malpha}}$ for simplicity.

The characteristic function $\chrc_\pp$ of a probability distribution $\pp$ is a well-known object with many desirable properties. It completely defines any distribution with no ambiguity and often has a closed-form expression (even for distributions which may not have a probability density function with closed-form expression, \eg, for $\alpha$-stable distributions~\cite{Salas-Gonzalez2009}), which makes it a suitable choice to build a sketch used with CL-OMP(R). It has been used as an estimation tool at an early stage \cite{Feuerverger1981} as well as in more recent developments on GeMM \cite{Carrasco2002}.

The proposed sketching operator is defined as a sampling of the characteristic function. Given frequencies $\freqs=\left\lbrace\freq_1,...,\freq_m\right\rbrace$ in $\mathbb{R}^d$, we define generalized moment functions:
\begin{equation}\label{eq:chrc_moment}
M_j(\bfx)=\exp\left(-\imaginaryi \freq_j^T\bfx\right),~j=1\cdots m,
\end{equation}
and the sketching operator \eqref{eq:generalized_moments} is therefore expressed as
\begin{equation}
\label{eq:skop}
\skop \meas=\fracsqm\left[\chrc_\meas(\freq_1),...,\chrc_\meas(\freq_m)\right]^{T}.
\end{equation}

Given a training collection $\dataset=\left\lbrace\bfx_1,...,\bfx_n\right\rbrace$ in $\mathbb{R}^d$, we denote $\hat\chrc(\freq)=\frac{1}{n} \sumn e^{-\imaginaryi\freq^T\bfx_i}$ the empirical characteristic function\footnote{Other more robust estimators can be envisioned such as the empirical median. The empirical average allows more easy streaming or distributed computing.}. The empirical sketch $\hat\bfz=\skop \hat \pp$ is
\begin{equation}\label{eq:empirical_sketch_chf}
\hat\bfz=\fracsqm\left[\hat\chrc(\freq_1),...,\hat\chrc(\freq_m)\right]^{T}.
\end{equation}

 
To fully specify the sketching operator \eqref{eq:skop}, one needs to choose the frequencies $\freq_j$ that define it. In the spirit of Random Fourier Sampling, we propose to define a probability distribution $\freqdist \in \pmeas$ to draw $(\freq_1,...,\freq_m) \overset{i.i.d.}{\sim} \freqdist$. Choosing this distribution is a problem of its own that will be discussed in details in Section~\ref{sec:freqdist}.

%

\paragraph{Connections with Random Neural Networks.} It is possible to draw connections between the proposed sketching operation and neural networks. Denote $\bfW:=\left[\freq_1,...,\freq_m\right] \in \mathbb{R}^{d\times m}$ and $\bfX:=\left[\bfx_1,...,\bfx_n\right]\in\mathbb{R}^{d\times n}$. To derive the sketch, one needs to compute the matrix $\bfU:=\bfW^T\bfX \in \mathbb{R}^{m\times n}$, take the complex exponential of each individual entry $\bfV:=\bar{f}(\bfU)$ where $\bar{f}$ is the pointwise application of the function $f(x)=e^{-\imaginaryi x}/\sqrt{m}$ and finally pool the columns of $\bfV$ to obtain the empirical sketch $\hat\bfz=\left(\sum_{i=1}^n\bfv_i\right)/n$. This procedure indeed shares similarities with a $1$-layer neural network: the output $\bfy\in\mathbb{R}^m$ of such a network can be expressed as $\bfy=f(\bfW^T\bfx)$, where $\bfx\in\mathbb{R}^d$ is the input signal, $f$ is a pointwise non-linear function and $\bfW \in \mathbb{R}^{d\times m} $ is some weighting matrix. Therefore, in the proposed framework, such a $1$-layer network is applied to many inputs $\bfx_i$, then the empirical average of the outputs $\bfz_i$ is taken to obtain a representation of the underlying distribution of the $\bfx_i$.

Neural networks with weights $\bfW$ chosen at random rather than learned on training data -- as is done with the frequencies $\freq_j$-- have been studied in the so-called \emph{Random Kitchen Sinks} \cite{Rahimi2009} or in the context of Deep Neural Networks (DNN) with Gaussian weights \cite{Giryes2015}. In the latter, they have been shown to perform a stable embedding of the input $\bfx$ when it lives on a low-dimensional manifold. Similar to \cite{Giryes2015}, we show in Section \ref{sec:theory} that with high probability the sketching procedure is a stable embedding {\em of the probability distribution of $\bfx$} when this distribution belongs to, \eg, a compact manifold.

\subsection{Complexity of sketching}\label{sec:complexity}

The main computational load of the sketching operation is the computation of the matrix $\bfU=\bfW^T\bfX$, which theoretically scales in $\order(dmn)$. Large multiplications by random matrices are indeed a well-known computational bottleneck in Compressive Sensing, and some methods circumvent this issue by using approximated fast transforms \cite{Do2011,Lemagoarou2015}. Closer to our work (see Section \ref{sec:kernel}), fast transforms have also been used in the context of Random Fourier Features and kernel methods \cite{Le2013,Yang2015}. 
We leave the study of a possible adaptation of these acceleration methods for future work, and focus on simple practical remarks about the computation of the sketch.
\paragraph{GPU computing.} Matrix multiplication is one of the most studied problem in the context of large-scale computing. A classical way to drastically reduce its cost consists in using GPU computing \cite{Volkov2008}. Recent architectures can even leverage giant matrix multiplication using multiple GPUs in parallel \cite{Zhang2015}.
\paragraph{Distributed/online computing.} The computation of the sketch can also be performed in a distributed manner. One can divide the database $\dataset$ in $T$ subsets $\dataset_t$ containing $n_t$ items respectively, after which individual computing units can compute the sketches $\hat\bfz_t$ of each subset $\dataset_t$ in parallel, using the same frequencies. Those sketches are then easily merged\footnote{A similar strategy can also be used on a single machine when the matrix $\bfU$ is too large to be stored} as $\hat\bfz=\sum_{t=1}^T n_t\hat\bfz_t/\sum_{t=1}^T n_t$. The cost of computing the sketch is thus divided by the number of units $T$. Similarly, this simple observation allows the sketch to be computed in an online fashion.

\section{Sketching Gaussian Mixture Models}\label{sec:freq}





In this section, we instantiate the proposed framework in the context of Gaussian Mixture Models (GMMs). A more scalable algorithm specific to GMMs is first introduced as a possible alternative to CL-OMP(R). We then focus on the design of the sketching operator $\skop$, \ie~on the design of the frequency distribution $\freqdist$ (see Section \ref{sec:sampling}).

\subsection{Gaussian Mixture Models}
In the GMM framework, the basic distributions $\pp_\mtheta \in \gaussset$ are Gaussian distributions with density functions denoted $\dens_\mtheta$:
\begin{equation}\label{eq:gaussdensity}
\dens_\mtheta(\bfx)=\mathcal{N}(\bfx;\mmu,\mSigma)=\frac{1}{(2\pi)^{d/2}|\mSigma|^{1/2}}\exp\left(-\frac{1}{2}(\bfx-\mmu)^T\mSigma^{-1}(\bfx-\mmu)\right),
\end{equation}
where $\mtheta=(\mmu,\mSigma)$ represents the parameters of the Gaussian with mean $\mmu \in \mathbb{R}^d$ and covariance $\mSigma \in \mathbb{R}^{d \times d}$. A Gaussian is said to be isotropic, or spherical, if its covariance matrix is proportional to the identity: $\mSigma=\sigma^2\bfI$. 

Densities of GMMs are denoted $\dens_{\hyppar,\malpha}=\sumk \alpha_k \dens_{\mtheta_k}$.
A $K$-GMM is then naturally parametrized by weight vector $\malpha \in \mathbb{R}^K$ and parameters $\mtheta_k=(\mmu_k,\mSigma_k),~k=1\cdots K$. 

Compressive density estimation with mixtures of isotropic Gaussians with fixed, known variance was considered in \cite{Bourrier2013}. In this work, we  consider mixtures of Gaussians with diagonal covariances, which is known to be sufficient for many applications \cite{Reynolds2000,Zhou2008} and is the default framework in well-known toolboxes such as VLFeat \cite{Vedaldi2010}. We denote $(\sigma_{k,1}^2,...,\sigma_{k,d}^2)=\mathrm{diag}\left(\mSigma_k\right)$. Depending on the context, we equivalently denote $\mtheta_k=[\mmu_k;\msigma_k] \in \mathbb{R}^{2d}$ where $\msigma_k=[\sigma^2_{k,\ell}]_{\ell=1\cdots d}$ is the diagonal of the covariance of the $k$-th Gaussian.

The characteristic function of a Gaussian $\pp_\mtheta$ has a closed-form expression:
\begin{equation}\label{eq:chrc}
\chrc_\mtheta(\freq)=\exp\left( -\imaginaryi\freq^T\mmu\right)\exp\left(-\frac{1}{2}\freq^T\mSigma\freq\right),
\end{equation}
from which we can easily derive the expression of the gradients necessary to perform the optimization schemes  $\texttt{maximize}_\mtheta$, $\texttt{minimize}_{\hyppar,\malpha}$ in CL-OMP(R), with the sketching operator introduced in Section \ref{sec:sampling}.

\subsection{A faster algorithm specific to GMM estimation}\label{sec:BS}



To handle mixtures of many Gaussians (large $K$), the fast hierarchical EM used in \cite{Sadjadi2013} alternates between binary splits of each Gaussian $k$ along its first principal component (in the case of Gaussians with diagonal covariance, this is the dimension $\ell$ with the highest variance $\sigma^2_{k,\ell}$) and a few EM iterations. 

Our compressive adaptation 
is summarized in Algorithm~\ref{algo:BS-CGMM}. The binary split is performed by calling the function \texttt{Split}, and the EM steps are replaced by \textbf{Step 5} of Algorithm \ref{algo:CL-OMP} to adjust all Gaussians. In the case where the targeted sparsity level $K$ is not a power of $2$, we split the GMM until the support reaches a size $2^{\left\lceil \log_2 K \right\rceil} > K$, then reduce it with a Hard Thresholding (Step 3 of Algorithm~\ref{algo:CL-OMP}), similar to CL-OMPR.

Since the number of iterations in Algorithm~\ref{algo:BS-CGMM} 
is $T=\left\lceil \log_2 K \right\rceil$, the computational cost of this algorithm scales in $\order(mdK\log K)$, which is much faster for large $K$ than the quadratic cost of CL-OMPR.

In practice, Algorithm~\ref{algo:BS-CGMM} 
is seen to sometimes yield worse results than CL-OMPR (see Section \ref{sec:results_N}), but be well-adapted to other tasks such as, \eg, GMM estimation for large-scale speaker verification (see Section \ref{sec:speaker_verification}).


\begin{function}
\KwData{Support $\hyppar=\{\mtheta_1,...,\mtheta_K\}$ where $\mtheta_{k} = [\mmu_{k};\msigma_{k}]$}
\KwResult{New support $\hyppar^{new}$ of size $|\hyppar^{new}|=2K$}
$\hyppar^{new} \leftarrow \emptyset$ \;
\For{$k \leftarrow 1$ \KwTo $K$}{$\ell \leftarrow \arg\max_{j\in[1,d]} \sigma_{k,j}^2$ \;
$\hyppar^{new} \leftarrow \hyppar^{new} \cup \left\lbrace \left( \mmu_k - \sigma_{k,\ell}\bfe_\ell,\mSigma_k\right), \left( \mmu_k + \sigma_{k,\ell}\bfe_\ell,\mSigma_k\right)\right\rbrace$ \;
}
\caption{Split($\hyppar$): split each Gaussian in the support along its dimension of highest variance\label{tab:split}}
\end{function}

\begin{algorithm2e}
\SetKwBlock{uBegin}{Step}{end}
\KwData{Sketch $\hat{\bfz}$, sketching operator $\skop$, sparsity $K$}
\KwResult{Support $\hyppar$, weights $\malpha$}
$\hyppar \leftarrow \emptyset$ \;
\Begin(Initialize with \emph{one} atom highly correlated with the sketch){
$\hyppar \leftarrow \left\lbrace\texttt{maximize}_\mtheta \left( \mathrm{Re}\ipeucl{\frac{\skop \pp_\mtheta}{\neucl{\skop \pp_\mtheta}}}{\hat\bfz}, \texttt{init}=\texttt{rand}\right)\right\rbrace$ \;
}
\For{$t \leftarrow 1$ \KwTo $\left\lceil\log_2 K \right\rceil$}{
\Begin(Split each Gaussian in the support along its dimension of highest variance){$\hyppar \leftarrow \texttt{Split}(\hyppar)$\;
}
Perform {\bf Step 3}, {\bf Step 4} and {\bf Step 5} of Algorithm \ref{algo:CL-OMP}\;
}
Normalize $\malpha$ such that $\sumk \alpha_k=1$
\caption{An algorithm with complexity $O(K \log K)$ for compressive GMM estimation}
 \label{algo:BS-CGMM}
\end{algorithm2e}


\subsection{Designing the frequency sampling pattern}\label{sec:freqdist}
A key ingredient in designing the sketching operator $\skop$ is the choice of 
a probability distribution $\freqdist$ to draw the frequency sampling pattern $\{\freq_{1},\ldots,\freq_{m}\}$ as defined in Section~\ref{sec:sampling}.  
We show in Section~\ref{sec:kernel} that this corresponds to the design of a translation-invariant kernel in the data domain. \review{Interestingly, working in the Fourier domain seems to make the heuristic design strategy more direct. Literature on designing kernels in this context usually focus on maximizing the distance between the sketch of two distributions \cite{Sriperumbudur2009,Oliva2015}, which cannot be readily applied in our context since we sketch only a single distribution. However, as we will see, the proposed approach follows the general idea of maximizing the capacity of the sketch to distinguish this distribution from others, which amounts to maximizing the variations of the sketch with respect to the parameters of the GMM at the selected frequencies.}

\subsubsection{Oracle frequency sampling pattern for a single known Gaussian}
We start by designing a heuristic for choosing frequencies adapted to the estimation of a single Gaussian $\pp_\mtheta$, assuming the parameters $\mtheta=(\mmu,\mSigma)$ are available -- which is obviously not the case in practice. We will deal in due time with mixtures, and with unknown parameters.

\paragraph{Gaussian frequency distribution.}

Recall the expression \eqref{eq:chrc} of the characteristic function $\chrc_\mtheta(\freq)=e^{ -\imaginaryi\freq^T\mmu-\frac{1}{2}\freq^T\mSigma\freq}$ of the Gaussian $\pp_\mtheta$. It is an oscillating function with Gaussian amplitude of inverted variance with respect to the original Gaussian. Given that $|\chrc_\mtheta| \propto \mathcal{N}\left(0,\mSigma^{-1}\right)$,
choosing a Gaussian frequency distribution $\freqdist^{(G)}_\mSigma=\mathcal{N}\left(0,\mSigma^{-1}\right)$ is a possible intuitive choice \cite{Bourrier2013,Bourrier2015} to sample the characteristic function $\chrc_\mtheta$. It concentrates frequencies in the regions where the sampled characteristic function has high amplitude.

However, points drawn from a high-dimensional Gaussian concentrate on an ellipsoid which moves away from the origin as the dimension $d$ increases. Such a Gaussian sampling therefore ``undersamples'' low or even middle frequencies. This phenomenon has long been one of the reasons for using dimensionality reduction for GMM estimation \cite{Dasgupta1999}.
Hence, in high dimension the amplitude 
of the characteristic function becomes negligible (with high probability) at all selected frequencies.


\paragraph{Folded Gaussian radial frequency distribution.}

In light of the problem observed with the Gaussian frequency distribution, we propose to draw frequencies from a distribution  allowing for an accurate control of the quantity $\freq^T\mSigma\freq$, and thus of the amplitude $e^{-\frac{1}{2}\freq^T\mSigma\freq}$ of the characteristic function. This is achieved by drawing
\begin{equation}
\label{eq:decomp}
\freq=R\Sigma^{-\frac{1}{2}}\mrho,
\end{equation}
where $\mrho \in \mathbb{R}^d$ is uniformly distributed on the $\ell_2$ unit sphere $\mathcal{S}_{d-1}$, and $R \in \mathbb{R}_+$ is a radius chosen independently from $\mrho$ with a distribution $p_R$ we will now specify.

With the decomposition \eqref{eq:decomp}, the characteristic function $\chrc_{\mtheta}$ is now expressed as
\begin{equation*}
\chrc_{\mtheta}\left(R\Sigma^{-\frac{1}{2}}\mrho\right)=\exp\left(-\imaginaryi R\mrho^T\Sigma^{-\frac{1}{2}}\mmu\right)\exp\left(-\tfrac{1}{2}R^2\right) =\chrc_{\theta}(R),
\end{equation*}
where $\chrc_{\theta}$ is the characteristic function of a \emph{one-dimensional} Gaussian with mean $\mu=\mrho^T\Sigma^{-\frac{1}{2}}\mmu$ and variance $\sigma^2=1$. We thus consider the estimation of a one-dimensional Gaussian $\pp_{\theta_0}=\mathcal{N}(\mu_0,\sigma_0^2)$, with $\sigma_0=1$, as our baseline to design a radius distribution $p_R$.

In this setting, we no longer suffer from unwanted concentration phenomena and can resort to the intuitive Gaussian radius distribution to sample $\chrc_{\theta_0}$. It corresponds to a radius density function $p_R= \mathcal{N}^+(0,\frac{1}{\sigma_0^2})=\mathcal{N}^+(0,1)$ (\ie~Gaussian with absolute value, referred to as \emph{folded} Gaussian). Using this radius distribution with the decomposition \eqref{eq:decomp} yields a frequency distribution $\freqdist^{(FGr)}_\mSigma$ referred to as \emph{Folded Gaussian radius} frequency distribution. Note that, similar to the Gaussian frequency distribution, the Folded Gaussian radius distribution only depends on the (oracle) covariance $\mSigma$ of the sketched distribution $\pp_\mtheta$.

\paragraph{Adapted radius distribution}
Though we will see it yields decent results in Section~\ref{sec:results}, the Folded Gaussian radius frequency distribution somehow produces too many frequencies with a low radius $R$. These carry a limited quantity of information about the original distribution, since all characteristic functions equal $1$ at the origin\footnote{In a way, numerous measures of the characteristic function near the origin essentially measure its derivatives at various orders, which are associated to classical polynomial moments.}.
We now present a heuristics that may avoid this ``waste'' of frequencies.

Intuitively, the chosen frequencies should properly discriminate Gaussians with different parameters, at least in the neighborhood of the true parameter $\theta_0=(\mu_{0},\sigma_{0})=(\mu_0,1)$. 
This corresponds to promoting frequencies $\omega$ leading to a large difference $|\chrc_\theta(\omega)-\chrc_{\theta_0}(\omega)|$
for parameters $\theta$ close to $\theta_0$. A way to achieve this is to promote frequencies where the norm of the gradient $\neucl{\nabla_\theta \chrc_{\theta}(\omega)}$ is large at $\theta = \theta_0$. 

%


Recall that for a one-dimensional Gaussian $\chrc_\theta(\omega)=e^{-\imaginaryi\mu\omega}e^{-\frac{1}{2}\sigma^2\omega^2}$. The norm of the gradient is expressed as:
\begin{align}
\neucl{\nabla_{\theta} \chrc_{\theta}(\omega)}^2&=\left\lvert\nabla_\mu \chrc_{\theta}(\omega)\right\rvert^2 + \left\lvert\nabla_{\sigma^2} \chrc_{\theta}(\omega)\right\rvert^2 
=\left\lvert-\imaginaryi\omega\chrc_{\theta}(\omega)\right\rvert^2 + \left\lvert-\frac{1}{2}\omega^2\chrc_{\theta}(\omega)\right\rvert^2 
=\left(R^2+\frac{R^4}{4}\right)e^{-\sigma^2 R^2} \notag
\end{align}
and therefore
$\neucl{\nabla_{\theta} \chrc_{\theta_0}(\omega)}=\left(R^2+\tfrac{R^4}{4}\right)^{\frac{1}{2}}e^{-\frac{1}{2}R^2}$
since $\sigma^2_0=1$. This expression still has a Gaussian decrease (up to polynomial factors), and indeed avoids very low frequencies. 
It can be normalized to a density function:
\begin{equation}
\label{eq:adapted_density}
p_R=C\left(R^2+\tfrac{R^4}{4}\right)^{\frac{1}{2}}e^{-\frac{1}{2}R^2},
\end{equation}
with $C$ some normalization constant. Using this radius distribution with the decomposition \eqref{eq:decomp} yields a distribution $\freqdist^{(Ar)}_\mSigma$ referred to as \emph{Adapted radius} frequency distribution. Once again, this distribution only depends on the covariance $\mSigma$.


\subsubsection{Oracle frequency sampling pattern for a known mixture of Gaussians}\label{sec:gmmfreqdist}

Any frequency distribution $\freqdist^{(.)}_\mSigma$ selected for sampling the characteristic function of a single known Gaussian $\pp_\mtheta$ can be immediately and naturally extended to a frequency distribution $\freqdist^{(.)}_{\hyppar,\malpha}$ to sample the characteristic function of a known GMM $\pp_{\hyppar,\malpha}$, by mixing the frequency distributions corresponding to each Gaussian:
\begin{equation}
\freqdist^{(.)}_{\hyppar,\malpha}=\sumk \alpha_k \freqdist^{(.)}_{\mSigma_k}.
\end{equation}
Each component $\freqdist^{(.)}_{\mSigma_k}$ has the same weight than its corresponding Gaussian $\pp_{\mtheta_k}$. Indeed, a Gaussian with a high weight must be precisely estimated, as its influence on the overall reconstruction error (\eg~in terms of Kullback-Leibler divergence) is more important than the components with low weights. Thus more frequencies adapted to this Gaussian are selected.

The draw of frequencies with an oracle distribution $\freqdist^{(.)}_{\hyppar,\malpha}$ is summarized in Function \texttt{DrawFreq}.

%
%



\subsubsection{Choosing the frequency sampling pattern in practice}\label{sec:practice_freqdist}

In practice the parameters $(\hyppar,\malpha)$ of the GMM to be estimated are obviously unknown beforehand, so the oracle frequency distributions $\freqdist^{(.)}_{\hyppar,\malpha}$ introduced in the previous section cannot be computed. We propose a simple method to obtain an approximate distribution that yields good results in practice. The reader must also keep in mind that it is very easy to integrate some prior knowledge in this design, especially since the proposed frequency distributions only take into account the variances of the GMM components, not their means.

The idea is to estimate the average variance $\bar\sigma^2=\frac{1}{Kd}\sum_{k=1}^K \sum_{\ell=1}^d \sigma^2_{k,\ell}$ of the components in the GMM -- note that this parameter may be significantly different from the global variance of the data, for instance in the case of well-separated components with small variances. This estimation is performed using a light sketch $\bfz_0$ with $m_0$ frequencies, computed on a small subset of $n_0$ items from the database, then a frequency distribution corresponding to a single isotropic Gaussian $\freqdist^{(.)}_{\bar\sigma^2\bfI}$ is selected.

\begin{figure}[h]
\centering
\includegraphics[height=\hght]{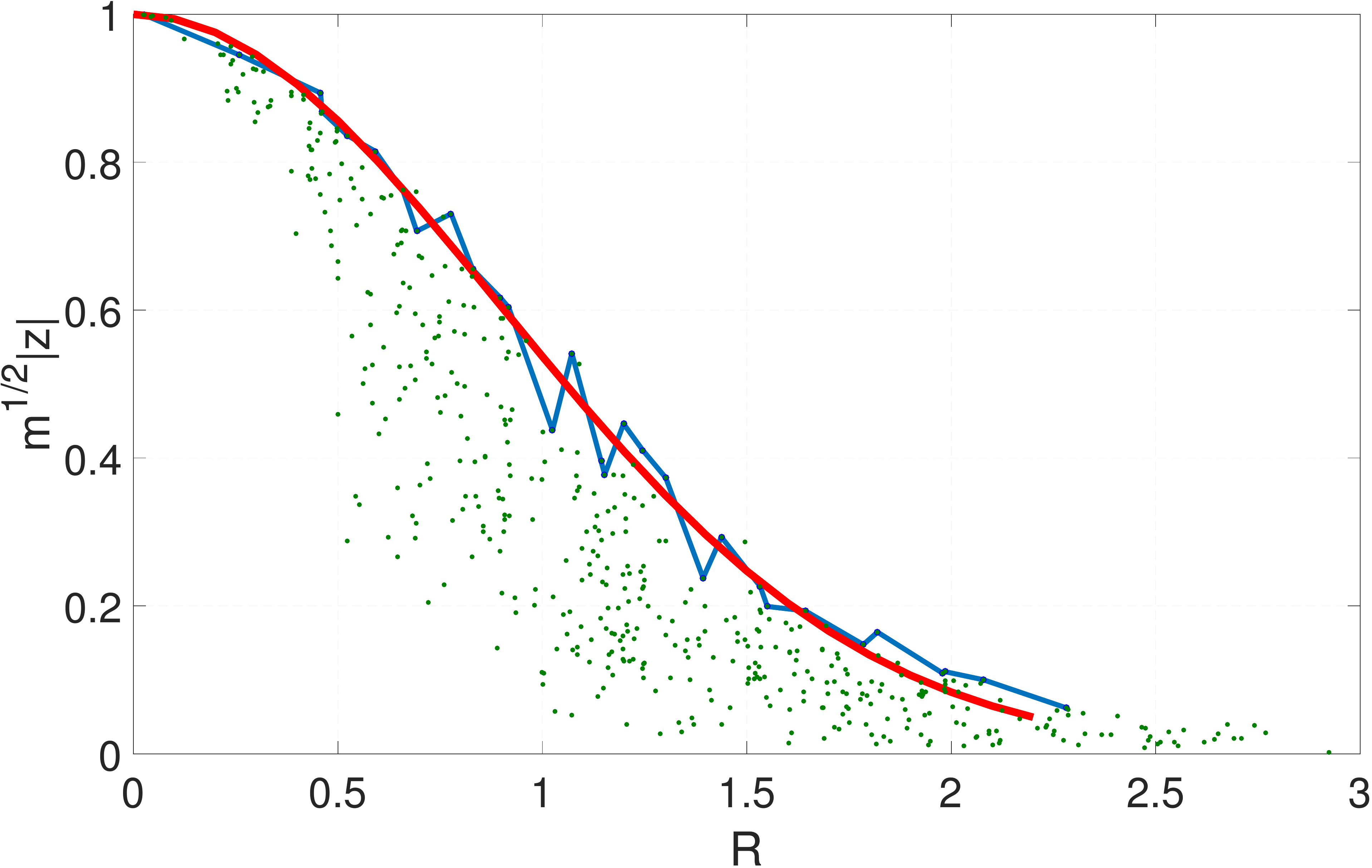}
\caption{Estimation of $\bar\sigma^2$ (Function \texttt{EstimMeanSigma}), for $d=10$, $K=5$, $m_0=500$ and $n_0=5000$. Green: modulus of the sketch with respect to the norm of frequencies  (ordered by increasing radius). Blue: visualization of the peaks in each block of $20$ consecutive values. Red: fitted curve $e^{-\frac{1}{2}R^2\bar\sigma^2}$ for the estimated $\bar\sigma^2$.}
\label{fig:fitsig}
\end{figure}

Indeed, if the variances $\sigma_{k,\ell}^2$'s are not too different from each other, the amplitude of the empirical characteristic function $|\hat\chrc(\freq)|$ \emph{approximately} follows $e^{-\frac{1}{2}\neucl{\freq}^2\bar\sigma^2}$ with high oscillations, allowing for a very simple amplitude estimation process: assuming the $m_0$ frequencies used to compute the sketch $\bfz_0$ are ordered by increasing Euclidean radius, the sketch $\bfz_0$ is divided into consecutive blocks, maximal peaks of its modulus are identified within each block forming a curve that approximately follows $e^{-\frac{1}{2}R^2\bar\sigma^2}$, then a simple regression is used to estimate $\bar\sigma^2$. This process is illustrated in Figure~\ref{fig:fitsig}.

To cope with the fact that the "range" of frequencies that must be considered to compute $\bfz_0$ is also not known beforehand, we initialize $\bar\sigma^2=1$ and reiterate this procedure several times, each time drawing $m_0$ frequencies adapted to the current estimation of $\bar\sigma^2$, \ie~with some choice of frequency distribution $\freqdist^{(.)}_{\bar\sigma^2\bfI}$, and update $\bar\sigma^2$ at each iteration. In practice, the procedure quickly converges in three iterations.

The entire process is summarized in detail in Function \texttt{EstimMeanSigma} and Algorithm \ref{algo:drawfreqpractice}.

\begin{function}
\SetKwBlock{uBegin}{Step}{end}
\KwData{Set of variances and weights of a GMM $\left\lbrace \mSigma_k\right\rbrace_{k=1,...,K}$, $\malpha$, number of frequencies $m$, type of frequency distribution $f\in\{(G),(FGr),(Ar)\}$}
\KwResult{Set of frequencies $\freqs=\{\freq_1,...,\freq_m\}$}
\For{$j\leftarrow 1$ \KwTo $m$}{
Draw a label according to the weights of the GMM $k_j \sim \sum_{k=1}^K \alpha_k \delta_k$\;
\If{$f=(G)$}{
$\freq_j \sim \mathcal{N}\left(0,\mSigma_{k_j}^{-1}\right)$ \tcp*{Gaussian}
}\Else{
Draw a direction $\mrho \sim \mathcal{U}\left(\mathcal{S}_{d-1}\right)$\;
\If{$f=(FGr)$}{
$R\sim \mathcal{N}^+(0,1)$ \tcp*{Folded Gaussian radius}
}\ElseIf{$f=(Ar)$}{
$R\sim p_R$ with $p_R$ defined by \eqref{eq:adapted_density} \tcp*{Adapted radius}
}
$\freq_j \leftarrow R \mSigma_{k_j}^{-\frac12}\mrho$\;
}
}
\caption{DrawFreq($\left\lbrace \mSigma_k\right\rbrace_{k=1,...,K},\malpha,m,f$): drawing frequencies for a GMM with known variances and weights, choosing one of the three distributions described in Section \ref{sec:freqdist}}
 \label{func:drawfreq}
\end{function}

\begin{function}
\SetKwBlock{uBegin}{Step}{end}
\KwData{Dataset $\dataset=\{\bfx_1,...,\bfx_n\}$, small number of items $n_0\leq n$, small number of frequencies $m_0$, number of blocks $c \in \mathbb{N}^*_+$, number of iterations $T$}
\KwResult{Estimated mean variance $\bar\sigma^2$}
\Begin(Initialize){
$\bar\sigma^2 \leftarrow 1$\;
}
\For{$t\leftarrow 1$ \KwTo $T$}{
\Begin(Draw some frequencies adapted to the current $\bar\sigma^2$){
$\{\freq_1,...,\freq_{m_0}\} \leftarrow \texttt{DrawFreq}(\bar\sigma^2\bfI,1,m_0,(Ar))$\;
Sort the frequencies $\{\freq_1,...,\freq_{m_0}\}$ by increasing radius $\|\freq_j\|_2$\;
}
\Begin(Compute small empirical sketch, without $\frac{1}{\sqrt{m_0}}$ normalization $\text{(}$Figure~\ref{fig:fitsig}, left$\text{)}$){$\hat\bfz_0\leftarrow \left[\frac{1}{n_0}\sum_{i=1}^{n_0} e^{-\imaginaryi\freq_{j}^T\bfx_i}\right]_{j=1...m_0}$ \;
}
\Begin(Divide sketch into blocks, find maximum peak in each block $\text{(}$Figure~\ref{fig:fitsig}, center$\text{)}$){$s \leftarrow \lfloor m_0/c \rfloor$\;
\For{$q \leftarrow 1$ \KwTo $c$}{
$j_q=\arg\max_{j \in [(q-1)s+1;qs]} |\hat{z}_{0,j}|$\;
}
}
\Begin(Update $\bar\sigma^2$ $\text{(}$Figure~\ref{fig:fitsig}, right$\text{)}$){$\hat\bfe \leftarrow \left[\hat{z}_{0,j_q}\right]_{q=1...c}$\;
$\bar\sigma^2=\arg\min_{\sigma^2>0} \left\lVert \hat\bfe - \left[e^{-\frac{1}{2}R_{j_q}^2\sigma^2}\right]_{q=1...c}\right\rVert_2$\;
}
}
\caption{EstimMeanSigma($\dataset,n_0,m_0,c,T$): Estimation of the mean variance $\bar\sigma^2$}
 \label{algo:learn_sigma}
\end{function}

\begin{algorithm2e*}
\SetKwBlock{uBegin}{Step}{end}
\KwData{Dataset $\dataset=\{\bfx_1,...,\bfx_n\}$, number of frequencies $m$, type of frequency distribution $f\in\{(G),(FGr),(Ar)\}$, parameters for the estimation of the mean variance $(n_0,m_0,c,T)$}
\KwResult{Set of frequencies $\freqs=\{\freq_1,...,\freq_m\}$}
\Begin(Estimate the mean variance){
$\bar\sigma^2 \leftarrow \EstimMeanSigma(\dataset,n_0,m_0,c,T)$\;
}
\Begin(Draw frequencies with the distribution $\freqdist^{(.)}_{\bar\sigma^2\bfI}$){
$\freqs \leftarrow \DrawFreq(\bar\sigma^2\bfI,1,m,f)$\;
}
\caption{Draw the frequencies in practice.}
 \label{algo:drawfreqpractice}
\end{algorithm2e*}


\subsection{Summary}

At this point, all procedures necessary for compressive GMM estimation have been defined. Given a database $\dataset=\{\bfx_1,...,\bfx_n\}$, a number of measurements $m$ and a number of components $K$, the entire process is as follow:
\begin{itemize}
\item In the absence of prior knowledge, draw $m$ frequencies using Algorithm \ref{algo:drawfreqpractice} on (a fraction of) the dataset. The proposed Adapted radius frequency distribution is recommended, other parameters have default values (see Section \ref{sec:basic_setup}), the effect of modifying them has been found to be negligible.
\item Compute the empirical sketch $\hat\bfz=\frac{1}{\sqrt{m}}\left[\frac{1}{n}\sum_{i=1}^n e^{-\imaginaryi \freq_j^T\bfx_i}\right]_{j=1...m}$. One may use GPU and/or distributed/online computing.
\item Estimate a $K$-GMM using CL-OMP(R) 
or the more scalable but less precise Algorithm \ref{algo:BS-CGMM}.
\end{itemize}

\paragraph{Connections with Distilled sensing.} The reader may note that designing a measurement operator \emph{adapted} to some particular data does not fit the classical paradigm of compressive sensing.

The two-stage approaches used to choose the frequency distribution presented above can be related to a line of work referred to as adaptive (or \emph{distilled}) sensing \cite{Haupt2010}, in which a portion of the computational budget is used to crudely design the measurement operator while the rest is used to actually measure the signal. Most often these methods are extended to multi-stage approaches, where the measurement operator is refined at each iteration, and have been used in machine learning \cite{Cohn1996} or signal processing \cite{Bashan2008}. Allocating the resources and choosing between \emph{exploration} (designing the measurement operator) and \emph{precision} (actually measuring the signal) is a classic trade-off in areas such as reinforcement learning or game theory.



\section{Kernel design and sketching}
\label{sec:kernel}

It turns out that sketching operators as \eqref{eq:skop} are intimately related to RKHS embedding of probability distributions \cite{Smola2007,Sriperumbudur2010} and RFFs \cite{Rahimi2007}.
The proposed, carefully-designed choice of frequency distribution appears as an innovative method to design a reproducing kernel, which is faster and provides better results than traditional choices in the kernel literature \cite{Sutherland2015}, as experimentally shown in Section \ref{sec:results_freqdist}. Additionally, approximate RKHS embedding of distributions turns out to be an appropriate framework to derive the information preservation guarantees of Section \ref{sec:theory}.

\subsection{Reproducing Kernel Hilbert Spaces (RKHS)}

We refer the reader to Appendix \ref{sec:def} for definitions related to positive definite (p.d.) kernels and measures.

Let $\kernel: \Xspace \times \Xspace \rightarrow \mathbb{C}$ be a p.d. kernel. By Moore-Aronszajn Theorem \cite{Aronzajn1950}, to this kernel is associated a unique Hilbert space $\rkhs \subset \mathbb{C}^{\Xspace}$ that satisfies the following properties: for any $\bfx \in \Xspace$ the function $\kernel(\bfx,.)$ belongs to $\rkhs$, and the kernel satisfies the reproducing property $\forall f \in \rkhs, \forall \bfx \in \Xspace, ~ \ipH{\kernel(\bfx,.)}{f}=f(\bfx)$. The space $\rkhs$ is referred to as the Reproducing Kernel Hilbert Space (RKHS) associated with the kernel $\kernel$. We denote $\left\langle.,.\right\rangle_\rkhs$ the scalar product of the RKHS $\rkhs$. We refer the reader to \cite{Hofmann2006} and references therein for a review of RKHS and kernel methods. 

We focus here on the space $\Xspace=\mathbb{R}^d$ and on \emph{translation-invariant} kernels of the form
\begin{equation}\label{eq:TIk}
\kernel(\bfx,\bfy)=\TIk(\bfx-\bfy),
\end{equation}
where $\TIk:\mathbb{R}^d\rightarrow\mathbb{C}$ is a positive definite function. Translation-invariant positive definite kernels are characterized by the Bochner theorem:
\begin{theorem}[Bochner \cite{Rudin1962}, Thm. 1.4.3]\label{thm:Bochner}
A continuous function $\TIk:\mathbb{R}^d \rightarrow \mathbb{C}$ is positive definite if and only if it is the Fourier transform of a finite (\ie, $\freqdist(\mathbb{R}^{d}) = \int_{\mathbb{R}^{d}} d\freqdist(\freq) < \infty$) nonnegative measure $\freqdist$ on $\mathbb{R}^d$, that is:
\begin{equation}\label{eq:freqdist}
\TIk(\bfx)=\int_{\mathbb{R}^n} e^{-\imaginaryi\freq^T\bfx} d\freqdist(\freq).
\end{equation}
\end{theorem}
This expression implies the normalization $|\kernel(\bfx,\bfy)|\leq |\kernel(\bfx,\bfx)|=\freqdist(\mathbb{R}^{d})$. Hence, without loss of generality, up to a scaling factor, we suppose $\freqdist(\mathbb{R}^{d})=1$. This means in particular that $\freqdist$ is a probability distribution on $\mathbb{R}^d$.

\subsection{Hilbert Space Embedding of probability distributions}
Embeddings of probability distributions \cite{Smola2007,Sriperumbudur2010} are obtained by defining the Mean Map $\mmap:\pmeas \rightarrow \rkhs$:
\begin{equation}
\mmap(\pp)=\mathbb{E}_{\bfx \sim \pp}(\kernel(\bfx,.)).
\end{equation}
Using the Mean Map we can define the following p.d. kernel $\mathcal{K}(\pp,\qq)=\ipH{\mmap(\pp)}{\mmap(\qq)}$ over the set of probability distributions $\pmeas$. Note that this expression is also equivalent to the following definition \cite{Sriperumbudur2010,Muandet2012}: $\mathcal{K}(\pp,\qq)=\mathbb{E}_{\bfx\sim \pp,\bfy\sim \qq}(\kernel(\bfx,\bfy))$.

We denote $\distfun{\pp}{\qq}{\kernel}=\|\mmap(\pp)-\mmap(\qq)\|_\rkhs$ the pseudometric induced by the Mean Map on the set of probability distributions, often referred to as Maximum Mean Discrepancy (MMD) \cite{Sutherland2015}. Fukumizu \etal~introduced in \cite{Fukumizu2008} the concept of \emph{characteristic kernel}, that is, a kernel $\kernel$ for which the map $\mmap$ is injective and the pseudometric $\distfuns{\kernel}$ is a true metric. Translation-invariant characteristic kernels are characterized by the following Theorem \cite{Sriperumbudur2010}.
\begin{theorem}[Sriperumbudur \etal~\cite{Sriperumbudur2010}]\label{thm:characTIk}
Assume that $\kernel(\bfx,\bfy)=\TIk(\bfx-\bfy)$ where $\TIk$ is a positive definite function on $\mathbb{R}^d$. Then $\kernel$ is characteristic if and only if $\supp(\freqdist)=\mathbb{R}^d$, where $\freqdist$ is defined as \eqref{eq:freqdist}.
\end{theorem}
Many classical translation-invariant kernels (Gaussian, Laplacian....) indeed exhibit a Fourier transform with a support that is equal to $\mathbb{R}^d$, and are therefore characteristic. This is also the case for all kernels corresponding to the proposed frequency distributions $\freqdist^{(.)}_\mSigma$ that we defined in Section \ref{sec:freqdist}.

\subsection{From Kernel design to the design of a frequency sampling pattern}\label{sec:kernel_design}

In the case of a translation-invariant kernel, the metric $\distfuns{\kernel}$ can be expressed as $\distfunsq{\pp}{\qq}{\kernel} = \normKsq{\pp}{\qq}$ where by abuse of notation, for a given frequency distribution $\freqdist$, we introduce
\begin{align}
\normKsq{\pp}{\qq} &:=\int_{\mathbb{R}^d}\left\lvert\chrc_\pp(\freq)-\chrc_\qq(\freq)\right\rvert^2 d\freqdist(\freq), \label{eq:chrcdiff}
\end{align}
where we recall that $\chrc_\pp(\freq)$ is the characteristic function of $\pp$. 
The proof of Theorem \ref{thm:characTIk} is actually based on this reformulation \cite{Sriperumbudur2010}.

Hence, given $m$ frequencies $(\freq_1,...,\freq_m) \overset{i.i.d.}{\sim} \freqdist$ and the corresponding sketching operator $\skop$ \eqref{eq:skop}, we can expect that for large enough $m$, with high probability,
\begin{align}
\normKsq{\pp}{\qq} = \mathbb{E}_{\freq\sim\freqdist}|\chrc_\pp(\freq)-\chrc_\qq(\freq)|^2 \approx \frac{1}{m} \sum_{j=1}^m |\chrc_\pp(\freq_{j})-\chrc_\qq(\freq_{j})|^2 = \neucl{\skop \pp - \skop \qq}^2 \label{eq:emp_norms}.
\end{align}

Building on this relation between the metric $\normKsq{\cdot}{\cdot}$ and the sketching operators considered in this paper, we next derive some connections with kernel design. We further exploit these connections to draw a theoretical analysis of the information preservation guarantees of the sketching operator $\skop$ in Section~\ref{sec:theory}.

\paragraph{Traditional kernel design.}
Given some learning task, the selection of an appropriate kernel --known as kernel design-- is a difficult, open problem, usually done by cross-validation. Even when using RFFs to leverage the computational gain of explicit embeddings (compared to potentially costly implicit kernels), the considered frequency distribution $\freqdist$ is usually \emph{derived from} a translation-invariant kernel $\kernel$ chosen in a parametric family endowed with closed-form expressions for both $\kernel$ and $\freqdist$ \cite{Rahimi2007,Sutherland2015}. A typical example is the Gaussian kernel (chosen for the simplicity of its Fourier transform), which bandwidth is often selected by cross-validation \cite{Sriperumbudur2009,Sutherland2015}.

\review{
\paragraph{Spectral kernel design: designing the frequency sampling pattern.}
Learning an appropriate kernel by directly deriving a spectral distribution of frequencies for Random Fourier Features has been an increasingly popular idea in the last few years. In the field of usual reproducing kernels on finite-dimensional objects, researchers have explored the possibility of modifying the matrix of frequencies to obtain a better approximation quality \cite{Xinnan2016} or to accelerate the computation of the kernel \cite{Le2013}. Both ideas have been exploited for learning an appropriate frequency distribution, often modeled as a mixture of Gaussians \cite{Wilson2013,Yang2015,Oliva2015a} or by optimizing weights over a combination of many distributions \cite{Sinha2016}.

In the context of using the Mean Map with Random Features, learning a kernel has been often explored for the two-sample test problem, mainly based on the idea of maximizing the discriminative power of the MMD \cite{Sriperumbudur2009}. Similar to our approach, such methods often divide the database in two parts to learn the kernel then perform the hypothesis test \cite{Gretton2012,Jitkrittum2016}, or are done in a streaming context \cite{Paige2016}. Variants of the Random Fourier Features have also been used \cite{Chwialkowski2015}.

Compared to these methods, the approach proposed in Section~\ref{sec:freqdist} is relatively simple, fast to perform, and based on an approximate theoretical expression of the MMD for clustered data instead of a thorough statistical analysis of the problem. In the spirit of our sketching framework, it only requires a tiny portion of the database and is performed in a \emph{completely unsupervised} manner, which is quite different from most existing literature. Furthermore, in this paper we only consider traditional Random Fourier Features for translation-invariant kernels. 

Using more exotic kernels and/or adapting more advanced learning methods to our framework is left for future work. Still, in Section \ref{sec:results_freqdist}, we empirically show that the estimation of $\bar\sigma^2$ (Function~\EstimMeanSigma) is much faster than estimation by cross-validation, and that the Adapted radius distribution $\freqdist^{(Ar)}_{\bar\sigma^2\bfI}$ performs better than a traditional Gaussian kernel with optimized bandwidth.
}

\section{Experiments with synthetic data}\label{sec:results}

To validate the proposed framework, the algorithms CL-OMP(R) (Algorithm~\ref{algo:CL-OMP}) are first extensively tested against synthetic problems for which the true parameters of the GMM are known. Experiments on real data will be conducted in Section~\ref{sec:speaker_verification}, with the additional analysis of 
Algorithm~\ref{algo:BS-CGMM}. \review{The full Matlab code is available at \cite{Keriven2016toolbox}.}
\subsection{Generating the data}\label{sec:data_drawing}

When dealing with synthetic data for various settings, it is particularly difficult to generate problems that are ``equally difficult'' when varying the dimension $d$ and the level of sparsity $K$. We opted for a simple heuristic to draw $K$ Gaussians neither too close nor to far from one another, given their variances.

Since an $d$-dimensional ball with radius $r$ has volume $D_d r^d$ (with $D_d$ a constant depending on $d$), we consider that any Gaussian which is \emph{approximately isotropic} with variance $\sigma^2$ (\ie, $\mathcal{N}(0,\mathbf{\Sigma})$ with $\mathbf{\Sigma} \approx \sigma^{2}\bfI$) ``occupies'' a volume
$V_{\sigma^2} 
=\sigma^d V_1,
$
where $V_1$ is a reference volume for $\sigma=1$.

Variances $\sigma_{k,\ell}^2$ are randomly drawn uniformly from $0.25$ to $1.75$, so that $\mathbb{E}(\bar\sigma^2)=1$.  The means $\mmu$ are chosen so that they lie in a ball sufficiently large to accommodate for a volume $K \times V_{\bar\sigma^2}$. We therefore choose $\mmu \sim \mathcal{N}(0,\sigma_{\mmu}^2 \bfI)$ with $\sigma_\mmu$ that verifies $V_{\sigma_{\mmu}^2}=KV_{\bar\sigma^2}$, which yields
$\sigma_{\mmu}=K^{\frac{1}{d}}\bar\sigma,  
$
\ie, $\sigma_{\mmu}=K^{\frac{1}{d}}$ by considering the expected value of $\bar\sigma$. In practice this choice seems to offer problems that are neither too elementary nor too difficult.

\subsection{Evaluation measure}\label{sec:evaluation_measure}

To evaluate reconstruction performance when the true distribution $p$ of the data is known, one can resort to the classic Kullback-Leibler divergence $D(\pp||\qq)$ \cite{Cover2006}. 
A symmetric version of the KL-divergence is more comfortable to work with: $d(\pp,\bar \pp)=D(\pp||\bar \pp)+D(\bar \pp || \pp)$, where $\pp$ is the true distribution and $\bar \pp$ is the estimated one (we still refer to this measure as ``KL-divergence''). In our framework, $\pp$ and $\bar \pp$ are GMMs with density functions denoted $\dens$ and $\bar \dens$, hence as in \cite{Bourrier2013} to estimate the KL-divergence in practice we draw $(\bfy_1,...,\bfy_{n'}) \overset{i.i.d.}{\sim} \pp$ with $n'=5\times 10^5$ and compute
\begin{equation}
d(\pp,\bar \pp)\approx \frac{1}{n'} \sum_{i=1}^{n'} \left[\ln\left(\frac{\dens(\bfy_i)}{\bar \dens(\bfy_i)}\right) + \frac{\bar \dens(\bfy_i)}{\dens(\bfy_i)}\ln\left(\frac{\bar \dens(\bfy_i)}{\dens(\bfy_i)}\right)\right].
\end{equation}


\subsection{Basic Setup}\label{sec:basic_setup}
The basic setup is the same for all following experiments, given data $(\bfx_1,...,\bfx_n) \in \mathbb{R}^d$ and parameters $m,K$. We suppose the data to be approximately centered (see Section \ref{sec:data_drawing}).

First, unless specified otherwise (\eg, in 
Section~\ref{sec:results_freqdist}
), we draw frequencies according to an Adapted radius distribution $\freqdist^{(Ar)}_{\bar\sigma^2\bfI}$, using Algorithm \ref{algo:drawfreqpractice} with parameters $m_0=500$, $n_0=\min(n,5000)$, $c=30$, $T=5$. The empirical sketch $\hat\bfz$ is then computed.

The compressive algorithms are performed with their respective number of iterations. For Step 1 of CL-OMP(R), the gradient descent is initialized with a centered isotropic Gaussian with random variance\footnote{After trying many variants of initialization strategy, we came to the conclusion that the algorithm is very robust to initialization. This is one of the possible choices.} $\sigma^2 \sim \mathcal{U}\left([0.5\bar\sigma^2;1.5\bar\sigma^2]\right)$. Furthermore, during all optimization steps in CL-OMP(R) or Algorithm \ref{algo:BS-CGMM}, we constrain all variances to be larger than a small $10^{-15}$ for numerical stability. \review{All continuous optimization schemes in CL-OMP(R) are performed with Stephen Becker's adaptation of the L-BFGS-B algorithm \cite{Byrd1995} in C, with Matlab wrappers \cite{Becker2013}.} 

We compare our results with an adaptation of the previous IHT algorithm \cite{Bourrier2013,Bourrier2015} for isotropic Gaussians with fixed variances, in which all optimization steps have been straightforwardly adapted for our non-isotropic framework with relaxed variances. The IHT is performed with $10$ iterations, which is the default value in the original paper.

We use the VLFeat toolbox \cite{Vedaldi2010} to perform the EM algorithm with diagonal variances. The algorithm is repeated 10 times with random initializations and the result yielding the best log-likelihood is kept. Each run is limited to 100 iterations.

All following reconstruction results are computed by taking the geometric mean over 50 experiments (\ie~the regular mean in logarithmic scale).

\subsection{Results: role of the choice of frequency distribution}\label{sec:results_freqdist}

\begin{table}
\centering
\begin{tabular}{|c|c||c|c|c|}
\cline{3-5}
\multicolumn{2}{c|}{}  & (G) & (FGr) & (Ar)  \\ \hline
\multirow{2}{*}{$d=2,~K=3$} & \multicolumn{1}{|c|}{$\freqdist_{\hyppar_0,\malpha_0}^{(.)}$} & $-9.30$ & $-8.90$ & $\mathbf{-9.38}$ \\ \cline{2-5}
 & \multicolumn{1}{|c|}{$\freqdist_{\bar\sigma^2\bfI}^{(.)}$} & $-8.93$ & $-8.67$ & $\mathbf{-9.20}$  \\ \hline \hline
\multirow{2}{*}{$d=20,~K=5$} & \multicolumn{1}{|c|}{$\freqdist_{\hyppar_0,\malpha_0}^{(.)}$} & $3.15$ & $-6.17$ & $\mathbf{-6.48}$ \\ \cline{2-5}
& \multicolumn{1}{|c|}{$\freqdist_{\bar\sigma^2\bfI}^{(.)}$} & $3.41$ & $-5.80$ & $\mathbf{-6.32}$  \\ \hline
\end{tabular}
\caption{Log-KL-divergence on synthetic data using the CL-OMPR algorithm, for $m=10(2d+1)K$ frequencies and $N=300\ 000$ items. We compare the three proposed frequency distributions: Gaussian \cite{Bourrier2013} (G), Folded Gaussian radius (FGr) or Adapted radius (Ar), using either the oracle distribution defined in Section \ref{sec:gmmfreqdist} or the approximate distribution used in practice, learned with \texttt{\ref{algo:learn_sigma}}.}
\label{tab:freq}
\end{table}

In this section we compare different choices of frequency distributions. We draw $N=300000$ items in two settings, respectively low and high dimensional: $d=2,~K=3$ and $d=20,~K=5$. In each setting we construct the sketch with $m=10K(2d+1)$ frequencies (see Section \ref{sec:results_phase}). Reconstruction is performed with CL-OMPR. 

In Table \ref{tab:freq}, we compare the three frequency distributions introduced in Section \ref{sec:freq}, both with the 
oracle frequency distribution $\freqdist^{(.)}_{\hyppar_{0}, \malpha_{0}}$ (\ie~using Function \DrawFreq with the true parameters of the GMM) -- we remind the reader that this setting unrealistically assumes that the variances and weights of the GMM are known beforehand --, and with the approximate one $\freqdist^{(.)}_{\bar\sigma^2 \bfI}$ (Algorithm \ref{algo:drawfreqpractice}). The results show that the Gaussian frequency distribution indeed yields poor reconstruction results in high dimension ($d=20$), while the Adapted radius frequency distribution outperforms the two others. The use of the approximate $\freqdist^{(.)}_{\bar\sigma^2 \bfI}$ instead of the 
oracle $\freqdist^{(.)}_{\hyppar_{0}, \malpha_{0}}$ is shown to have little effect.

\begin{table}
\centering
\begin{tabular}{|c|c||c|r|}
\cline{3-4}
\multicolumn{2}{c|}{}  & Reconstruction result & Est. time \\ \hline
\multirow{2}{*}{$d=2,~K=3$} & \multicolumn{1}{|c|}{Adapted radius (proposed)} & $\mathbf{-9.38}$ & ($\bar\sigma^2$)~~ $1.02s$ \\ \cline{2-4}
 & \multicolumn{1}{|c|}{Gaussian kernel \cite{Rahimi2007,Sutherland2015}} & $-9.14$ & ($\sigma_{best}^2$) $23.10s$  \\ \hline \hline
\multirow{2}{*}{$d=20,~K=5$} & \multicolumn{1}{|c|}{Adapted radius (proposed)} & $\mathbf{-6.32}$ & ($\bar\sigma^2$)~~ $0.98s$ \\ \cline{2-4}
& \multicolumn{1}{|c|}{Gaussian kernel \cite{Rahimi2007,Sutherland2015}} & $-5.18$ & ($\sigma_{best}^2$) $74.10s$  \\ \hline
\end{tabular}
\caption{Log-KL-divergence results on synthetic data using the CL-OMPR algorithm, for $m=10(2d+1)K$ frequencies and $n=300\ 000$ items, comparing the proposed Adapted radius distribution $\freqdist^{(Ar)}_{\bar\sigma^2\bfI}$ and a frequency distribution $\freqdist^{(Gk)}_{\sigma^2_{best}}$ corresponding to a Gaussian kernel (Gk) \cite{Rahimi2007,Sutherland2015}, with a bandwidth $\sigma_{best}^2$ selected among $15$ values exponentially spread from $10^{-1}$ to $10^2$ yielding the best result in each case. Estimation times for parameters $\bar\sigma^2$ and $\sigma^2_{best}$ are also given.}
\label{tab:freqker}
\end{table}

In Table \ref{tab:freqker} we compare the proposed Adapted radius frequency distribution with a frequency distribution $\freqdist=\mathcal{N}(0,\bfI\sigma^{-2}_{best})$ corresponding to a Gaussian kernel $\kernel(\bfx,\bfy) = \exp\left(-\tfrac{\|\bfx-\bfy\|_2^2}{2\sigma^2_{best}}\right)$ \cite{Rahimi2007,Sutherland2015}, where the bandwidth $\sigma^2_{best}$ is selected among $15$ values exponentially spread from $10^{-1}$ to $10^2$ to yield the best reconstruction results in each setting. It is seen that the Adapted radius distribution $\freqdist^{(Ar)}_{\bar\sigma^2}$ outperforms the Gaussian kernel in each case, and the estimation of the parameter $\bar\sigma^2$ is significantly faster than the tedious learning of the bandwidth $\sigma^2_{best}$.

\review{We conduct an experiment to determine how robust is the isotropic distribution of frequencies $\freqdist^{(Ar)}_{\bar{\sigma}^2\bfI}$ when treating strongly non-isotropic data. We first generate a GMM in dimension $d=10$ with $K=10$ components, with the process described in Section \ref{sec:basic_setup} but with identity covariances. Then the first five dimensions of the parameters are divided by a factor $A>0$, meaning that: for $\ell=1,...,5$ and $k=1,...,K$, we do $\mu_{k,\ell} \leftarrow \mu_{k,\ell}/A$ and $\sigma^2_{k,\ell} \leftarrow \sigma_{k,\ell}^2/A^2$. It simulates data which do not have the same range in each dimension (Fig. \ref{fig:non_isotropy}, top). In Fig. \ref{fig:non_isotropy} (bottom), for increasing values of $A$ we compare CL-OMPR with the learned frequency distribution $\freqdist^{(Ar)}_{\bar{\sigma}^2\bfI}$, CL-OMPR using the ideal frequency distribution $\freqdist^{(Ar)}_{\hyppar_0,\malpha_0}$, and EM. It is indeed seen that, in terms of KL-divergence (left), the results produced by CL-OMPR with the learned frequency distribution deteriorates as the non-isotropy of the GMM increases. On the contrary, CL-OMPR with the oracle frequency distribution and EM do not seem to be affected. Hence in the first case the problem lies with the learned frequency distribution and not the CL-OMPR algorithm itself. To further confirm this fact, we examine the MMD $\gamma_\freqdist$ with the corresponding frequency distribution $\freqdist$ in each case. It is seen that CL-OMPR with the learned frequency distribution performs as well as the oracle one, meaning that the MMD defined with an isotropic choice of frequencies $\gamma_{\freqdist^{(Ar)}_{\bar{\sigma}^2\bfI}}$ is indeed not adapted for strongly non-isotropic problems, and that despite the ability of CL-OMPR to approximately minimize the cost function \eqref{eq:costfun} (which approximates the MMD, see \eqref{eq:emp_norms}), it does not produce good results.

In that particular case, the problem could be potentially resolved by a whitening of the data, e.g. by computing the empirical covariance of the fraction of the database used for the frequency distribution design phase, and multiplying each datapoint by its inverse during the sketching phase. However there are of course cases where the proposed methods would be further challenged, for instance if components are flat in a dimension but far apart: the global covariance of the data along that dimension would be large, even if the variance of each Gaussian components is small (these cases are however rarely encountered in practice). As mentioned before (see Sec. \ref{sec:kernel_design}), another solution would be to use more advanced methods for designing the MMD $\gamma_\freqdist$ than the proposed simple one. Overall, we leave treatment of strongly non-isotropic data for future work.}

Nevertheless, from now on, all experiments are performed with an approximate Adapted radius distribution $\freqdist^{(Ar)}_{\bar\sigma^2 \bfI}$.

\begin{figure}[h]
\centering
\includegraphics[height=4cm]{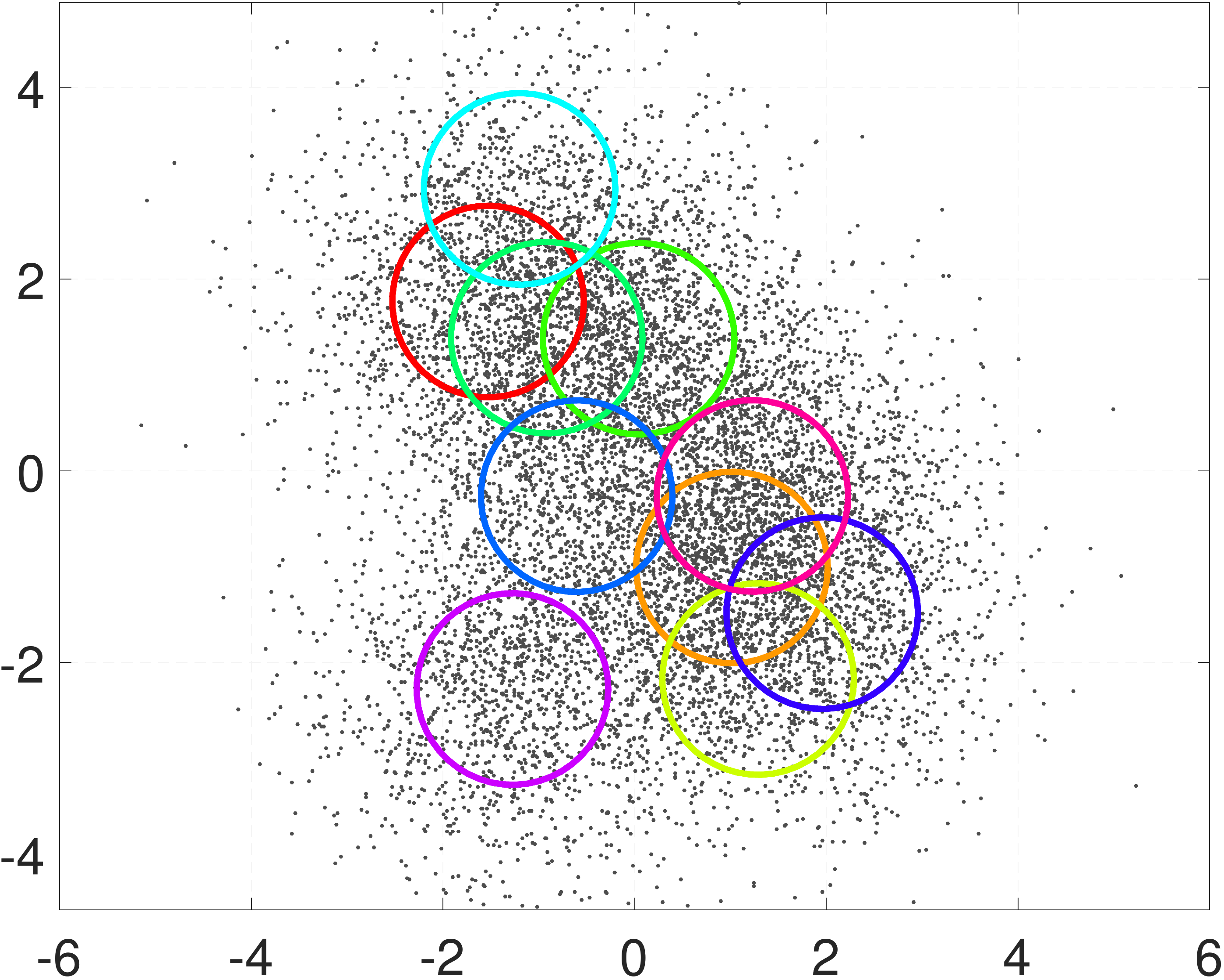}
\includegraphics[height=4cm]{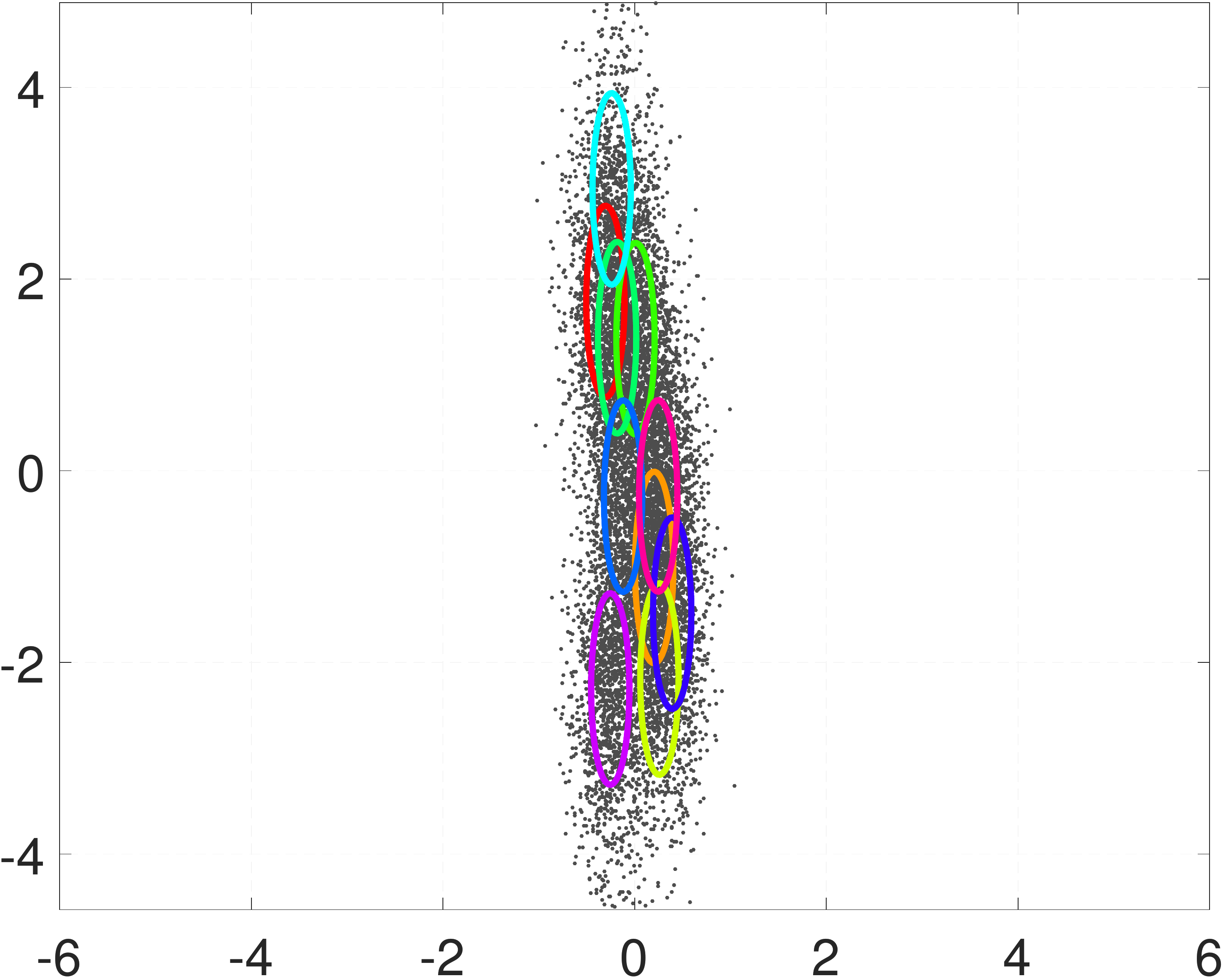}\\
\includegraphics[height=\hght]{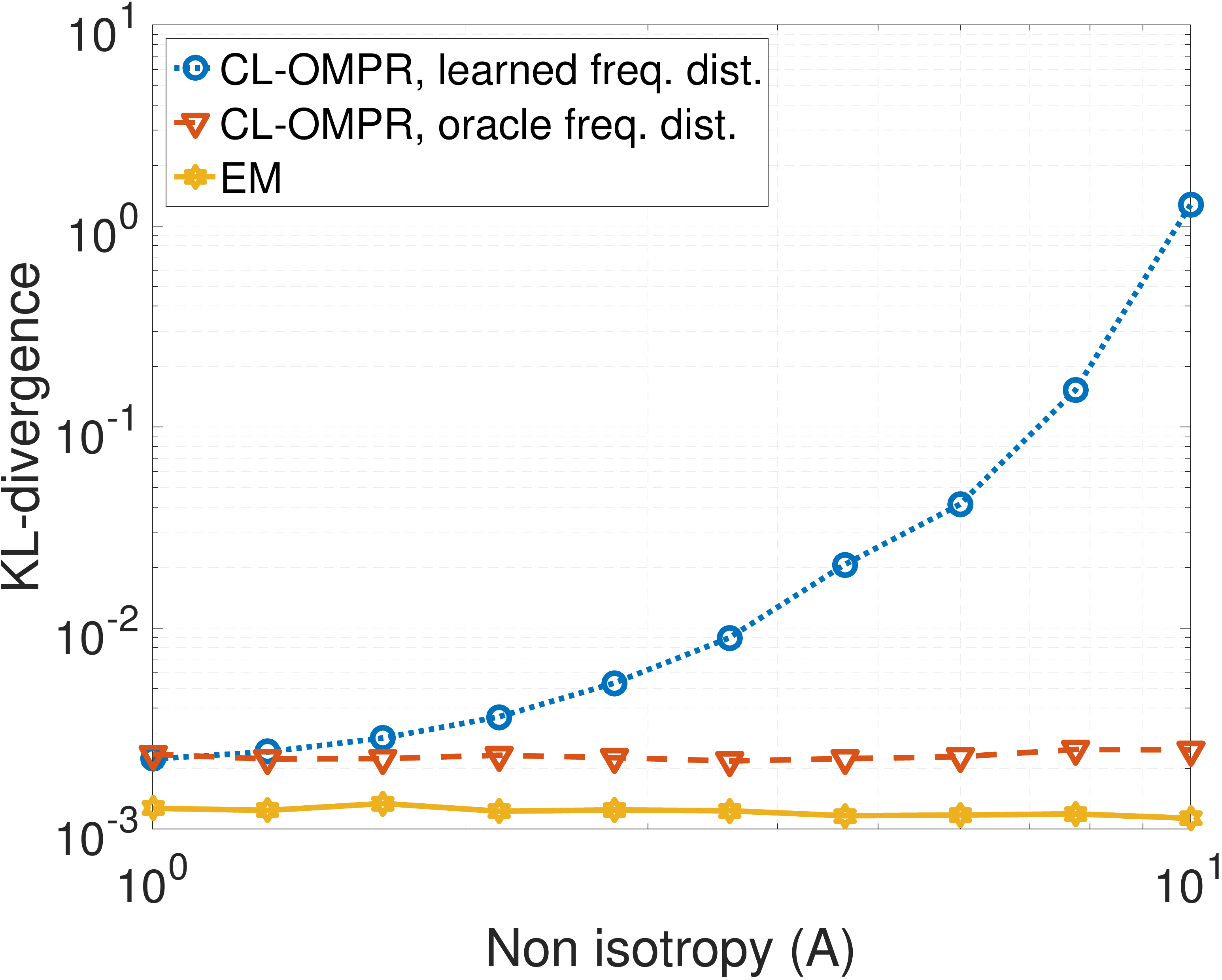}
\includegraphics[height=\hght]{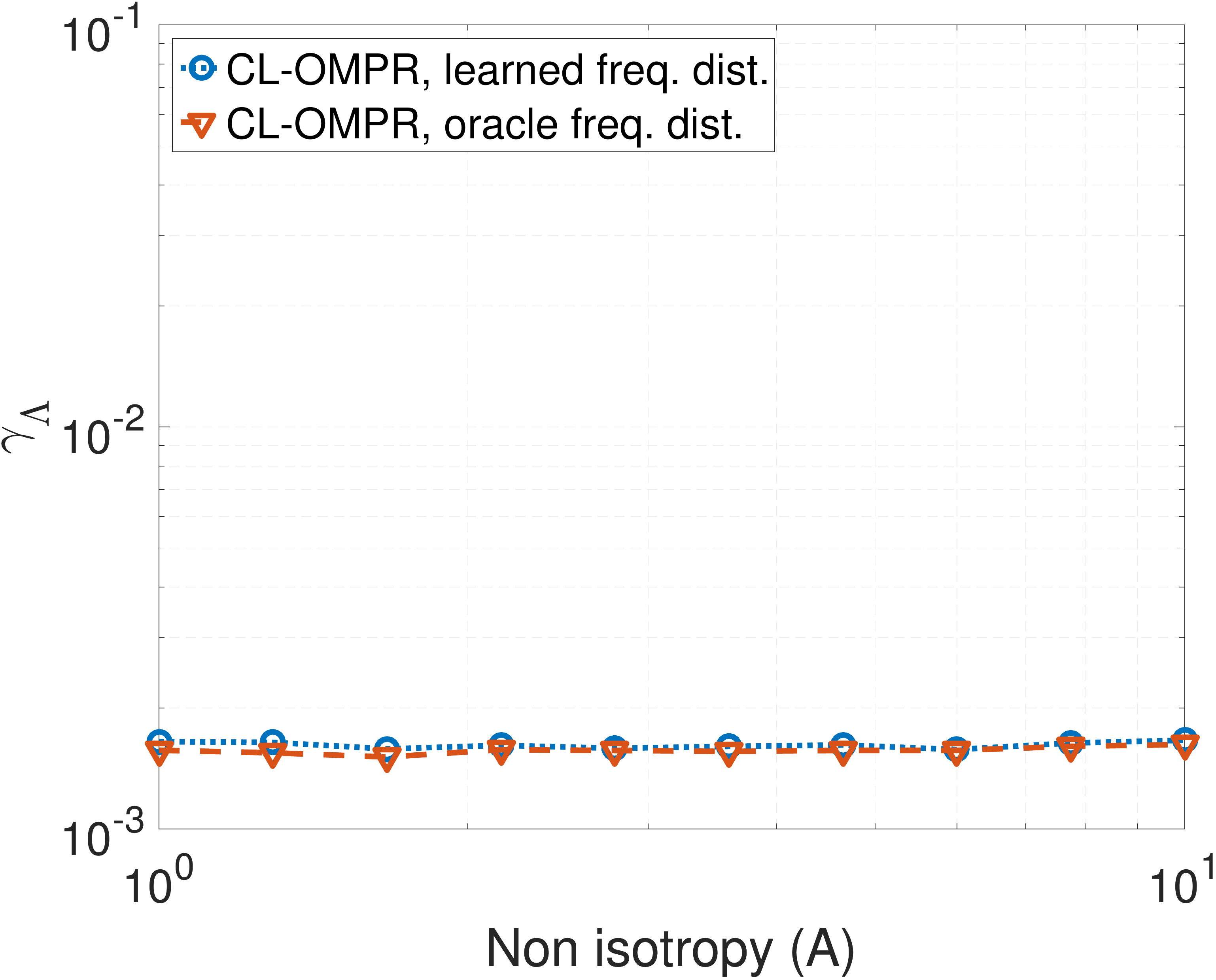}
\caption{\review{Top: isotropic GMM in dimension $n=10$ with $K=10$ components displayed along the first two dimensions (left), the same GMM with one dimension divided by $A=5$ (right). Bottom: reconstruction results for the KL-divergence (left) or MMD (right) with $n=2.10^5$ items, using $m=2000$ frequencies for CL-OMPR, with respect to the coefficient $A$.}}
\label{fig:non_isotropy}
\end{figure}

\subsection{Results: number of measurements}\label{sec:results_phase}

\review{We now evaluate the quality of the reconstruction with respect to the number $m$ of frequencies, for varying dimension $d$ and number of components $K$ in the GMM (Figure~\ref{fig:phasetrans}). 

It is seen that the KL-divergence 
exhibits a sharp phase-transition with respect to $m$, whose occurrence seems to be proportional to the ``dimension'' of the problem, \ie~the number of free parameters $(2d+1)K$. 
This phenomenon is akin to usual compressive sensing. 
In light of this observation, the number of frequencies in all following experiments is chosen as $m=5(2d+1)K$, beyond the empirically obtained phase transition.}


\begin{figure}[h]
\centering
\includegraphics[height=\hght]{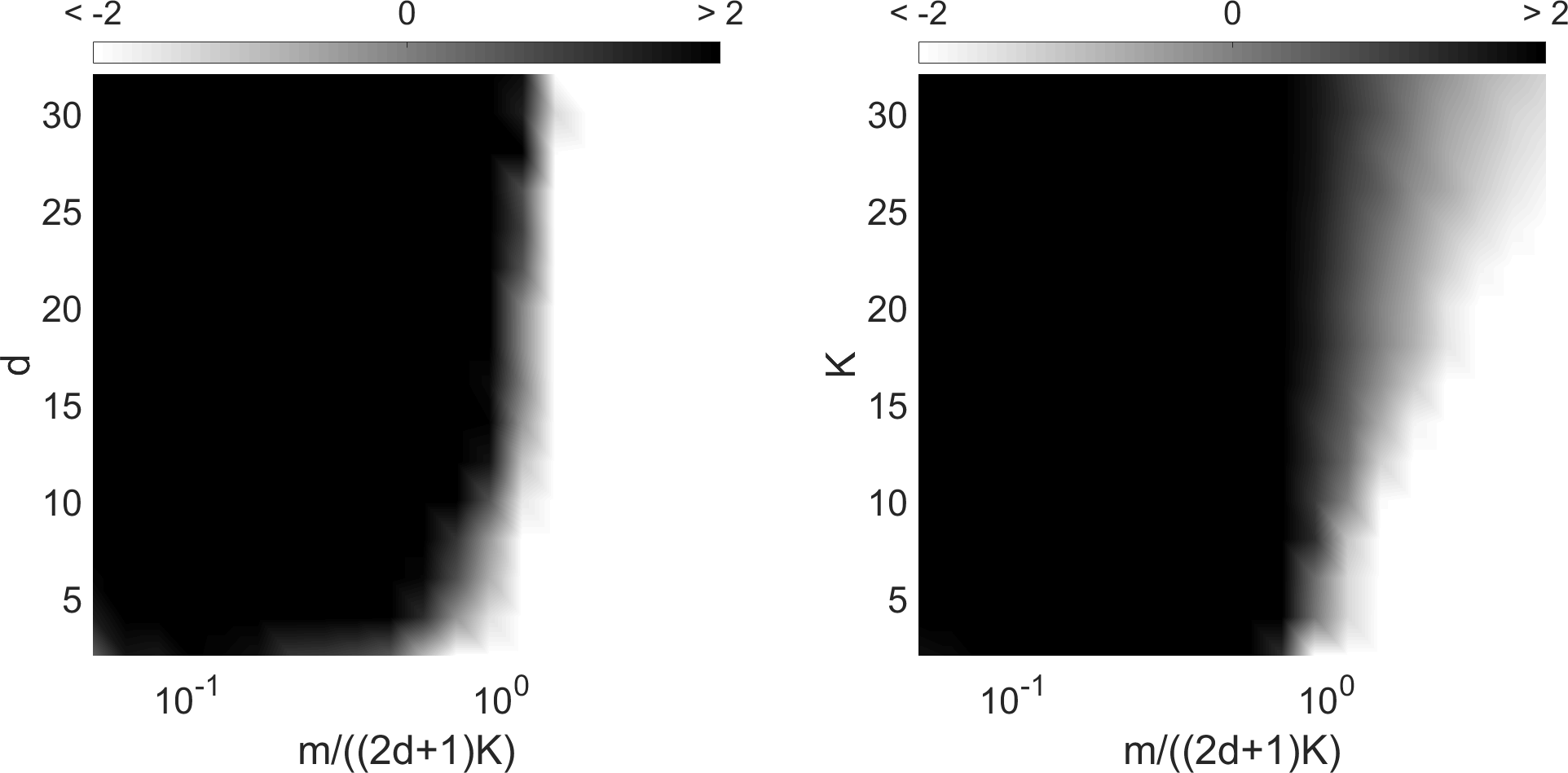}
\caption{\review{Log KL-divergence reconstruction results for CL-OMPR with respect to the normalized number of frequencies $m/((2d+1)K)$ and the dimension $d$ (left) or number of components $K$ (right), using $n=300\ 000$ items. On the left $K=10$, and on the right $d=10$.}}
\label{fig:phasetrans}
\end{figure}

\subsection{Results: comparison of the algorithms}\label{sec:results_N}

We compare the compressive algorithms and EM in terms of reconstruction, computation time and memory usage, with respect to the number of samples $N$.

\paragraph{Precision of the reconstruction} In Figure~\ref{fig:result}, KL-divergence 
reconstruction results are shown for EM and all compressed algorithms: IHT \cite{Bourrier2013}, CL-OMP (Algorithm~\ref{algo:CL-OMP} with $T=K$), CL-OMPR (Algorithm~\ref{algo:CL-OMP} with $T=2K$) and Algorithm~\ref{algo:BS-CGMM}. We consider two settings, one with low $K=5$ (left) and one with high $K=20$ (right), in dimension $n=10$. The number of measurements is set at $m=5(2d+1)K$ in each case.

With few Gaussians ($K=5$), all compressive algorithms yield similar results, close to those achieved by EM. The precision of the reconstruction is seen to improve steadily with the size $n$ of the database. With more Gaussians ($K=20$), CL-OMPR clearly outperforms the other compressive algorithms, and even outperforms EM for very large $n$.

The IHT algorithm \cite{Bourrier2013} adapted to non-isotropic variances often fails to recover a satisfying GMM. Indeed, IHT fills the support of the GMM at the very first iteration, and is seen to converge toward a local minimum of \eqref{eq:costfun}, in which all Gaussians in the GMM are equal to the same large Gaussian that encompasses all data. Note that this situation could not happen in \cite{Bourrier2013}, where all Gaussians have fixed, known variance.

%
%

\begin{figure}
\centering
\includegraphics[height=\hght]{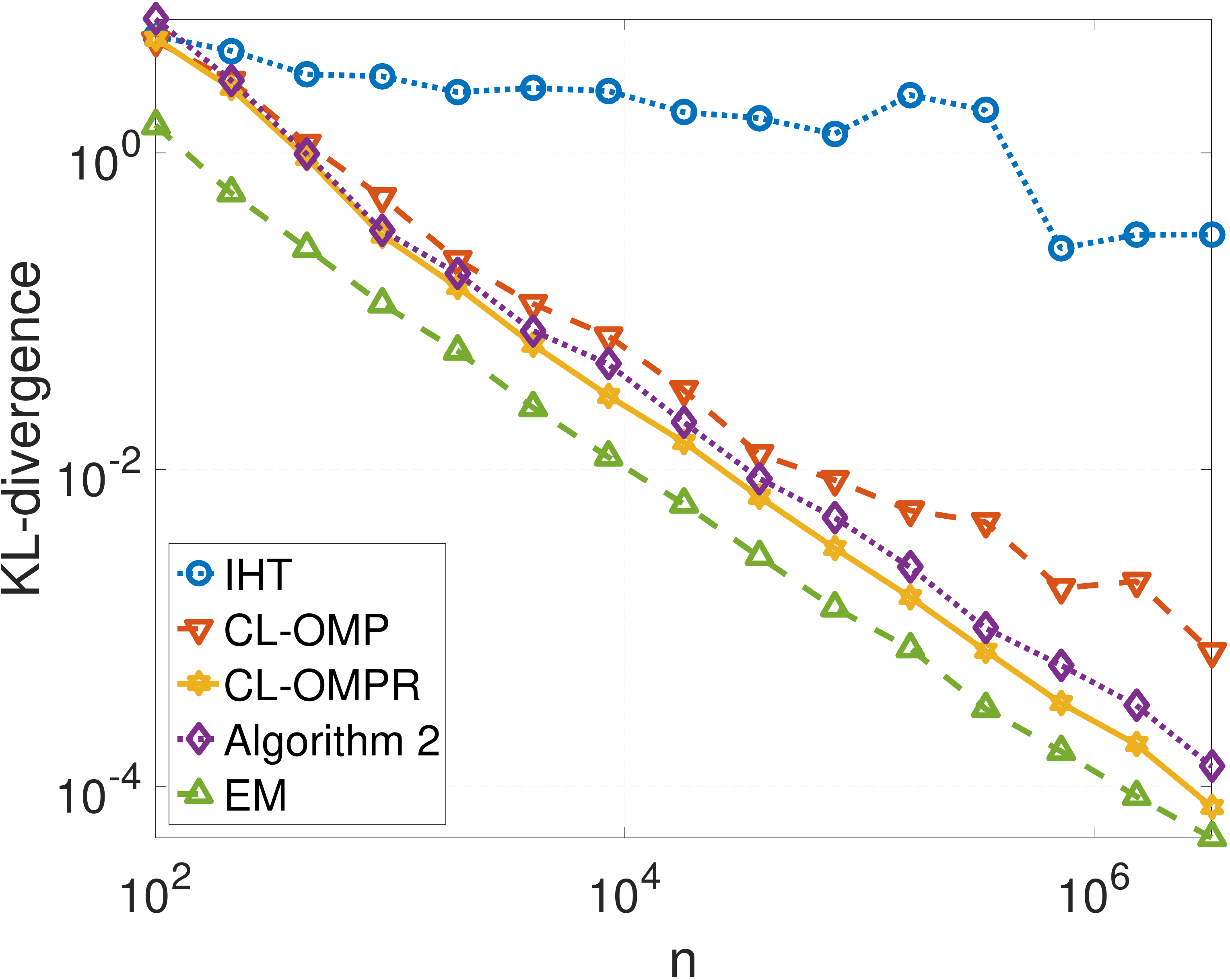}
\includegraphics[height=\hght]{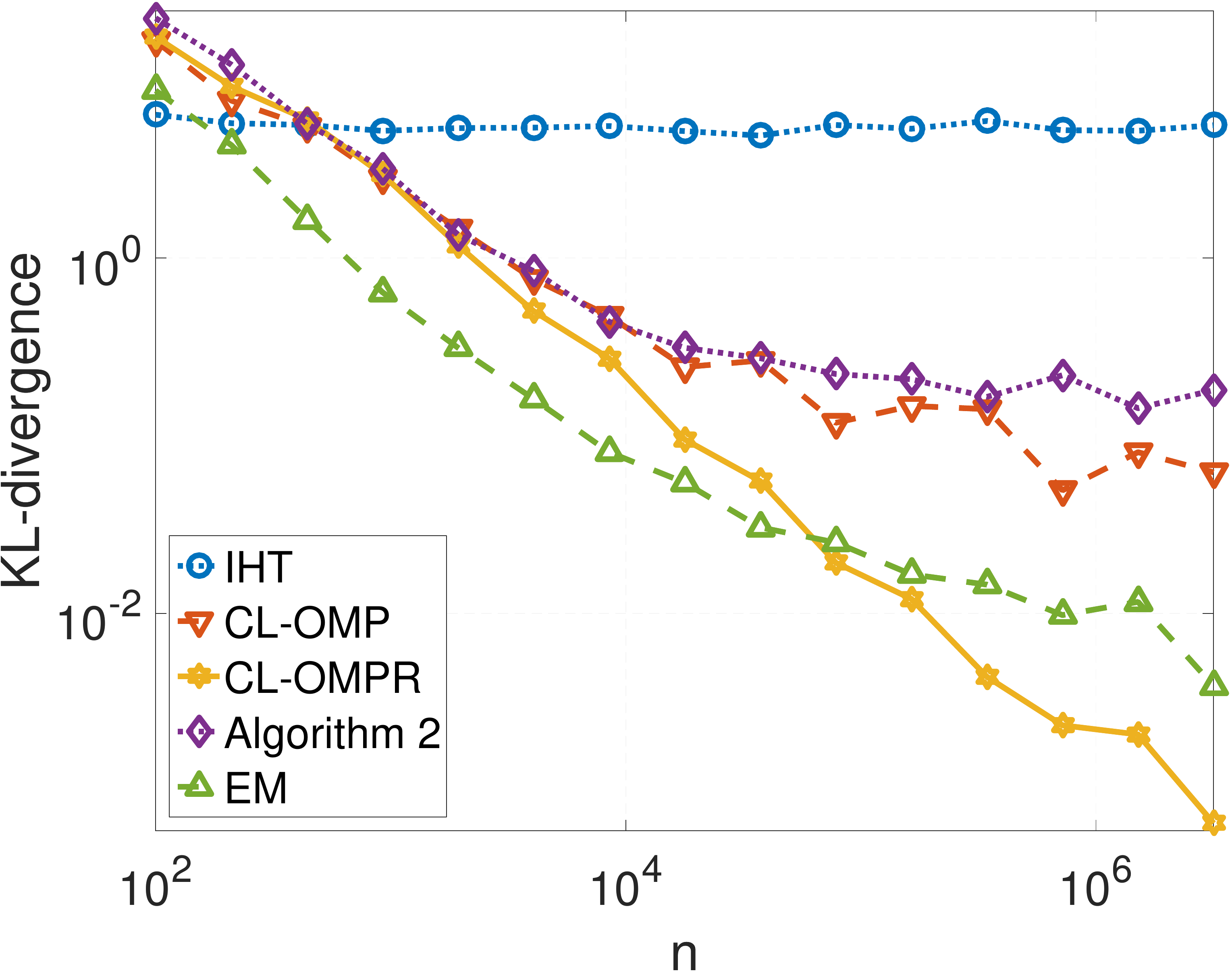}
\caption{Reconstruction results on synthetic data in dimension $d=10$, with $K=5$ component (left) or $K=20$ (right), and number of frequencies $m=5(2d+1)K$, with respect to the number of items in the database $n$. 
}
\label{fig:result}
\end{figure}

\begin{figure}[h!]
\centering

\includegraphics[height=\hght]{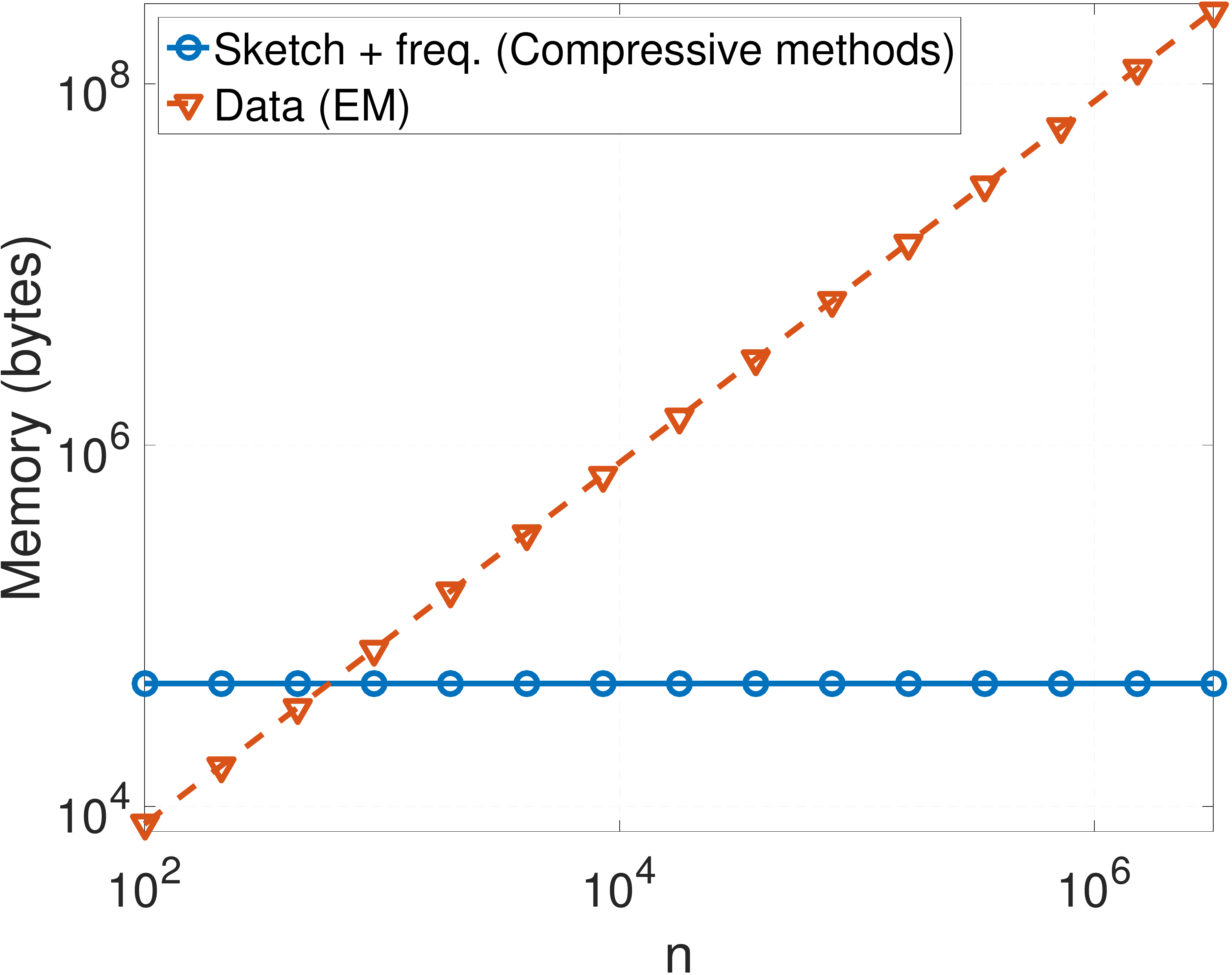}
\includegraphics[height=\hght]{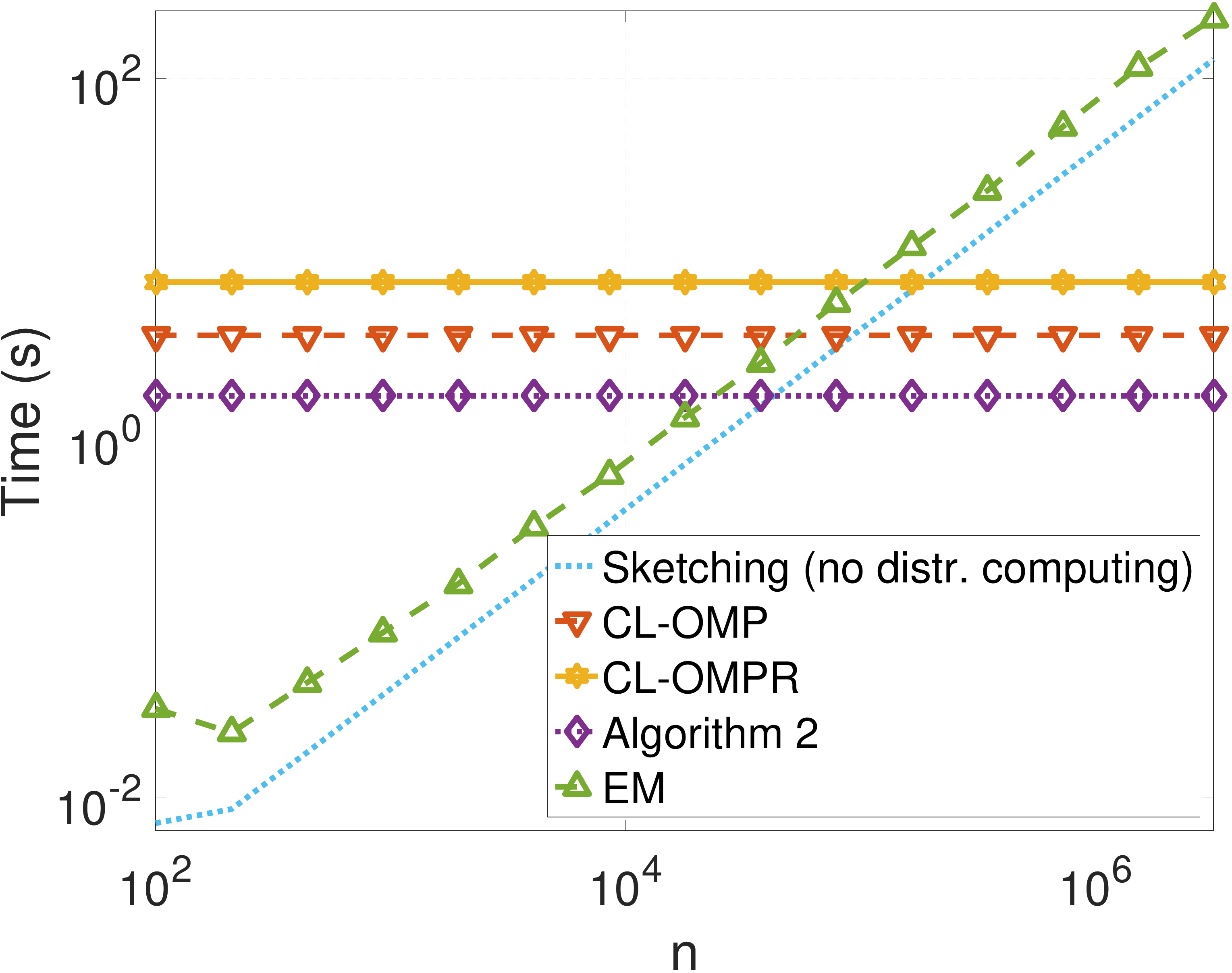}
\caption{Memory (left) and time (right) usage of all algorithms on synthetic data with dimension $n=10$, number of components 
$K=5$, 
and number of frequencies $m=5(2d+1)K$, with respect to the number of items in the database $n$. \review{On the left, ``Sketching'' refers to the time of computing the sketch with a naive, direct computation, which must be added to the computation time of the recovery algorithm (that does not vary with $n$) to obtain the total computation time of the proposed method. However the reader must keep in mind that the sketching step can be massively parallelized, is adapted to streams of data, and so on.}}
\label{fig:timeresult}
\end{figure}

\paragraph{Computation time and memory} In Figure~\ref{fig:timeresult}, computation time and memory usage of the compressive methods and EM are presented with respect to the database size $n$, using an \emph{Intel Core i7-4600U 2.1 GHz} CPU with $8~GB$ of RAM. In terms of time complexity (resp. memory usage), the EM algorithm scales in $\order(dnKT_{EM})$ for a fixed number of iterations $T_{EM}=100$ (resp. $\order(nd)$).
The CL-OMP or CL-OMPR algorithms scale in $\order(mdK^2)$ (resp. $\order(md)$), while Algorithm~\ref{algo:BS-CGMM} scales in $\order(mdK\log K)$ (resp. the same $\order(md)$).

At large $n$ the EM algorithm indeed becomes substantially slower than all compressive methods (Fig \ref{fig:timeresult}, left). We also keep in mind that we compare a MATLAB implementation of the compressive methods with a state-of-the-art C++ implementation of EM
\cite{Vedaldi2010}. Similarly, at large $n$ the compressive algorithms outperform EM by several orders of magnitude in terms of memory usage (Fig \ref{fig:timeresult}, right). 

We discussed the computational cost of the sketching operation in Section \ref{sec:complexity}, and the possible use of parallelization and distributed computing. In our settings, even when it is done linearly this operation is still faster than EM (Figure~\ref{fig:timeresult}, left).



\section{Large-scale proof of concept: speaker verification}\label{sec:speaker_verification}


Gaussians Mixture Models are popular for their capacity to smoothly approximate \emph{any} distribution \cite{Reynolds1995} by a large number of Gaussians. This is often the case with real data, and the problem of \emph{fitting} a large GMM to data drawn from some distribution is somewhat different from that of \emph{clustering} data and identifying reasonably well separated components, as presented in the previous section. 
In order to try out compressive methods on this challenging task, we test them on a speaker verification problem, with a classical approach requiring GMM referred to as Universal Background Model (GMM-UBM) \cite{Reynolds2000}.

\subsection{Overview of Speaker Verification}\label{sec:sv_method}
Given a fragment of speech and a candidate speaker, the goal is to assess if the fragment was indeed spoken by that person. 

We quickly describe GMM-UBM in this section. For more details we refer the reader to the original paper \cite{Reynolds2000}.
Similar to many speech processing tasks, this approach uses Mel Frequency Cepstrum Coefficients (MFCC) and their derivatives ($\Delta$-MFCC) as features $\bfx_{i}$. Those features have been often modeled with GMMs or more advanced Markov models. However, in our framework we do \emph{not} use $\Delta$-MFCC; indeed those coefficients typically have a negligible range in dynamic compared to the MFCC, which results in a difficult and unstable choice of frequencies. \review{As mentioned before, this problem may be potentially solved by a pre-whitening of the data, which we leave for future work. In this configuration, the speaker verification results will indeed not be state-of-the-art, but our goal is mainly to test our compressive approach on a different type of problem than that of clustering synthetic data, for which we have already observed excellent results.}

In the GMM-UBM model, each speaker $S$ is represented by one GMM $(\hyppar_S,\malpha_S)$. The key point is the introduction of a model $(\hyppar_{UBM},\malpha_{UBM})$ that represents a ``generic'' speaker, referred to as Universal Background Model (UBM). Given speech data $\dataset$ and a candidate speaker $S$, the statistic used for hypothesis testing is a likelihood ratio between the speaker and the generic model:
\begin{equation}
T(\dataset)=\frac{\dens_{\hyppar_{S},\malpha_{S}}(\dataset)}{\dens_{\hyppar_{UBM},\malpha_{UBM}}(\dataset)}.
\end{equation}
If $T(\dataset)$ exceeds a threshold $\tau$, the data $\dataset$ are considered as being uttered by the speaker $S$.

The GMMs corresponding to each speaker must somehow be ``comparable'' to each other and to the UBM. Therefore, the UBM is learned \emph{prior} to individual speaker models, using a large database of speech data uterred by many speakers. Then, given training data $\dataset_S$ specific to one speaker, one M-step from the EM algorithm \emph{initialized with the UBM} is used to adapt the UBM and derive the model $(\hyppar_S,\malpha_S)$. We refer the reader to \cite{Reynolds2000} for more details on this procedure.

{\bf In our framework, the EM or compressive estimation algorithms are used to learn the UBM.}

We note that this type of signal processing task may fully benefit from the advantages of the sketch structure described in Section \ref{sec:complexity}. For instance, in practice one can imagine collecting bit by bit the data to train the UBM in a real-life environment, in which case the sketch and the UBM may be progressively updated without having to keep the spoken fragments, possibly of sensitive nature.

\subsection{Setup}\label{sec:sv_setup}

The experiments were performed on the classical NIST05 speaker verification database. Both training and testing fragments are 5-minutes conversations between two speakers. The database contains approximately 650 speakers, and 30\,000 trials.

The MFCCs are computed using the Voicebox toolbox \cite{Brooks2005}. After filtering the audio data by a speech activity detector, the MFCCs are computed on 23ms frames with a $50\%$ overlap. The first coefficient is removed and we obtain 12-dimensional features ($d=12$).

Results are presented by choosing the threshold $\tau$ that yields the same rates of false alarm and missed detection, referred to as Equal Error Rate (EER). Each result is obtained as the mean of five experiments.

In all experiments, except when indicated otherwise, the compressive methods are performed using a sketch obtained by compressing the entire database of $n=2.10^8$ MFCC vectors after voice activity detection. The compression is performed taking advantage of distributed computing, by dividing the database into $200$ parts that are then compressed simultaneously on a computer cluster. Hence, even for a high number of frequencies $m=10^5$ the compression of the $n=2.10^8$ items takes less than an hour.


\subsection{Results}\label{sec:sv_results}

\begin{table}[h!]
\centering
\begin{tabular}{c|c|c||c|c|}
\cline{2-5}
&\multicolumn{2}{|c||}{EER (\%)}& \multicolumn{2}{|c|}{Time (s)} \\ \cline{2-5}
& CL-OMPR & Algorithm~\ref{algo:BS-CGMM} & CL-OMPR & Algorithm~\ref{algo:BS-CGMM} \\ \hline
 \multicolumn{1}{|c|}{$m=10^3$} & $40.3$ & $32.5$ & $7.10^2$ & $5.10$ \\ \hline
 \multicolumn{1}{|c|}{$m=10^4$} & $29.4$ & $29.0$ & $7.10^3$ & $5.10^2$ \\ \hline
 \multicolumn{1}{|c|}{$m=10^5$} & $28.8$ & $28.6$ & $7.10^4$ & $5.10^3$ \\ \hline
\end{tabular}
\caption{Comparison between CL-OMPR and Algorithm~\ref{algo:BS-CGMM} for speaker verification, with $K=64$. \label{tab:ubm_algos}}
\end{table}

\paragraph{Splitting algorithm} In the previous section, Algorithm~\ref{algo:BS-CGMM} was observed to be less accurate than CL-OMPR (Figure~\ref{fig:result}). However, as mentioned before the estimation problem considered here is somehow not to identify well-separated components, but rather to fit a GMM with a large number of components to a smooth probability density. In the first case, on synthetic data, Algorithm~\ref{algo:BS-CGMM} is indeed expected to sometimes yield poor results: unlike a Matching Pursuit-based approach such as CL-OMPR, at each iteration it locally divides the current Gaussians rather than ``exploring'' elsewhere. In the second case however, Algorithm~\ref{algo:BS-CGMM} may yield a correct approximation of the smooth density, by successively approaching it with GMMs at increasingly finer scales.

In Table \ref{tab:ubm_algos}, we compare the results obtained with CL-OMPR and Algorithm~\ref{algo:BS-CGMM} on the speaker verification task using $K=64$ Gaussians in the UBM. Results are indeed similar when the number of frequencies $m$ is large, and even surprisingly better with Algorithm~\ref{algo:BS-CGMM} for a low number of frequencies $m=1000$. Naturally, Algorithm~\ref{algo:BS-CGMM} is much faster than CL-OMPR, with more than a $10$ times speedup.

\paragraph{Sketching a large database} In Table \ref{tab:ubm_size}, we compare EER results when using either $n_1=3.10^5$ items uniformly selected in the database to cover all speakers, or all $n_2=2.10^8$ items in the database. The compressive Algorithm~\ref{algo:BS-CGMM} is performed at both scales, while EM is only performed with $n_1$ items, since the whole database is too large to be handled by the VLFeat toolbox on a machine with $8~GB$ of RAM. For the compressive approach, the use of the entire database indeed improves the results when compared to using only $n_1$ items to compute the sketch. At low $K=8$ or $K=64$ and high number of frequencies $m$, the compressive approach using $n_2$ items outperforms EM using only $n_1$ items.

\begin{table}
\centering
\begin{tabular}{cc|c|c||c|c||c|c|}
\cline{3-8}
&&\multicolumn{2}{c||}{$K=8$}& \multicolumn{2}{c||}{$K=64$}& \multicolumn{2}{c|}{$K=512$} \\ \cline{3-8}
&& $n_1$ & $n_2$ & $n_1$ & $n_2$ & $n_1$ & $n_2$ \\ \hline
\multicolumn{2}{|c|}{EM} & $31.4$ & n/a & $29.5$ & n/a & $27.5$ & n/a \\ \hline\hline
\multicolumn{1}{|c|}{\multirow{3}{*}{Alg.~\ref{algo:BS-CGMM}}} & $m=10^3$ & $32.5$ & $\mathbf{31.2}$ & $31.1$ & $32.5$ & $31.2$ & $29.4$ \\ \cline{2-8}
\multicolumn{1}{|c|}{} & $m=10^4$ & $32.1$ & $\mathbf{30.7}$ & $30.2$ & $\mathbf{29.0}$ & $30.3$ & $29.1$ \\ \cline{2-8}
\multicolumn{1}{|c|}{} & $m=10^5$ & $32.5$ & $\mathbf{30.7}$ & $29.8$ & $\mathbf{28.6}$ & $29.4$ & $29.2$ \\ \hline
\end{tabular}
\caption{Comparison EM and Algorithm~\ref{algo:BS-CGMM} for speaker verification, in terms of EER. For Algorithm \ref{algo:BS-CGMM}, results that outperform these of EM are outlined.\label{tab:ubm_size}}
\end{table}

\paragraph{Limitations due to coherence.} While increasing the number of components $K$ seems to consistently improves the results of EM, it is not the case with the compressive method for a fixed sketch size $m$.
A possible intuitive explanation could be that, by increasing the number of components we also increase the coherence between them -- \ie~the Gaussians in the GMM are increasingly overlapping each other -- which makes it more and more difficult to handle for any sparsity-based approach. In practice, it results in many components in the GMM having weights $\alpha\approx 0$. In other words, the algorithm outputs a $K'$-GMM with $K' < K$: there seems to be a ``limit'' number of components above which additional Gaussians are useless. It may be possible to deal with a higher level of sparsity by drastically increasing the number of frequencies $m$, at the cost of higher compression and estimation times.

\begin{figure}
\centering
\includegraphics[height=\hght]{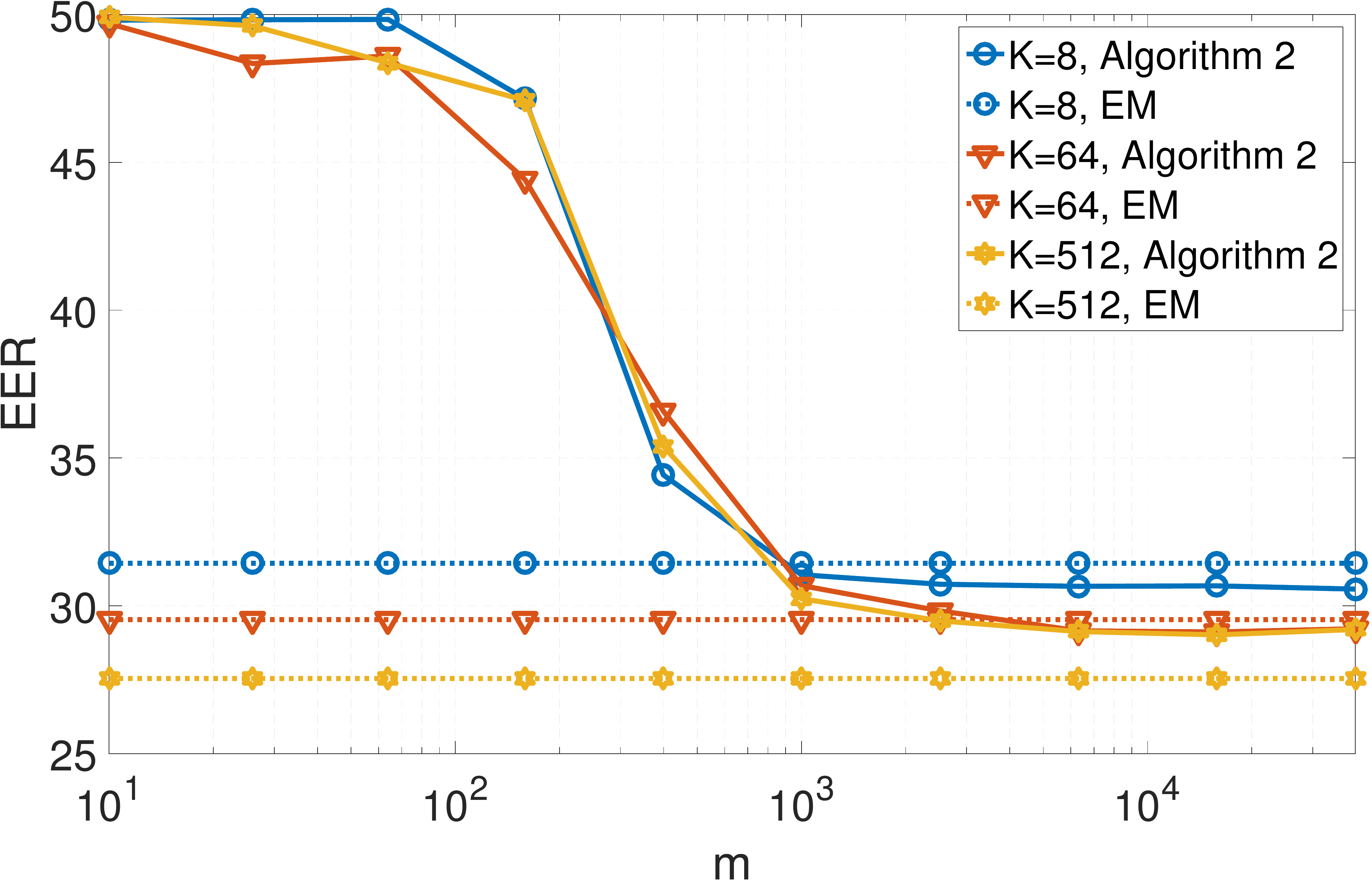}
\caption{Effect of the number of frequencies $m$ on speaker verification results. Algorithm \ref{algo:BS-CGMM} is performed on the whole database with $n_2=2.10^8$ items, while EM can only be performed on $n_1=3.10^5$ items.} \label{fig:ubm_phase}
\end{figure}

\paragraph{Number of components $K$ and compression.} In Figure~\ref{fig:ubm_phase} we study the effect of $m$ for various number of components $K=8$, $64$ and $512$. In each case we observe a sharp phase transition going from an $EER$ of $50 \%$, which corresponds to random guessing, to the results observed in Table \ref{tab:ubm_size}. Somehow surprisingly, this phase transition does not seem to depend on $K$, unlike the one observed on synthetic data (Figure~\ref{fig:phasetrans}). As mentioned before it could be interesting to drastically increase $m$ to see if the gap between results obtained EM and those obtained with Algorithm~\ref{algo:BS-CGMM} can be bridged in the $K=512$ case, however the phase transition pattern does not support this idea but rather a limitation of the method itself, maybe in the algorithmic approach. 

Overall, results on synthetic and real data show that the \emph{fitting} problem is, as expected, more challenging than the \emph{clustering} problem for the proposed sparsity-based approach. Indeed, while the clustering problem (synthetic data) is that of identifying well-separated components of a sparse distribution, the fitting problem is similar to a \emph{sparse approximation} task, which is known to be challenging when the ``signal'' (\ie~the true distribution of the data) is \emph{not} sparse. Nevertheless, let us point out that in Figure~\ref{fig:ubm_phase}, results approaching those of EM are obtained for $m=3000$ frequencies only, which corresponds to a whopping $33000$-fold compression of the database. 


\section{Information preservation guarantees ?}\label{sec:theory}
In this section we derive a number of information preservation guarantees of the proposed sketching operator. Let us come back to the ``generic'' compressive learning inverse problem introduced in Section \ref{sec:intro}:
\begin{equation}\label{eq:problem}
\argmin{\pp\in \Sigma}~\neucl{\hat{\bz}-\skop \pp},
\end{equation}
where $\model$ is some ``low-dimensional'' model. Compressive estimation algorithms such as CL-OMP(R) (Algorithm~\ref{algo:CL-OMP}) or the more scalable Algorithm~\ref{algo:BS-CGMM} (specifically designed for GMM) seek an approximate solution to this problem in the case $\model=\gaussmix{K}$, with $\gaussmix{K}$ the set of $K$-sparse distributions in $\gaussset$.

Precise recovery guarantees for CL-OMP(R) or Algorithm~\ref{algo:BS-CGMM} are beyond the scope of this paper due to the random nature of several steps in these algorithms and the many non-convex optimization schemes that they contain. Instead, we rather demonstrated empirically in Sections~\ref{sec:results} and~\ref{sec:speaker_verification} that these algorithms perform well on a large range of GMM estimation problems, with synthetic and real data.

In parallel, a fundamental question consists in asking if the problem is well-posed, \ie~if a potential solution of \eqref{eq:problem} is somehow guaranteed to be a ``good'' estimate of the distribution of the data. Namely, we ask the following questions:

\begin{itemize}
\item For a distribution $\pp_\model \in \model$, does the sketch $\bfz=\skop \pp_{\model}$ contain ``enough'' information to retrieve the distribution $\pp_\model$ ?
\item Is this retrieval robust to using the empirical sketch $\hat\bfz$ instead of the true sketch $\bfz$?
\item Is this retrieval robust if $\bfz = \skop \pp$ where the encoded distribution $\pp$ is not \emph{exactly} in the model $\model$ but only {\em well-approximated} by a distribution in the model (\ie~the distance $d(\pp,\model)=\inf_{\pp_\model \in \model}d(\pp,\pp_\model)$ is small for some metric $d$) ?
\end{itemize}

The answers to these questions are linked to the existence of a so-called \emph{instance-optimal} decoder \cite{Cohen2009,Bourrier2014}, \ie~a (not necessarily tractable) reconstruction paradigm that is robust to noise and modeling error. In this paper, we show that with high probability
the decoder induced by solving \eqref{eq:problem} is instance optimal with respect to the Maximum Mean Discrepancy (MMD) metric \cite{Sutherland2015} (Eq. \eqref{eq:chrcdiff}), provided the model and frequency distribution satisfy some assumptions. We then prove this result for GMMs with bounded parameters, starting with the toy example of single Gaussians ($K=1$).

\review{The proof of our main Theorem (Theorem \ref{thm:appli}), given in Appendix \ref{sec:proof}, introduces variants of usual tools in compressive sensing.
The idea is to use the fact that a Lower Restricted Isometry Property (LRIP) induces the existence of a robust instance optimal decoder, as shown by Bourrier \etal~\cite{Bourrier2014}. To prove the LRIP for the measurement operator $\skop$, we use the fact that the empirical mean $\|\skop \pp - \skop \qq\|_2^2$ concentrates around its expectation $\normKsq{\pp}{\qq}$, and use $\epsilon$-coverings to extend this concentration result uniformly over the whole model -- this method is similar to the ``simple proof'' of the RIP developed by Baraniuk \etal~\cite{Baraniuk2008}.

As we will see, the results are still in a preliminary state and suboptimal in some cases. However, the use of compressive sensing tools in the context of kernel mean embedding and Random Features is an original and promising lead for future work, as is the introduction of guarantees such as robustness to modeling error for a Generalized Method of Moments problem. They can be seen as our main theoretical contribution.}


\subsection{Existence of instance optimal decoders for sketched distributions}\label{sec:thm}

In this section we formulate a general result that guarantees robust decoding of any model $\model$ (not necessarily restricted to mixture models) under some hypotheses.
The reader should refer to Appendix \ref{sec:def} for definitions related to $\epsilon$-coverings.

\begin{assumption}{$\mathbf{A_1}(\model)$}[Compactness of the model]\label{assum:compact}
The model $\model$ is compact with respect to the total variation norm $\|.\|_{TV}$. In particular, it implies that it has finite covering numbers $N_{\model,\|.\|_{TV}}(\epsilon) < \infty$.
\end{assumption}

In the context of mixture models, the following Lemma shows that compactness of any basic set of distributions $\gaussset$ extends to its set of $K$-sparse distributions $\gaussmix{K}$, and that their covering numbers are related. Its proof is given in Appendix \ref{sec:proofs_mixture}.

\begin{lemma}\label{lem:covnummix}
Suppose the set of basic distributions $\gaussset$ is compact with respect to some norm $\|.\|$, denote $C=\max_{\pp\in\gaussset}\|\pp\|$. Then for all $K$ the set of $K$-sparse distributions $\gaussmix{K}$ is also compact and satisfies, for all $\epsilon>0$ and $0<\tau<1$,
\begin{equation}
\label{eq:covnummix}
N_{\gaussmix{K},\|.\|}(\epsilon) \leq \left(\frac{8C \cdot N_{\gaussset,\|.\|}(\tau\epsilon)}{(1-\tau)\epsilon}\right)^K\,.
\end{equation}
\end{lemma}

Note that, in the case where $\|.\|=\|.\|_{TV}$ as in Assumption \ref{assum:compact}, we have $C=1$.

The second assumption 
involves the model $\model$, the frequency distribution $\freqdist$, a small nonnegative constant $\eta\geq 0$ and a constant $\cstdom>0$.

\begin{assumption}{$\mathbf{A_2}(\eta,\model,\freqdist,\cstdom)$}[Domination of the total variation norm]\label{assum:domin}
For all $\pp_1,\pp_2 \in \model$, we have
\begin{equation}
\label{eq:domination}
\Big[\normK{\pp_1}{\pp_2}\geq \eta\Big] \Rightarrow \Big[\|\pp_1-\pp_2\|_{TV}\leq \cstdom \normK{\pp_1}{\pp_2}\Big],
\end{equation}
where $\normKs$ is the MMD \eqref{eq:chrcdiff}.
\end{assumption}
Note that, since $\|\pp_1-\pp_2\|_{TV} \leq 2$ for all measures, if $\eta>0$ Assumption \ref{assum:domin} is always verified with $\cstdom = 2/\eta$.

\paragraph{``Ideal'' decoder.} Given any sketch $\bfz \in \mathbb{C}^m$ and measurement operator $\skop$, for all $\pp_1,\pp_2 \in \pmeas$ we have
\begin{equation}\label{eq:continuityfordecoder}
\Big| \|\bfz-\skop \pp_1\|_2-\|\bfz-\skop \pp_2\|_2 \Big| \leq \|\skop (\pp_1-\pp_2)\|_2 \leq \|\pp_1-\pp_2\|_{TV},
\end{equation}
by Lemma \ref{lem:bounds} in Appendix \ref{sec:def_meas}. Hence the function $\pp \in \model \mapsto \|\bfz-\skop \pp\|_2$ is continuous with respect to the total variation norm. If Assumption \ref{assum:compact} is satisfied and the model $\model$ is compact, this function reaches its minimum on it, which implies that the problem \eqref{eq:problem} has at least one solution.

In light of this observation, under Assumption \ref{assum:compact}, we analyze below the information-theoretic estimation guarantees of the idealized decoder $\decod$ that minimizes \eqref{eq:problem}, \ie~return a distribution verifying:
\begin{equation}\label{eq:decod}
\decod(\bfz,\skop) \in \argmin{\pp \in \model}\|\bfz-\skop \pp\|_2.
\end{equation}

We now turn to the main result of this section.

\begin{theorem}\label{thm:appli}
Consider a model $\model$, a frequency distribution $\freqdist$, a small positive constant $1 \geq \eta > 0$ and a constant $\cstdom\geq 1$ such that Assumptions \ref{assum:compact} and \ref{assum:domin} hold.

Let $\bfx_i \in \mathbb{R}^d$, $i=1...n$ be $n$ points drawn $i.i.d.$ from an arbitrary distribution $\pp^*\in\pmeas$, and $\freq_{j} \in \mathbb{R}^d$, $j=1...m$ be $m$ frequencies drawn $i.i.d.$ from $\freqdist$. Denote $\bar\pp=\decod(\hat\bfz,\skop)$ the distribution reconstructed from the empirical sketch $\hat\bfz$.

Let $\rho>0$. Suppose that
\begin{equation}\label{eq:thmmbound}
m\geq 12\cstdom^2\log\left(\tfrac{2}{\rho} \cdot N_{\model,\|.\|_{TV}}\left(\tfrac{\eta^2}{24}\right)\right).
\end{equation}

Then, with probability at least $1-\rho$ on the drawing of the items $\bfx_i$ \underline{and} sampling frequencies $\freq_{j}$, we have
\begin{equation}
\normK{\pp^*}{\bar\pp}\leq 5\ d_{TV}(\pp^*,\model)+\tfrac{4(1+\sqrt{2\log(2/\rho)})}{\sqrt{n}}+\eta,
\end{equation}
where $\normKs$ is the MMD \eqref{eq:chrcdiff} and $d_{TV}(\pp^*,\model)=\inf_{\pp \in \model} \|\pp^*-\pp\|_{TV}$ is the distance from $\pp^*$ to the model.
\end{theorem}


Hence the MMD between the distribution in the model recovered from the empirical sketch and the original distribution $\pp^*$ is controlled by the distance (measured by the total variation norm) between $\pp^*$ and the model, and a small additive error. This proves that the decoding is robust both to the use of the empirical sketch instead of the true sketch, and to the fact that $\pp^*$ may not be exactly in the model. \review{The choice of the kernel for the MMD is crucial to obtain meaningful guarantees. It somehow justifies the general strategy of choosing a kernel that maximizes the discriminative power of the MMD \cite{Sriperumbudur2009}, which as mentioned earlier is the idea behind the proposed Adapted Radius heuristic. Further work will examine relationship between the MMD and other classic metrics such as likelihood or KL-divergence \cite{Reddi2015}.}

\paragraph{Limitations and future work.} Theorem \ref{thm:appli} does not hold for an additive error $\eta=0$, which would be more akin to usual compressive sensing. In Appendix \ref{sec:thmbis}, we give a version of the Theorem under a different set of Hypotheses $\mathbf{H_i}$ that include the $\eta=0$ case. However, exhibiting a model $\model$ and frequency distribution $\freqdist$ that satisfy those hypotheses in the $\eta=0$ case is yet to be done.

The control of the MMD $\normK{\pp^*}{\bar\pp}$ with the distance $d_{TV}(\pp^*,\model)$ is not optimal -- ideally we would like to have the same metric on both sides of the inequality. 
In Appendix \ref{sec:proof}, we formulate all results for a general metric $d$ under some assumptions, then specialize in the case $d=\|.\|_{TV}$ allowing for possible future generalizations.

We now apply Theorem \ref{thm:appli} to GMMs, starting with the case $K=1$ as a toy example.

\subsection{Applications to single Gaussians (toy example)}
We consider the case $K=1$ where the model $\model$ is formed by single Gaussians as a toy example\footnote{Obviously, fitting a single Gaussian to a dataset can easily be done by direct empirical estimators.}, \ie~$\model=\gaussset$. We show that when the set $\thetaset$ of parameters of the Gaussians is compact, Assumption \ref{assum:compact} is verified with an explicit bound on the covering numbers of the model. Additionally, when the frequency distribution is Gaussian, Assumption \ref{assum:domin} is verified with a bound $\cstdom$ that does not depend on the additive error $\eta$. \review{Using the Gaussian kernel has the advantage of yielding closed-form expressions for the mean kernel, which at this point does not seem to be feasible with the proposed Adapted Radius distribution.}

Recall the parametrization $\mtheta=[\mmu,\msigma]$ with $\mmu\in \mathbb{R}^d$ and $\msigma\in\left(\mathbb{R}_+\setminus\{0\}\right)^d$ of the set $\gaussset=\left\lbrace \pp_\mtheta\right\rbrace_{\mtheta \in \thetaset}$ of Gaussians with diagonal covariance. We suppose here that the set of parameters $\thetaset \subset \mathbb{R}^{2d}$ is compact, \ie~closed and bounded. In particular, the variances of the considered Gaussians are bounded, and we denote $\sigma^2_{\min}:=\min_{[\mmu,\msigma] \in \thetaset}\min_{\ell} \sigma^2_{\ell} >0$ (resp. $\sigma^2_{\max}:=\max_{[\mmu,\msigma] \in \thetaset}\max_{\ell} \sigma^2_{\ell}<\infty$) their minimum (resp. maximum) value. We also define $M:=\max_{[\mmu,\msigma]} \|\mmu\|_2$, and finally we denote $\rad{\thetaset}:=\min\{r>0;\exists \bfx \in \mathbb{R}^{2d}, \thetaset \subset B(\bfx,r)\}$ the \emph{Chebyshev radius} \cite{ChebyshevRadius2016} of $\thetaset$, \ie~the minimal radius $r$ such that $\thetaset$ is contained in a ball of radius $r$ for the Euclidean norm.

Note that our framework is significantly different from many other works \cite{Smola2007,Gretton2006,Kanagawa2014} which provide guarantees when the \emph{support} of the distributions is compact, while here the Gaussian densities have \emph{infinite support} but their \emph{parameters} belong to a compact set, which is a far more realistic setting.


\begin{theorem}\label{thm:covnum}
Suppose the set of parameters $\thetaset \subset \mathbb{R}^{2d}$ is compact. Then the set of Gaussians $\gaussset=\left\lbrace \pp_\mtheta\right\rbrace_{\mtheta \in \thetaset}$ is compact.
Furthermore, for all $\epsilon>0$, we have
\begin{equation}
\label{eq:covnum_g}
N_{\gaussset,\|.\|_{TV}}(\epsilon)\leq \left(\frac{B}{\epsilon}\right)^{2d},
\end{equation}
where $B:=8\max\left(\sigma_{\min}^{-1},\sigma_{\min}^{-2}/\sqrt{2}\right)\rad{\thetaset}$.
\end{theorem}
Thus Assumption \ref{assum:compact} is verified for the model $\model = \gaussset$.
Assumption \ref{assum:domin} is also verified, and for a Gaussian frequency distribution we have the following:
\begin{theorem}\label{thm:toy}
Suppose the set of parameters $\thetaset \subset \mathbb{R}^{2d}$ is compact, and the frequency distribution is an isotropic Gaussian $\freqdist=\mathcal{N}\left(0,\frac{\sigkersmall}{d}\bfI\right)$ for some $\sigkersmall >0$.

Then, for all $\pp,\qq \in \gaussset$, we have
\begin{equation}
\|\pp-\qq\|_{TV}\leq D\normK{\pp}{\qq},
\end{equation}
where
\begin{equation}
D=\max(\sigma_{\min}^{-1},\sigma_{\min}^{-2}/\sqrt{2})\sqrt{\frac{2dD_1\cdot e^{3\sigkersmall\sigma_{\max}^2}}{\sigkersmall(1-e^{-D_1})}} \text{~ with } D_1=\sigma_{\max}^2\sigkersmall\left(1+\frac{2M^2}{d}\right).
\end{equation}
\end{theorem}
The proof is given in Appendix \ref{sec:proofs_covnum_gaussians}.

As a consequence, Assumption \ref{assum:domin} is verified for $\model = \gaussset$, $\freqdist=\mathcal{N}(0,\sigmaker^2\bfI)$, any $\eta\geq 0$ and $\cstdom=\min(D,2/\eta)$.
Theorems \ref{thm:covnum} and \ref{thm:toy} lead to the following immediate corollary of Theorem \ref{thm:appli}.
\begin{corollary}\label{cor:toy}
In the case $\model=\gaussset$ of single Gaussians with a compact set of parameters, for a Gaussian frequency distribution $\freqdist=\mathcal{N}\left(1,\frac{\sigkersmall}{d}\bfI\right)$ and any constant $0<\eta\leq 1$, Theorem \ref{thm:appli} is verified with \eqref{eq:thmmbound} replaced by
\begin{equation}
m\geq 12\cstdom^2\left(4d\log\left(\frac{C}{\eta}\right)+\log\frac{2}{\rho}\right),
\end{equation}
where $\cstdom=\min(D,2/\eta)$, $C=\sqrt{24B}$, $B$ is defined as in Theorem \ref{thm:covnum} and $D$ is defined as in Theorem \ref{thm:toy}.
\end{corollary}

Hence, for a fixed model $\model=\gaussset$, the additive error $\eta$ decreases exponentially with the number of measurements $m$. It is also interesting to note that conversely, for a small additive error $\eta\leq 2/D=\order(1/\sqrt{d})$, assuming that all parameters of the model appearing in the expression of $D$ are constant, we have $\cstdom=\order(\sqrt{d})$ and thus the number of measurements $m$ must grow as $\order(d^2)$ with the dimension. It is sub-optimal in the sense that the ideal estimators of the mean and diagonal covariance of a single Gaussian, \ie~the empirical mean and covariance, have size $\order(d)$. This number of measurements may scale with the size of the empirical estimators for Gaussian with \emph{full} covariance, although the results presented in this paper do not directly apply to this case.

\subsection{Application to GMMs}

Theorem \ref{thm:covnum} and Lemma \ref{lem:covnummix} allow for an immediate extension of Assumption \ref{assum:compact} from the set of basic distributions to the corresponding mixture models. In the case of GMMs, we have the following corollary, whose proof is given in Appendix \ref{sec:proofs_covnum_gaussians}.

\begin{corollary}\label{cor:covnumgmm}
Suppose the set of parameters $\thetaset \subset \mathbb{R}^{2d}$ is compact. Then, for all $K>1$ the set of GMMs $\gaussmix{K}$ is compact.
Furthermore, for all $\epsilon>0$, we have
\begin{equation}
\label{eq:covnumgmm}
N_{\gaussmix{K},\|.\|_{TV}}(\epsilon)\leq \left(\frac{2(B+1)}{\epsilon}\right)^{(2d+1)K},
\end{equation}
where $B$ is defined as in Theorem \ref{thm:covnum}.
\end{corollary}

Unfortunately, unlike the compactness property, Assumption \ref{assum:domin} cannot be immediately extended from the set of basic distributions $\gaussset$ to the mixture model $\gaussmix{K}$, and it is not clear whether doing so would require some additional hypotheses or not. Though we strongly believe that a result similar to the $K=1$ case holds (see the discussion at the end of this section), here we use the fact that Assumption \ref{assum:domin} is verified with $\cstdom=2/\eta$ regardless of the model and frequency distribution.

This leads to the following corollary.

\begin{corollary}\label{cor:gmms}
In the case $\model=\gaussmix{K}$ of GMMs with a compact set of parameters, for any frequency distribution and constant $0<\eta\leq 1$, Theorem \ref{thm:appli} is verified with \eqref{eq:thmmbound} replaced by
\begin{equation}
m\geq 48\eta^{-2}\left(2K(2d+1)\log\left(\frac{C}{\eta}\right)+\log\frac{2}{\rho}\right),
\end{equation}
where $C=\sqrt{48(B+1)}$, with $B$ defined as in Theorem \ref{thm:covnum}.
\end{corollary}


\paragraph{Conjecture.} Corollary \ref{cor:gmms} suggests that the reconstruction error $\eta$ for GMMs decreases as $\order\left(n^{-\frac{1}{2}}+m^{-\frac{1}{2}}\right)$ (up to some inverse exponential factor), which seems to nullify the advantages of the ``compressive'' approach. This is due to the use of the ``worst'' bound $\cstdom=2/\eta$ in Assumption \ref{assum:domin}. However, we strongly believe that Assumption \ref{assum:domin} may hold with a better bound that does not depend on $\eta$, similar to the $K=1$ case.

We support this claim by empirically evaluating reconstruction results of the CL-OMPR algorithm with respect to the MMD in Figure \ref{fig:phasetrans2}. We observe a phase transition pattern similar to the one already noted in Section \ref{sec:results_phase} for the KL-divergence, which is inconsistent with an additive error that scales in $\order\left(m^{-1/2}\right)$. On the contrary, the $\order\left(n^{-1/2}\right)$ decrease is indeed observed, which supports the theory but also the capacity of CL-OMPR to approximate the ideal decoder \eqref{eq:decod}.

The proof of Assumption \ref{assum:domin} for general GMMs seems complex and beyond the scope of this paper. A possible strategy would be to be able to directly extend Assumption \ref{assum:domin} on the basic set of distributions $\gaussset$ to the corresponding mixture model $\gaussmix{K}$, while another approach would be to use a different metric other than the total variation norm. The latter is particularly discussed in Appendix \ref{sec:proof}.

\begin{figure}
\centering
\includegraphics[height=\hght]{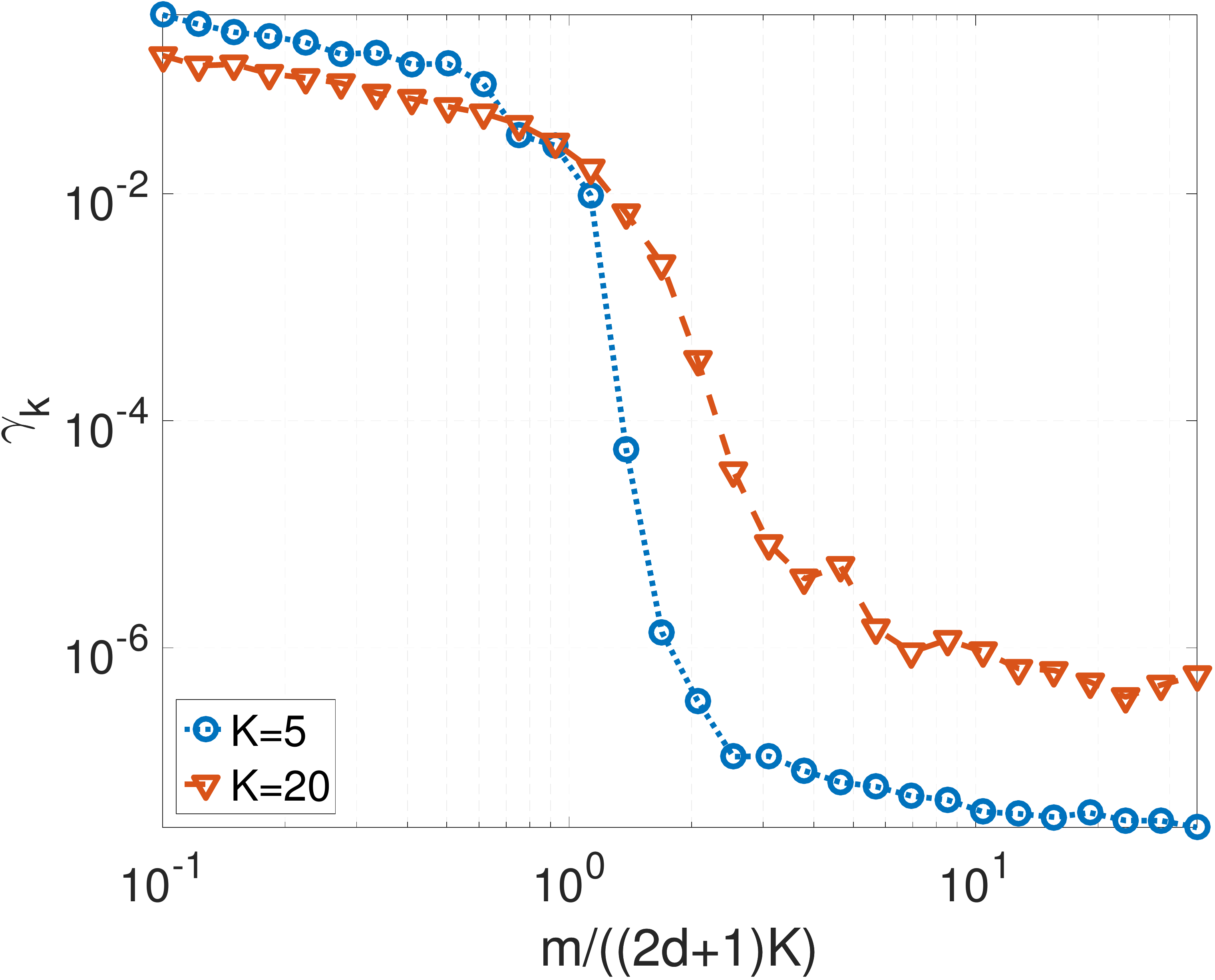}
\includegraphics[height=\hght]{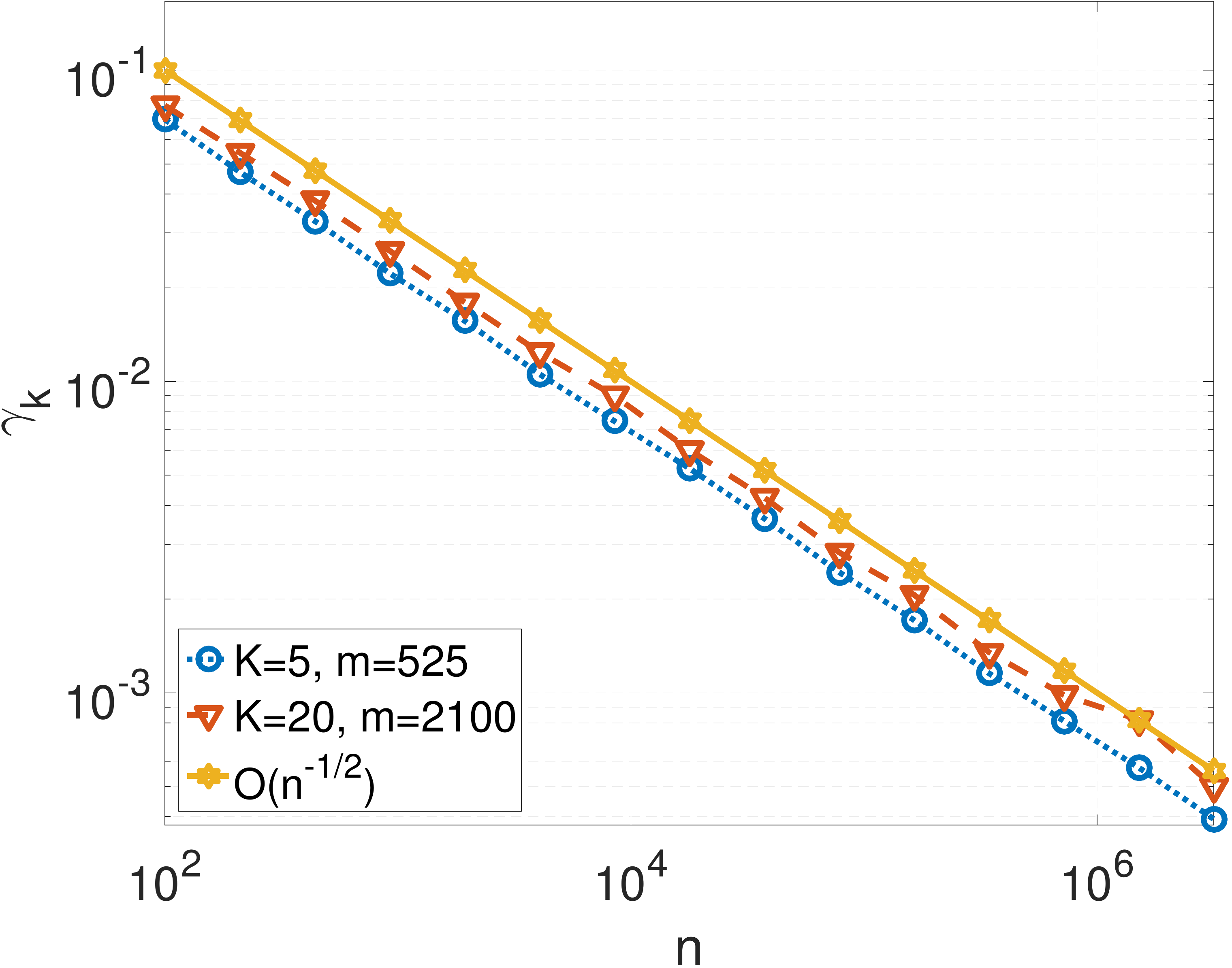} 
\caption{Reconstruction results of CL-OMPR for the $\normKs$ metric with respect to $m$ (left) and $n$ (right), in dimension $d=10$ and $K=5$ components, using the true theoretical sketch $\skop \pp_{\hyppar_0,\malpha_0}$ on the left and $m=5K(2d+1)$ frequencies on the right. In a similar fashion to the KL-divergence (Section \ref{sec:evaluation_measure}), the MMD $\normKs$ is approximated by drawing $5\cdot 10^5$ frequencies from $\freqdist$ and by empirically evaluating \eqref{eq:chrcdiff}, using the closed-form expression of the characteristic function of GMMs.}
\label{fig:phasetrans2}
\end{figure}

\section{Conclusion and outlooks}

We presented a method for probability mixture estimation on a large database exploiting a \emph{sketch} of the data instead of the data itself. The sketch is an appropriate structure that leads to considerable gain in terms of memory. It can be computed in a distributed or streaming manner, and it can fully exploit the advantages of GPU computing.

A typical greedy method for sparse reconstruction was defined, leading to reconstruction algorithms both efficient and stable, even when the dictionary of atoms is infinite and uncountable. In the case of GMM, an additional efficient algorithm based on hierarchical splitting of GMMs was described. 

A heuristic to select generalized moments based on a decomposition robust to high dimension and maximal variations of the characteristic function was designed. A procedure to estimate the parameter of this heuristic was described, resulting in a method that is faster than traditional kernel design by cross-validation, has the advantage of being unsupervised and thus is probably suited for other tasks, and yields better reconstruction results.

Excellent results were observed on synthetic data, where the greedy algorithms approach the reconstruction results of EM, using less memory and computation time when the number of database elements is large. The method was successfully applied to a large-scale speaker verification task. The hierarchical approach proved to be the most efficient method for this challenge, illustrating the diversity of the problem and of the proposed solutions. As in usual compressive sensing, limitations of the method when the number of sparse components in the distribution is large were observed.

Finally, information preservation guarantees were developed for the recovery of any compact set of distributions. The proof of Theorem \ref{thm:appli} (Appendix \ref{sec:proof}) introduced a weaker variant of the Restricted Isometry Property (RIP) for non-uniform recovery. We then applied this result to GMMs with bounded parameters, and observed a technical bottleneck between the toy $K=1$ case and general GMMs.

\paragraph{Outlooks.} As mentioned earlier, the method can readily be applied to other mixture models, such as mixtures of $\alpha$-stable distributions which do not have explicit likelihood but whose characteristic function is known \cite{Salas-Gonzalez2009}. Based on the principle of maximizing the variation of the characteristic function, suitable heuristics for the choice of the sampling pattern may be derived for other models.

\review{The method can also easily incorporate variants of Random Fourier Features that are faster to compute or more precise \cite{Le2013,Xinnan2016}. Existing methods to learn the kernel \cite{Gretton2012,Yang2015,Sinha2016,Paige2016} may be adapted to our framework, and in return the proposed unsupervised kernel learning procedure and Adapted radius heuristic may be useful for other tasks.}


As mentioned earlier, technical difficulties on the domination between certain metrics were observed between the toy $K=1$ case and general GMMs, pointing to a promising lead for future work. The proof of Theorem \ref{thm:appli} also uses innovative variants of several classical tools in compressive sensing, which may be useful in the study of other instances of generalized compressive sensing.

\section*{Acknowledgment}
This work was supported in part by the European Research Council, PLEASE project (ERC-StG- 2011-277906).

\bibliographystyle{plain}
\section*{References}
\small
\bibliography{library}

\normalsize

\appendix
\renewcommand{\theequation}{\Alph{section}.\arabic{equation}}
\renewcommand{\thelemma}{\Alph{section}.\arabic{lemma}}
\renewcommand{\thetheorem}{\Alph{section}.\arabic{theorem}}
\renewcommand{\theremark}{\Alph{section}.\arabic{remark}}
\renewcommand{\thecorollary}{\Alph{section}.\arabic{corollary}}
\renewcommand{\thedefinition}{\Alph{section}.\arabic{definition}}


\section{Definitions, preliminary results}\label{sec:def}

In this section, we group some definitions and useful results.

\subsection{Positive definite kernels}\label{sec:def_kernel}

We recall the definition of p.d. kernels.

\begin{definition}[Positive definite kernel]
Let $\Xspace$ be an arbitrary set. A symmetric function (or \textbf{kernel}) $\kernel:\Xspace \times \Xspace \rightarrow \mathbb{C}$ is called \textbf{positive definite (p.d.)} if, for all $n\in\mathbb{N}$, $c_1,...,c_n\in\mathbb{C}$ and all $\bfx_1,...,\bfx_n\in X$, we have
\begin{equation*}
\sum_{i,j=1}^n c_i\bar{c_j}\kernel(\bfx_i,\bfx_j)\geq 0.
\end{equation*}
\end{definition}
Note that strict positivity is not mandatory in the above equation. In terms of vocabulary, p.d. kernels bear connections with, \eg, positive \emph{semi}-definite matrices (however they are indeed called positive definite kernels in the literature).
\begin{definition}[Positive definite function]
A function $\TIk:\mathbb{R}^d \rightarrow \mathbb{C}$ is called \textbf{positive definite} if the kernel defined by $\kernel(\bfx,\bfy)=\TIk(\bfx-\bfy)$ is positive definite.
\end{definition}

\subsection{Measures}\label{sec:def_meas}

\begin{definition}[Nonnegative measure]
A measure $\meas\in E$ over a measurable space $(X,\mathcal{B})$ is said \textbf{nonnegative} if:
\begin{equation*}
\forall B \in \mathcal{B},~\meas(B)\geq 0.
\end{equation*}
\end{definition}
\begin{definition}[Support of a measure]
The \textbf{support} of a signed measure $\meas \in E$ over a measurable, topological space $X$ is defined to be the closed set,
\begin{equation*}
\supp(\meas):=X\backslash \bigcup\left\lbrace U\subset \Xspace:~U \text{ is open},~\meas(U)=0\right\rbrace.
\end{equation*}
\end{definition}

\begin{definition}[Total variation norm, Finite measure]
Let $\meas\in E$ be a signed measure over a measurable space $(X,\mathcal{B})$. Define the Jordan decomposition $(\meas^-, \meas^+)$ of $\meas$  where $\meas^+$ and $\meas^-$ are positive measures (see \cite{Fischer2012} and \cite{Rudin1987} Chap. 6 for more details). Denote $|\meas|=\meas^+ +\meas^-$. The \textbf{total variation norm} of $\meas$ is defined as:
\[
\|\meas\|_{TV}=|\meas|(\Xspace)=\int_\Xspace d|\meas|(\bfx).
\]
The measure $\meas$ is said \textbf{finite} if $\|\meas\|_{TV}< \infty$.

Note that if $\meas$ is totally continuous with respect to the Lebesgue measure, \ie~if there exists an integrable function $f$ such that $d\meas(\bfx)=f(\bfx)d\bfx$, then the total variation norm is the classic $L^1$-norm of this function: $\|\meas\|_{TV}=\|f\|_{L^1}$.
\end{definition}

We have the following bounds.
\begin{lemma}\label{lem:bounds}
For any sketching operator $\skop$ obtained by sampling the characteristic function, and any finite signed measure $\meas \in E$, we have
\begin{equation}
\label{eq:lemskopbound}
\|\skop \meas\|_2 \leq \|\meas\|_{TV}.
\end{equation}

For any frequency sampling distribution $\freqdist$ and any pair of probability distributions $\pp,\qq \in \pmeas$, we have
\begin{equation}
\label{eq:lemnormkbound}
\normK{\pp}{\qq} \leq \|\pp-\qq\|_{TV}.
\end{equation}
\end{lemma}

\begin{proof}
For any $\meas \in E$, we have
\begin{equation*}
\|\skop\meas\|_2^2 = \frac{1}{m} \sum_{j=1}^m \left\lvert\int_{\mathbb{R}^d} e^{-\imaginaryi\freq_j^T\bfx}d\meas(\bfx)\right\rvert^2 \leq \frac{1}{m} \sum_{j=1}^m \left(\int_{\mathbb{R}^d} d|\meas|(\bfx)\right)^2 =\|\meas\|_{TV}^2.
\end{equation*}

For all $\pp,\qq \in \pmeas$ a simple reformulation of $\normK{\pp}{\qq}$, see \cite{Sriperumbudur2010}, is:
\begin{equation*}
\normK{\pp}{\qq}^2 = \iint \kernel(\bfx,\bfy) d(\pp-\qq)(\bfx)d(\pp-\qq)(\bfy),
\end{equation*}
where the kernel $\kernel$ is defined by \eqref{eq:TIk} and \eqref{eq:freqdist}. Since $|\kernel|\leq 1$, we immediately obtain the result.
\end{proof}

\subsection{Covering numbers}\label{sec:def_covnum}

\begin{definition}[Ball, $\epsilon$-covering, Covering number]\label{def:covering}
Let $(X,d)$ be a metric space. For any $\epsilon >0$ and $x\in X$, we denote $B_X(x,\epsilon)$ the \textbf{ball} of radius $\epsilon$ centered at the point $x$:
\[
B_X(x,\epsilon)=\left\lbrace y\in X,~d(x,y)\leq\epsilon\right\rbrace.
\]

Let $Y \subseteq X$ be a subset of $X$. A subset $Z\subseteq Y$ is an \textbf{$\epsilon$-covering} of $Y$ if $Y\subseteq\bigcup_{z\in Z} B_X(z,\epsilon)$.

The \textbf{covering number} $N_{Y,d}(\epsilon) \in \mathbb{N}\cup\lbrace +\infty \rbrace$ is the smallest number of points $y_i \in Y$ such that the set $\{y_i\}$ is an $\epsilon$-covering of $Y$.
\end{definition}

\begin{remark}\label{rem:compact}
A subset $Y$ of a topological space $(X,d)$ that has finite covering numbers for any $\epsilon>0$ is called \emph{totally bounded} and is \emph{not} necessarily compact: a set is in fact compact if and only if it is totally bounded \emph{and} complete. Hence, though in the rest of the paper we often focus on explicitly bounding the covering numbers of certain sets, if compactness of these sets is required it will have to be proved independently.
\end{remark}

Our definition of covering numbers is that of \emph{internal} covering numbers, meaning that the centers of the covering balls are required to be included in the set being covered. Somewhat counter-intuitively these covering numbers (for a fixed radius $\epsilon$) are not necessarily increasing with the inclusion of sets\footnote{For instance, consider a set $A$ formed by two points, included in set $B$ which is a ball of radius $\epsilon$. Suppose those two points diametrically opposed in $B$. We have $A\subset B$, but two balls of radius $\epsilon$ are required to cover $A$ (since their centers have to be in $A$), while only one such ball is sufficient to cover $B$.}. We have instead the following property:

\begin{lemma}\label{lem:covnumsub}
Let $A \subseteq B \subseteq X$ be subsets of a metric space $(X,d)$, and $\epsilon>0$. Then,
\begin{equation}
\label{eq:covnumsub}
N_{A,d}(\epsilon) \leq N_{B,d}(\epsilon/2).
\end{equation}
\end{lemma}
\begin{proof}\review{Let $b_1,...,b_N$ be a $\epsilon/2$-covering of $B$. We construct a $\epsilon$-covering $a_i$ of $A$ in the following way. Each $b_i$ is either: a) in the set $A$, in which case we take $a_i=b_i$, b) at distance less than $\epsilon/2$ of a point $a\in A$, in which case we take $a_i=a$ and note that the ball centered on $a_i$ covers at least as much as the ball centered in $b_i$, i.e. $B_X(b_i,\epsilon/2)\subset B_X(a_i,\epsilon)$, c) in none of these cases and we discard it. There are less $a_i$'s than $b_i$'s, and the union of balls of radius $\epsilon$ with centers $a_i$ covers at least as much as the balls of radius $\epsilon/2$ with centers $b_i$, and therefore the set of $a_i$'s is a $\epsilon$-covering of $B$ and of $A$.}
\end{proof}

Another useful property is related to the embedding of sets by a Lipschitz function.

\begin{lemma}\label{lem:covnumlipschitz}
Let $(X,d)$ and $(X',d')$ be two metric spaces, and $Y \subseteq X$, $Y' \subseteq X'$. If there exists a surjective function $f: Y \rightarrow Y'$ which is $L$-Lipschitz with $L>0$, \ie~such that
\begin{equation*}
\forall x,y\in Y,~d'(f(x),f(y)) \leq L d(x,y),
\end{equation*}
then for all $\epsilon>0$ we have
\begin{equation}
\label{eq:covnumlipschitz}
N_{Y',d'}(\epsilon) \leq N_{Y,d}(\epsilon/L).
\end{equation}
\end{lemma}
\begin{proof}
Define $\epsilon_2=\epsilon/L$, denote $N=N_{Y,d}(\epsilon_2)$, and let $y_i \in Y$, $i=1,...,N$ be an $\epsilon_2$-covering of $Y$. Let any $y' \in Y'$. There exists $y \in Y$ such that $f(y)=y'$ since $f$ is surjective. Let $y_i$ be a center of a ball in the $\epsilon_2$-covering of $Y$, we have
\begin{equation*}
d'(y',f(y_i)) = d'(f(y),f(y_i)) \leq Ld(y,y_i) \leq L\epsilon_2 = \epsilon.
\end{equation*}
Thus $\{f(y_i)\}_{i=1,...,N}$ is an $\epsilon$-covering of $Y'$, and we have $N_{Y',d'}(\epsilon) \leq N$.
\end{proof}

Finally, we report a property from \cite{Cucker2002}:
\begin{lemma}[\cite{Cucker2002}, Prop. 5]\label{lem:covnumball}
Let $(X,\|.\|)$ be a Banach space of finite dimension $d$. Then for any $\epsilon>0,~x\in X$ and $R>0$ we have
\begin{equation}
N_{B_X(x,R),\|.\|}(\epsilon)\leq \left(\frac{4R}{\epsilon}\right)^d.
\end{equation}
\end{lemma}

\subsection{Concentration of averages}

We will use Bernstein's inequality in the following simple version \cite{Sridharan2002}:
\begin{lemma}[Bernstein's inequality (\cite{Sridharan2002}, Thm. 6)]\label{lem:bernstein}
Let $x_i \in \mathbb{R}$, $i=1,...,n$ be $i.i.d.$ bounded random variables such that $\mathbb{E}x_i=0$, $|x_i| \leq M$ and $Var(x_i) \leq \sigma^2$ for all $i$'s.

Then for all $t>0$ we have
\begin{equation}
P\left(\frac{1}{n}\sum_{i=1}^n x_i \geq t\right)\leq \exp\left(-\frac{nt^2}{2\sigma^2+2Mt/3}\right).
\end{equation}
\end{lemma}

We also report a concentration result in Hilbert spaces from \cite{Rahimi2009}.

\begin{lemma}[\cite{Rahimi2009}, Lemma 4]\label{lem:rahimi}
Let $\bfx_i \in \mathcal{H}$, $i=1,...,n$ be $i.i.d.$ random variables in a Hilbert Space $(\mathcal{H},\|.\|)$ such that $\|\bfx_i\|\leq M$ with probability one. Denote $\bar \bfx$ their empirical average $\bar \bfx=\left(\sum_{i=1}^n \bfx_i\right)/n$. Then for any $\rho>0$, with probability at least $1-\rho$,
\begin{equation}
\|\bar \bfx - \mathbb{E} \bar \bfx\| \leq \frac{M}{\sqrt{n}}\left(1+\sqrt{2\log\frac{1}{\rho}}\right).
\end{equation}
\end{lemma}

\section{Proof of Theorem \ref{thm:appli}}\label{sec:proof}

\subsection{Lower RIP}

A measurement operator $\skop$ satisfies the usual generalized Lower Restricted Isometry Property (LRIP) \cite{Bourrier2014} on the model $\model$ with constant $\beta>0$ if:
\begin{equation}
\label{eq:rip}
\forall \pp,\qq \in \model, ~ \beta \normK{\pp}{\qq}^2 \leq \|\skop \pp - \skop \qq\|_2^2.
\end{equation}

For a measurement operator $\skop$ drawn at random (in our case by randomly drawing a set of frequencies $\freqs=\{\freq_j\}_{j=1,...,m}$), the usual approach from compressive sensing theory is to prove that, with high probability, \eqref{eq:rip} is satisfied: 
\begin{equation}
\label{eq:unirip}
\PP_\freqs\left(\forall \pp,\qq \in \model, ~ \beta \normK{\pp}{\qq}^2 \leq \|\skop \pp - \skop \qq\|_2^2\right) \geq 1-\rho,
\end{equation}
where $\PP_\freqs$ indicates probability with respect to the set of frequencies $\freqs$. 

Defining the {\em normalized secant set}
\begin{equation}
\secant:=\left\lbrace \frac{\pp-\qq}{\normK{\pp}{\qq}}; ~ \pp,\qq \in \model, ~ \normK{\pp}{\qq}\neq 0\right\rbrace \subset E,
\end{equation}
the LRIP~\eqref{eq:rip} is equivalent to : $\forall \meas \in \secant, ~ \beta \leq \|\skop \meas\|_2^2$, 
and~\eqref{eq:unirip} is equivalent to 
\begin{equation}
\label{eq:uniripsecant}
\PP_\freqs\left(\forall \meas \in \secant, ~ \beta \leq \|\skop \meas\|_2^2\right) \geq 1-\rho.
\end{equation}
Hence, a typical proof of the (L)RIP \cite{Baraniuk2008, Puy2015} consists in defining an $\epsilon$-covering of the normalized secant set, proving a pointwise LRIP at the center of each ball using concentration results, then uniformly extending the result to the whole normalized secant set using 
Lipschitz continuity of the measurement operator.


\paragraph{Semi-uniform LRIP.} In our framework, we introduce a ``non-uniform'' version of the LRIP, in which the inequality \eqref{eq:unirip} will be verified for a \emph{given} $\pp \in \model$ with high probability, \emph{uniformly} for all $\qq \in \model$. It is expressed as:
\begin{equation}
\label{eq:semirip}
\forall \pp \in \model, ~ \PP_\freqs\left(\forall \qq \in \model, ~ \beta \normK{\pp}{\qq}^2 \leq \|\skop \pp - \skop \qq\|_2^2\right) \geq 1-\rho.
\end{equation}
We refer to this version of the LRIP as \emph{semi-uniform} in probability. It holds with a smaller number of measurements $m$ than the uniform case, and we show in the next section that it is \emph{sufficient} to obtain recovery guarantees with \emph{joint} probability on the drawing of frequencies $\{\freq_j\}$ and items $\{\bfx_i\}$. For more details on non-uniform compressive sensing results, we refer the reader to the book by Foucart and Rauhut \cite{Foucart2013}, Chaps. 9 and 11.

Similar to the uniform case, we introduce a family of ``non-uniform'' normalized secant sets, defined for each $\pp \in \model$ as:
\begin{equation}
\secant_\pp := \left\lbrace \frac{\pp-\qq}{\normK{\pp}{\qq}}; ~ \qq \in \model, ~ \normK{\pp}{\qq}\neq 0\right\rbrace \subset E.
\end{equation}
The semi-uniform LRIP \eqref{eq:semirip} is then equivalent to
\begin{equation}
\label{eq:uniripsecantnonuni}
\forall \pp \in \model,~\PP_\freqs\left(\forall \meas \in \secant_\pp, ~ \beta \leq \|\skop \meas\|_2^2\right) \geq 1-\rho.
\end{equation}
A typical proof would therefore follow the exact same pattern than the uniform case, using non-uniform normalized secant sets instead of the normalized secant set.

\paragraph{Restricted, semi-uniform LRIP.} Unlike finite dimensional frameworks, where normalized secant sets are contained in a unit ball that is necessarily compact, here it is in general challenging to prove the existence of finite covering numbers for this set. Under Assumption \ref{assum:compact}, the model $\model$ \emph{itself} is compact, which suggests using the embedding $\qq \in \model \rightarrow \frac{\pp-\qq}{\normK{\pp}{\qq}} \in \secant_\pp$. However the behavior of this function when $\qq$ gets close to $\pp$ may be delicate to analyze. Thus for all $\eta\geq0$ and $\pp\in\model$ we define the \emph{restricted} non-uniform normalized secant set:
\begin{equation}
\label{eq:secant}
\secant^\eta_\pp := \left\lbrace \frac{\pp-\qq}{\normK{\pp}{\qq}}; ~ \qq \in \model, ~ \normK{\pp}{\qq}>\eta\right\rbrace \subset E.
\end{equation}

Note that, when we let $\eta=0$ the restricted non-uniform normalized secant set $\secant^0_\pp$ is just the previous non-uniform normalized secant set $\secant_\pp$.

\paragraph{Hypotheses to establish the restricted, semi-uniform LRIP} We are going to prove the restricted semi-uniform LRIP~\eqref{eq:semirip} under two hypotheses. 
The first hypothesis depends on a model $\model$, a frequency distribution $\freqdist$, a non-negative constant $\eta \geq 0$, and a choice of metric $d(\cdot,\cdot)$.

\begin{hypothesis}{$\mathbf{H_1}(\eta,\model,\freqdist,d)$}[Covering numbers of the secant set]
\label{assum:covnumsecant}
For all $\pp \in \model$, the \emph{restricted} non-uniform normalized secant set $\secant^\eta_\pp$ has finite covering numbers 
$N_{\secant^\eta_\pp,d}(\epsilon)<\infty$.
\end{hypothesis}

In the case where the constant $\eta>0$ is positive and the metric $d=\|.\|$ is a norm, the covering numbers of the secant set are controlled by those of the model.

\begin{lemma}\label{lem:covnumsecant}
Let $\|.\|$ be a norm on the space of finite signed measure $E$, and $\freqdist$ a frequency distribution such that the model $\model$ has finite covering numbers with respect to some metric $\bar{d}$ which satisfies
\[
\forall \pp,\qq\in\model,~\bar{d}(\pp,\qq)\geq\max\Big(\|\pp-\qq\|,\normK{\pp}{\qq}\Big).
\]
The model is in particular bounded for the norm $\|.\|$, denote $C=\max\left(1,\sup_{\pp,\qq\in\model}\|\pp-\qq\|\right)$.

Then, for any \underline{strictly} positive constant $1\geq \eta > 0$, Assumption \ref{assum:covnumsecant} holds with $d = \|.\|$. Furthermore, for any $\pp\in\model$ and $\epsilon>0$ we have
\begin{equation}
N_{\secant^\eta_\pp,\|.\|}(\epsilon)\leq N_{\model,\bar{d}}\left(\frac{\epsilon\eta^2}{2(C+1)}\right).
\end{equation}
\end{lemma}

\begin{proof}
Let $\pp\in\model$ be any distribution in the model. 
Consider the complement
 of the ball $B_{\model,\normKs}(\pp,\eta)$:
\begin{equation}
\mathcal{Q}^\eta_\pp=B_{\model,\normKs}^c(\pp,\eta)=\left\lbrace \qq \in \model,~ \normK{\pp}{\qq} >\eta \right\rbrace \subset \model,
\end{equation}
and the function $f_\pp:\mathcal{Q}^\eta_\pp\rightarrow\secant^\eta_\pp$ such that $f_\pp(\qq)=\frac{\pp-\qq}{\normK{\pp}{\qq}}$, which is surjective by definition of $\secant^\eta_\pp$. Let us show that $f_P$ is $(C+1)/\eta^{2}$-Lipschitz continuous for the metric $\bar{d}$, and conclude with Lemma \ref{lem:covnumlipschitz}.

For any $\qq_1,\qq_2 \in \mathcal{Q}^\eta_\pp$, we have
\begin{align*}
\|f_\pp(\qq_1)-f_\pp(\qq_2)\| =& \left\lVert \frac{\pp-\qq_1}{\normK{\pp}{\qq_1}} - \frac{\pp-\qq_2}{\normK{\pp}{\qq_2}} \right\rVert, \\
\leq& \left\lVert \frac{\pp-\qq_1}{\normK{\pp}{\qq_1}} - \frac{\pp-\qq_2}{\normK{\pp}{\qq_1}} \right\rVert + \left\lVert \frac{\pp-\qq_2}{\normK{\pp}{\qq_1	}} - \frac{\pp-\qq_2}{\normK{\pp}{\qq_2}} \right\rVert, \\
\eqcomment{since $\normK{\pp}{\qq_1}> \eta$}\leq& \frac{1}{\eta}\|\qq_2-\qq_1\| + \|\pp-\qq_2\|\left\lvert\frac{1}{\normK{\pp}{\qq_1}}-\frac{1}{\normK{\pp}{\qq_2}}\right\rvert, \\
\eqcomment{since $\|\pp-\qq_2\|\leq C$}\leq& \frac{1}{\eta}\|\qq_1-\qq_2\| +\frac{C}{\eta^2}\Big| \normK{\pp}{\qq_2}-\normK{\pp}{\qq_1}\Big|, \\
\eqcomment{by the triangle inequality,}\leq& \frac{1}{\eta}\|\qq_1-\qq_2\| +\frac{C}{\eta^2}\normK{\qq_1}{\qq_2}, \\
\eqcomment{since $\eta\leq 1$}\leq& \frac{C+1}{\eta^2}\bar{d}\left(\qq_1,\qq_2\right).
\end{align*}

Hence the function $f_\pp$ is Lipshitz continuous with constant $L=(C+1)/\eta^2$, and therefore for all $\epsilon>0$:
\begin{equation*}
N_{\secant^\eta_\pp,\|.\|}(\epsilon) \stackrel{\text{Lemma \ref{lem:covnumlipschitz}}}{\leq} N_{\mathcal{Q}^\eta_\pp,\bar{d}}(\epsilon/L) \stackrel{\text{Lemma \ref{lem:covnumsub}}}{\leq} N_{\model,\bar{d}}\left(\frac{\epsilon}{2L}\right).
\end{equation*}
\end{proof}

To formulate the second hypothesis, we denote $f$ a function from $\mathbb{N} \times \mathbb{R}_+$ to $\mathbb{R}_+$.

\begin{hypothesis}{$\mathbf{H_2}(\eta,\model,\freqdist,f)$}[Probability of the pointwise LRIP]\label{assum:bernstein}
For any $\pp \in \model$, any $\meas \in \secant^\eta_\pp$, any $t\geq 0$ and any integer $m>0$, we have
\begin{equation}
\PP_\freqs\left(1-\|\skop \meas\|_2^2\geq t \right) \leq f(m,t),
\end{equation}
where $\skop:E\rightarrow \mathbb{C}^m$ is a sketching operator built by independently drawing $m$ frequencies according to $\freqdist$.
\end{hypothesis}
In Section \ref{sec:thmbis}, under hypotheses \ref{assum:covnumsecant} and \ref{assum:bernstein}, we prove an extended version of Theorem \ref{thm:appli} (referred to as Theorem~\ref{thm:applibis}). Unlike Theorem~\ref{thm:appli} 
this extended version covers the case $\eta=0$.

Then, in Section \ref{sec:proofthm} we prove that under the Assumptions \ref{assum:compact} and \ref{assum:domin} used to state Theorem \ref{thm:appli}, \ref{assum:covnumsecant} holds with $d(\meas,\meas') = \|\meas-\meas'\|_{TV}$ \emph{provided that $\eta>0$}, and \ref{assum:bernstein} holds for an appropriate choice of function $f$.
\begin{remark}
Ideally, one would like to exploit Theorem \ref{thm:applibis} with $\eta=0$ to obtain performance guarantees without the extra additive term $\eta$. However, putting this into practice would require characterizing covering numbers for the normalized secant set $\secant_\pp=\secant_\pp^{0}$, which can be tricky and is left to future work. It is indeed already not trivial to determine when this set has finite covering numbers. 
\end{remark}


We can now state our version of the LRIP.

%

\begin{theorem}\label{thm:LRIP}
Consider a model $\model$,  a frequency distribution $\freqdist$, a non-negative constant $\eta\geq0$, a metric $d(\cdot,\cdot)$ and a function $f$ such that Assumptions \ref{assum:covnumsecant} and \ref{assum:bernstein} are satisfied.

Let $\pp^* \in \model$ be any distribution in the model.

Assume that for all $\meas,\meas' \in \secant^{\eta}_{\pp^*}$
\begin{equation}\label{eq:HypSkopLipschitz}
\sup_\skop\Big|\|\skop \meas\|_{2} -\| \skop \meas'\|_{2}\Big| \leq d(\meas,\meas'),
\end{equation}
where the supremum is over all possible frequencies $\freq$ defining the sketching operator $\skop$.

Define
\[
\rho=N_{\secant^\eta_{\pp^*},d
}\left(\tfrac{1}{4}\right)\cdot f\left(m,\tfrac{7}{16}\right).
\]

Then, with probability at least $1-\rho$ on the drawing of the $m$ frequencies $\freq_j$'s, we have
\begin{equation}
\label{eq:riplem}
\forall \pp \in \model ~\text{s.t. } \normK{\pp^*}{\pp}\geq \eta:\quad \frac{1}{4}\normK{\pp^*}{\pp}^2 \leq \|\skop \pp^* - \skop \pp\|_2^2.
\end{equation}
\end{theorem}

\begin{proof}
The idea is to define an $\epsilon$-covering of the restricted non-uniform normalized secant set $\secant_{\pp^*}^\eta$ with respect to the metric $d$, to apply the concentration result of Assumption \ref{assum:bernstein} at the center of each ball and to conclude with a union bound.

Let $\epsilon > 0$, $1>t>0$ and denote
 $N=N_{\secant_{\pp^*}^\eta,d
 }(\epsilon)$ for simplicity, which is finite by Assumption \ref{assum:covnumsecant}. We consider $\meas_1,...,\meas_{N}$ an $\epsilon$-covering of $\secant_{\pp^*}^\eta$ with respect to the metric $d$. 
 Considering Assumption \ref{assum:bernstein}, a union bound yields that, with probability greater than 
 $1-N \cdot f(m,t)$,
\begin{equation}\label{eq:rip_bernstein_union}
\forall i \in [1,N],~  1-\left\lVert \skop\meas_i\right\rVert^2_2 < t. 
\end{equation}

Assuming \eqref{eq:rip_bernstein_union} is satisfied, let any distribution $\pp \in  \model$ such that $\normK{\pp^*}{\pp}>\eta$. Denote $\meas:=\frac{\pp^*-\pp}{\normK{\pp^*}{\pp}} \in \secant_{\pp^*}^\eta$, and let $i\in[1,N]$ such that $\meas_{i}$ is the center of the ball closest to $\meas$ in the covering. We have
\begin{align}
 1- \|\skop\meas\|_2  =&  1 - \|\skop\meas_i\|_2 +  \|\skop\meas_i\|_2  - \|\skop\meas\|_2  \notag\\
\eqcomment{\text{since \eqref{eq:rip_bernstein_union} is verified}} 
\leq & 1-\sqrt{1-t} + \|\skop \meas_i\|_2-\|\skop\meas\|_2 \notag\\
\eqcomment{\text{by \eqref{eq:HypSkopLipschitz}}}
\stackrel{(a)}\leq& 1-\sqrt{1-t} + d(\meas_i,\meas)
\\
\leq& 1-\sqrt{1-t} + \epsilon\notag.
\end{align}
Choosing $t>0$ and $\epsilon = \sqrt{1-t}-1/2$ (for example: $t = 7/16$, $\epsilon=1/4$), we obtain $\|\skop \meas\|_{2}^{2} \geq 1/4$, that is to say
\begin{equation}
\frac{1}{4}\normK{\pp^*}{\pp}^2 \leq \|\skop \pp^*-\skop \pp\|_2^2.
\end{equation}
This shows that the LRIP \eqref{eq:riplem} is verified except with probability at most
\[
N_{\secant^\eta_{\pp^*},d}\left(\sqrt{1-t}-1/2\right)\cdot f\left(m,t\right)\, .
\]

Specializing to $t=7/16$ 
yields the desired result.
\end{proof}
\begin{remark}[LRIP without the restricted Lipschitz property~\eqref{eq:HypSkopLipschitz} ?] 
In (a) we used Property~\eqref{eq:HypSkopLipschitz}, which could be called a {\em restricted Lipschitz property} for the function $\meas \mapsto \|\skop \meas\|_{2}$. It is restricted to $\secant^{\eta}_{\pp^*}$, but assumed to hold uniformly over all possible draws of the sketching operator $\skop$. 
Thanks to Lemma~\ref{lem:bounds} and the triangle inequality,
\[
\Big| \|\skop \meas\|_{2}-\|\skop \meas'\|_{2}\Big| \leq \|\skop (\meas-\meas')\|_{2} \leq \|\meas-\meas'\|_{TV},
\]
hence property~\eqref{eq:HypSkopLipschitz} indeed holds when $d(\meas,\meas') = \|\meas-\meas'\|_{TV}$, even without the restriction to $\secant^{\eta}_{\pp^*}$. In the rest of this paper we primarily concentrate on this setting.

It would however be interesting for future work to consider metrics $d$ with much larger balls than those of the total variation norm, since they may lead to substantially smaller covering numbers $N_{\secant^\eta_{\pp^*},\|.\|} \ll N_{\secant^\eta_{\pp^*},\|.\|_{TV}}$, and thus LRIP guarantees for much smaller $m$. 
Theorem~\ref{thm:LRIP} still yields guarantees with such metrics, provided the restricted Lipschitz property~\eqref{eq:HypSkopLipschitz} holds. 
An analog of Theorem~\ref{thm:LRIP}  for even weaker metrics, that do not satisfy~\eqref{eq:HypSkopLipschitz}, can also be envisioned using chaining arguments \cite{Puy2015} provided that
\[
\PP_{\freqs}\left(\|\skop \meas - \skop \meas'\|_{2} \geq (1+t) d(\meas,\meas')\right) \leq g(m,t)
\]
with an appropriately decaying function $g(m,t)$.
\end{remark}
%

\subsection{A version of Theorem \ref{thm:appli} with weaker assumptions}\label{sec:thmbis}

In this section, we formulate a version of Theorem \ref{thm:appli}, referred to as Theorem \ref{thm:applibis}, that allows for an additive error $\eta=0$, under the Hypotheses $\mathbf{H_i}$. In the next section we deduce from it the version given in Section \ref{sec:thm}, that uses 
Assumptions $\mathbf{A_i}$.

\begin{mythm}{\ref{thm:appli} bis}\label{thm:applibis}
Consider a model $\model$, a frequency distribution $\freqdist$, a non-negative constant $\eta\geq 0$ and a function $f$ such that Assumptions \ref{assum:covnumsecant} and \ref{assum:bernstein} hold with $d(\meas,\meas') = \|\meas-\meas'\|_{TV}$.

Consider $\tilde{d}(\cdot,\cdot)$ a metric such that, for any $\pp_{1},\pp_{2} \in \pmeas$ we have
\begin{equation}
\label{eq:DefMNormLikeCondition}
\tilde{d}(\pp_{1},\pp_{2}) \geq \max\left(\normK{\pp_{1}}{\pp_{2}}, \sup_{\skop} \|\skop \pp_{1}-\skop \pp_{2}\|_{2}\right),
\end{equation}
where the supremum is over all possible frequencies $\freq$ defining the sketching operator $\skop$.

Assume that $\model$ is compact\footnote{Compactness of the model is still required here, in order to define a projection operator onto it as well as to ensure at least one solution to the problem \eqref{eq:problem}. This assumption could be relaxed, with the addition of an 
arbitrary small additive error, similar to \cite{Bourrier2014a}.} with respect to $\tilde{d}$, and note that under this assumption the decoder $\decod$ is still well-defined by \eqref{eq:decod}, since one can replace the total variation norm by the metric $\tilde{d}$ in the r.h.s. of \eqref{eq:continuityfordecoder}.

Let $\bfx_i \in \mathbb{R}^d$, $i=1...n$ be $n$ points drawn $i.i.d.$ from an arbitrary distribution $\pp^*\in\pmeas$, and $\freq_{j} \in \mathbb{R}^d$, $j=1...m$ be $m$ frequencies drawn $i.i.d.$ from $\freqdist$. Denote $\bar\pp=\decod(\hat\bfz,\skop)$ the distribution reconstructed from the empirical sketch $\hat\bfz$.

Define $\pproj \in \model$ as (one of) the projection(s) of the probability $\pp^*$ onto the model:
\begin{equation}\label{eq:proj}
\pproj \in \argmin{\pp\in \model} \ \tilde{d}(\pp^{*},\pp),
\end{equation}
which indeed exists since $\model$ is assumed to be compact.

Define
\begin{equation}\label{eq:thmbisprb}
\rho = 2N_{\secant^\eta_{\pp_{proj}},\|.\|_{TV}}\left(\tfrac{1}{4}\right)\cdot f\left(m,\tfrac{7}{16}\right).
\end{equation}

Then, with probability at least $1-\rho$ on the drawing of the items $\bfx_i$ \underline{and} sampling frequencies $\freq_{j}$, we have
\begin{equation}
\normK{\pp^*}{\bar\pp}\leq 5\ \tilde{d}(\pp^{*},\model)
+\tfrac{4(1+\sqrt{2\log(2/\rho)})}{\sqrt{n}}+\eta,
\end{equation}
where 
$\tilde{d}(\pp^*,\model)=\inf_{\pp \in \model}\ \tilde{d}(\pp^*,\pp)$ 
is the distance from $\pp^*$ to the model.
\end{mythm}

\begin{proof}

Recall that the target distribution $\pp^*$ and its projection $\pproj$ are fixed. By Lemma \ref{lem:bounds}, the restricted Lipschitz property~\eqref{eq:HypSkopLipschitz} holds with $d(\meas,\meas') = \|\meas-\meas'\|_{TV}$. Considering \eqref{eq:thmbisprb}, Theorem \ref{thm:LRIP} yields that since Assumptions \ref{assum:covnumsecant} and \ref{assum:bernstein} hold, the LRIP applied to $\pproj$ is satisfied with probability at least $1-\rho/2$ on the drawing of frequencies:
\begin{equation*}
\forall \pp_\model \in \model ~\text{s.t. } \normK{\pproj}{\pp_\model}\geq \eta:\quad \frac{1}{4}\normK{\pproj}{\pp_\model}^2 \leq \|\skop \pproj - \skop \pp_\model\|_2^2,
\end{equation*}
which can be reformulated in:
\begin{equation}
\label{eq:applilrip}
\forall \pp_\model \in \model,~\normK{\pproj}{\pp_\model} \leq \max\Big(2\|\skop \pp^* - \skop \pp_\model\|_2,\eta\Big) \leq 2\|\skop \pp^* - \skop \pp_\model\|_2+\eta.
\end{equation}

Let $\pp \in \pmeas$ be any distribution. For some random draw of the operator $\skop$, denote $\bar{\pp}=\decod\left(\skop \pp,\skop\right) \in \model$ which, by definition of the decoder, belongs to the model. We have:
\begin{align*}
\eqcomment{\text{triangle ineq.}}
\normK{\bar{\pp}}{\pp^*} \leq& \normK{\bar{\pp}}{\pproj} + \normK{\pproj}{\pp^*} \\
\eqcomment{\text{using \eqref{eq:applilrip}}}
\leq& 2\|\skop \bar{\pp}-\skop \pproj\|_2 + \normK{\pproj}{\pp^*} + \eta \\
\eqcomment{\text{triangle ineq.}}
\leq& 2\|\skop \bar{\pp}-\skop \pp\|_2 + 2\|\skop \pp-\skop \pp^* \|_2 +  2\|\skop \pp^*-\skop \pproj \|_2 + \normK{\pproj}{\pp^*} + \eta
\end{align*}
Given the definition \eqref{eq:decod} of the decoder $\decod$ and of the distribution $\bar{\pp}=\decod\left(\skop\pp,\skop\right)$, we have
\begin{equation}
\|\skop \bar{\pp}-\skop \pp\|_2 = \min_{\pp_\model\in\model} \|\skop \pp_\model-\skop \pp\|_2.
\end{equation}
Since $\pproj$ is in the model $\model$, we thus have $\|\skop \bar{\pp}-\skop \pp\|_2 \leq \|\skop\pproj-\skop \pp\|_2$. Hence
\begin{align*}
\normK{\bar{\pp}}{\pp^*} 
\leq& 2\|\skop\pproj-\skop \pp\|_2 + 2\|\skop \pp-\skop \pp^* \|_2 + 2\|\skop \pp^*-\skop \pproj \|_2 + \normK{\pproj}{\pp^*} +\eta\\
\eqcomment{\text{triangle ineq.}}
\leq& 4\|\skop \pp-\skop \pp^* \|_2 +  4\|\skop \pp^*-\skop \pproj \|_2 + \normK{\pproj}{\pp^*}  \\
\eqcomment{using \eqref{eq:DefMNormLikeCondition} and \eqref{eq:proj}}\stackrel{(b)}{\leq}& 4\|\skop \pp-\skop \pp^* \|_2 + 5\ \tilde{d}(\pp^*,\model) +\eta.
\end{align*}
Thus we proved that, with probability at least $1-\rho/2$ on the drawing of frequencies, we have:
\begin{equation*}
\forall \pp\in\pmeas,~ \normK{\pp^*}{\decod\left(\skop\pp,\skop\right)} \leq 5\ \tilde{d}(\pp^*,\model)+ 4\|\skop \pp-\skop \pp^* \|_2 +\eta.
\end{equation*}
In particular, with a joint probability of at least $1-\rho/2$ on the drawing of frequencies $\freq_j$ \emph{and} items $\bfx_i$, we have
\begin{equation}
\label{eq:step1}
\normK{\pp^*}{\decod\left(\skop\hat\pp,\skop\right)} \leq 5\ \tilde{d}(\pp^*,\model)+ 4\|\skop \hat \pp-\skop \pp^* \|_2 +\eta,
\end{equation}
where $\hat \pp=\frac{1}{n}\sum_{i=1}^n \delta_{\bfx_i}$.

We now show that with high probability the term $\|\skop \hat \pp-\skop \pp^* \|_2$ is bounded by $\epsilon:=\frac{(1+\sqrt{2\log(2/\rho)})}{\sqrt{n}}$. Denote $\PP_\freqs$ (resp. $\PP_\dataset$) the probability distribution of the set of frequencies $\freqs=\{\freq_j\}_{j=1,...,m}$ (resp. of the set of items in the database $\dataset=\{\bfx_i\}_{i=1,...,n}$). Their joint distribution is denoted $\PP_{\freqs,\dataset}$, and is such that $d\PP_{\freqs,\dataset}(\freqs,\dataset)=d\PP_\freqs(\freqs)d\PP_\dataset(\dataset)$ by independence.

Consider the set:
\begin{equation*}
A:=\left\lbrace (\freqs,\dataset) \text{ s.t. } \|\skop \hat \pp - \skop \pp^*\|_2 \leq \epsilon \right\rbrace.
\end{equation*}
For a \emph{fixed} measurement operator $\skop$, we use Lemma \ref{lem:rahimi} on the random variables $\skop \delta_{\bfx_i}$ in $\mathbb{C}^m$. We observe that $\|\skop\delta_{\bfx_i}\|_2 = 1$, $\skop \hat \pp=\left(\sum_{i=1}^n \skop \delta_{\bfx_i}\right)/n$ and $\skop \pp^*=\tfrac{1}{\sqrt{m}}\left[\mathbb{E}_{\bfx \sim \pp^{*}}e^{-\imaginaryi \freq_j^T\bfx}\right]_{j=1...m}=\mathbb{E}_{\bfx \sim \pp^{*}}\skop \delta_{\bfx}$. 
Hence, with probability at least $1-\rho/2$ on the drawing of items:
\begin{equation*}
\|\skop \hat \pp - \skop \pp^*\|_2 \leq \epsilon.
\end{equation*}
In other words, we obtain the following result:
\begin{equation}\label{eq:empbound}
\forall \freqs,~ \int_{\dataset}  \mathbf{1}_A(\freqs,\dataset) d\PP_{\dataset}(\dataset) \geq 1-\rho/2.
\end{equation}

Hence,
\begin{align*}
\iint_{\freqs,\dataset} \mathbf{1}_A(\freqs,\dataset) d\PP_{\freqs,\dataset}(\freqs,\dataset) &= \int_\freqs \left(\int_{\dataset} \mathbf{1}_A(\freqs,\dataset) d\PP_{\dataset}(\dataset) \right) d\PP_\freqs(\freqs)  \notag \\
\geq& (1-\rho/2)\int_\freqs d\PP_\freqs(\freqs) = 1-\rho/2,  
\end{align*}
meaning that, with probability at least $1-\rho/2$ on the drawing of frequencies \emph{and} items, we have
\begin{equation}
\label{eq:step2}
\|\skop \hat \pp - \skop \pp^*\|_2 \leq \epsilon.
\end{equation}

We can now conclude: a union bound yields that \eqref{eq:step1} and \eqref{eq:step2} are simultaneously satisfied with probability at least $1-\rho$, which leads to the desired result.

\end{proof}

\begin{remark}
In (b) we used inequality \eqref{eq:DefMNormLikeCondition}, in particular the assumption that the inequality
\begin{equation}\label{eq:TmpAssumption}
\tilde{d}(\pp_{1},\pp_{2}) \geq \|\skop \pp_{1}-\skop \pp_{2}\|_{2}
\end{equation}
holds {\em uniformly} over all possible choices of frequencies defining $\skop$.
Thanks to Lemma~\ref{lem:bounds}, the inequality~\eqref{eq:DefMNormLikeCondition} indeed holds with $\tilde{d} = d_{TV}$. In the rest of this paper, we concentrate on this setting. 
It would however be interesting for future work to consider metrics $\tilde{d}$ much closer to the natural choice $\normK{}{}$, since they may lead to sharper upper bounds. A possibility would be to relax~\eqref{eq:DefMNormLikeCondition} and only assume that inequality~\eqref{eq:TmpAssumption} (possibly up to multiplicative constants) holds with high probability on the draw of $\skop$, given a pair $\pp_{1} = \pp^{*}$ and $\pp_{2}=\pp_{proj}$. 
\end{remark}

\subsection{Proof of Theorem \ref{thm:appli}}\label{sec:proofthm}

We now turn to the proof of Theorem \ref{thm:appli}, which consists in proving that the Assumptions $\mathbf{A_i}$ imply the hypotheses $\mathbf{H_i}$ in the $\eta>0$ case, and applying Theorem \ref{thm:applibis}.

For Hypothesis \ref{assum:covnumsecant}, we apply Lemma \ref{lem:covnumsecant} with the total variation norm $\|.\|=\bar{d}=\|.\|_{TV}$ to obtain the following corollary.

\begin{corollary}\label{cor:covnumsecant}
Let a model $\model$ such that Assumption \ref{assum:compact} is satisfied, and a positive constant $\eta>0$.

Then, for any frequency distribution $\freqdist$, Hypothesis \ref{assum:covnumsecant} is satisfied with $d=\|.\|_{TV}$, and we have
\begin{equation}
N_{\secant^\eta_\pp,\|.\|_{TV}}(\epsilon)\leq N_{\model,\|.\|_{TV}}\left(\frac{\epsilon\eta^2}{6}\right).
\end{equation}
\end{corollary}

Hypothesis \ref{assum:bernstein} is an application of Bernstein's inequality.

\begin{lemma}\label{lem:applibernstein}
Consider a model $\model$, a frequency distribution $\freqdist$, a non-negative constant $\eta\geq0$ and a constant $\cstdom$ such that Assumption \ref{assum:domin} is satisfied.

Then Assumption \ref{assum:bernstein} is satisfied with the function $f$ defined as
\begin{equation}
f(m,t)=\exp\left(-\frac{m}{2\cstdom^2}\cdot \frac{t^2}{1+t/3}\right).
\end{equation}
\end{lemma}

\begin{proof}
Fix $\pp^*\in\model$. Suppose $\freq_j$, $j=1...m$ are drawn $i.i.d.$ from $\freqdist$ and let $\skop$ be the corresponding sketching operator. Let any $\meas \in \secant^\eta_{\pp^*}$, denote $\pp\in\model$ such that $\meas=\frac{\pp^*-\pp}{\normK{\pp^*}{\pp}}$. Denote $Z_j=1-\frac{|\chrc_{\pp^*}(\freq_j)-\chrc_\pp(\freq_j)|^2}{\normK{\pp^*}{\pp}^2}$. Since Assumption \ref{assum:domin} is verified and $|\chrc_{\pp^*}(\freq_j)-\chrc_\pp(\freq_j)|\leq \|\pp^*-\pp\|_{TV}$ for all frequencies, the $Z_j$'s are $i.i.d.$ random variables verifying $Z_j \in [1-\cstdom^2;1]$. Furthermore, according to Lemma \ref{lem:bounds} we have necessarily $\cstdom\geq 1$ and thus we have
\begin{equation}
|Z_j| \leq \cstdom^2.
\end{equation}
The $Z_j$'s are also centered:
\begin{align*}
\eqcomment{using \eqref{eq:chrcdiff}} \mathbb{E}_{\freq_l}Z_j=1-\frac{\mathbb{E}_{\freq\sim\freqdist}|\chrc_{\pp^*}(\freq)-\chrc_\pp(\freq)|^2}{\normK{\pp^*}{\pp}^2}=0.
\end{align*}
Furthermore, we have 
\begin{align*}
Var(Z_j)=&Var\left(\frac{|\chrc_{\pp^*}(\freq)-\chrc_\pp(\freq)|^2}{\normK{\pp^*}{\pp}^2}\right) \leq\frac{\mathbb{E}|\chrc_{\pp^*}(\freq)-\chrc_\pp(\freq)|^4}{\normK{\pp^*}{\pp}^4} \leq \frac{\|\pp^{*}-\pp\|_{TV}^{2} \cdot \mathbb{E}|\chrc_{\pp^*}(\freq)-\chrc_\pp(\freq)|^2}{\normK{\pp^*}{\pp}^2 \cdot \normK{\pp^*}{\pp}^2}\\
\leq&\cstdom^2 \cdot \frac{\mathbb{E}|\chrc_{\pp^*}(\freq)-\chrc_\pp(\freq)|^2}{\normK{\pp^*}{\pp}^2} = \cstdom^2.
\end{align*}
Since $\frac{1}{m}\sum_{j=1}^m Z_j=1-\|\skop\meas\|_2^2$, applying Bernstein's inequality (Lemma \ref{lem:bernstein}) we get : for all $t>0$,
\begin{equation}\label{eq:rip_bernstein}
\forall \meas \in \secant_{\pp^*}^\eta,~\PP_\freqs\left(1-\|\skop \meas \|^2_2 \geq t\right) 
= 
\PP_\freqs\left(\tfrac{1}{m}\sum_{j=1}^m Z_j \geq t\right)
\leq \exp\left(-\frac{m}{2\cstdom^2} \cdot \frac{t^{2}}{1+t/3}\right).
\end{equation}
\end{proof}

We can finally prove Theorem \ref{thm:appli}, by combining Corollary \ref{cor:covnumsecant}, Lemma \ref{lem:applibernstein} and Theorem \ref{thm:applibis}. For simplification, we also use the following bound on the function $f$ defined in Lemma \ref{lem:applibernstein}:
\[
f(m,7/16)=\exp\left(-\frac{147 m}{1760\cstdom^2}\right) \leq \exp\left(-\frac{m}{12\cstdom^2}\right).
\]

\section{Application to GMMs}\label{sec:add_proofs}

\subsection{Compactness of mixture models}\label{sec:proofs_mixture}

In this section we prove Lemma \ref{lem:covnummix}, which extends the compactness of the set $\gaussset$ of basic distributions to the corresponding mixture model $\gaussmix{K}$.

\begin{proof}[Proof of Lemma \ref{lem:covnummix}]
Recall that we have assumed compactness of the set $\gaussset$ with respect to some norm $\|.\|$. In particular, it is bounded, and we note $C=\max_{\pp\in\gaussset}\|\pp\|$.
Let $\epsilon>0$ and $\tau \in ]0;1[$. Denote $\epsilon_1=\tau\epsilon$ and $\epsilon_2=(1-\tau)\epsilon/C$. Also denote $N_1=N_{\gaussset,\|.\|}(\epsilon_1)$ and let $\enet_1=\{\pp_1,...,\pp_{N_1}\}$ be an $\epsilon_1$-covering of $\gaussset$.

Similarly, let $B_1^+=\{\malpha \in \mathbb{R}_+^K, \sumk \alpha_k=1\}$, denote $N_2=N_{B_1^+,\|.\|_1}(\epsilon_2)$, let $\enet_2=\{\malpha_1,...,\malpha_{N_2}\}$ be an $\epsilon_2$-covering of $B_1^+$. Denote $B_1:=B_{\mathbb{R},\|.\|_1}(0,1)$ the unit $\ell_1$-ball in $\mathbb{R}^K$, note that $B_1^+ \subset B_1$, such that we have
\begin{equation}
N_2 = N_{B_1^+,\|.\|_1}(\epsilon_2)\stackrel{\text{Lemma \ref{lem:covnumsub}}}{\leq} N_{B_1,\|.\|_1}(\epsilon_2/2) \stackrel{\text{Lemma \ref{lem:covnumball}}}{\leq} (8/\epsilon_2)^K. \label{eq:covnumalpha}
\end{equation}

Define the following set:
\begin{equation}
\gmmcover:=\left\lbrace \pp_{\hyppar,\malpha}\in\gaussmix{K};~ \forall k,~ \pp_{\mtheta_k} \in \enet_1,~ \malpha \in \enet_2\right\rbrace.
\end{equation}
The cardinality of this set verifies $\left\lvert \gmmcover\right\rvert \leq \left(\left\lvert\enet_1\right\rvert\right)^K \left\lvert\enet_2\right\rvert=(N_1)^K N_2$.

Let us show that $\gmmcover$ is an $\epsilon$-covering of $\gaussmix{K}$. Let $\pp_{\hyppar,\malpha} \in \gaussmix{K}$ be any $K$-sparse distribution. For all $k=1...K$, let $\pp_{\bar\mtheta_k} \in \enet_1$ be the distribution in $\enet_1$ which is the closest to $\pp_{\mtheta_k}$, and let $\bar\malpha \in \enet_2$ be the weight vector in $\enet_2$ that is the closest to $\malpha$. Denote $\bar\hyppar=\{\bar\mtheta_1,...,\bar\mtheta_K\}$, and note that $\pp_{\bar\hyppar,\bar\malpha} \in \gmmcover$. We have
\begin{align}
\|\pp_{\hyppar,\malpha}-\pp_{\bar\hyppar,\bar\malpha}\|=& \left\lVert \sumk \alpha_k \pp_{\mtheta_k} - \sumk \bar\alpha_k \pp_{\bar\mtheta_k} \right\rVert, \notag \\
\leq& \left\lVert \sumk \alpha_k \pp_{\mtheta_k}- \sumk \alpha_k \pp_{\bar\mtheta_k} \right\rVert + \left\lVert \sumk \alpha_k \pp_{\bar\mtheta_k} - \sumk \bar\alpha_k \pp_{\bar\mtheta_k} \right\rVert, \notag \\
\leq& \sumk \alpha_k \left\lVert \pp_{\mtheta_k} - \pp_{\bar\mtheta_k} \right\rVert + \sumk |\alpha_k-\bar\alpha_k| \left\lVert \pp_{\bar\mtheta_k} \right\rVert, \notag \\
\leq& \sumk \alpha_k \|\pp_{\mtheta_k}-\pp_{\bar\mtheta_k}\| + C\|\malpha-\bar\malpha\|_1, \label{eq:lipgmm}\\
\leq& \epsilon_1 \sumk \alpha_k + C\epsilon_2 = \epsilon_1 + C\epsilon_2 = \epsilon \notag ,
\end{align}
and $\gmmcover$ is indeed an $\epsilon$-covering of $\gaussmix{K}$. Therefore, we have the bound (for all $\tau$)
\begin{equation*}
N_{\gaussmix{K},\|.\|}(\epsilon) \leq \left\lvert \gmmcover\right\rvert \leq (N_1)^K N_2 \stackrel{\text{by \eqref{eq:covnumalpha}}}{\leq} \left(\frac{8C\cdot N_{\gaussset,\|.\|
}(\tau\epsilon)}{(1-\tau)\epsilon}\right)^K.
\end{equation*}

Furthermore, in equation \eqref{eq:lipgmm}, we have shown in particular that for all $\pp_{\hyppar,\malpha},\pp_{\hyppar',\malpha'} \in \gaussmix{K}$,
\begin{align*}
\|\pp_{\hyppar,\malpha}-\pp_{\hyppar',\malpha'}\|\leq& \sumk \alpha_k\|\pp_{\mtheta_k}-\pp_{\mtheta'_k}\|
 + \|\malpha-\malpha'\|_1
\end{align*}
and therefore the embedding $(\pp_{\mtheta_1},...,\pp_{\mtheta_K},\malpha)\rightarrow \pp_{\hyppar,\malpha}$ from $(\gaussset)^K\times B_1^+$ to $\gaussmix{K}$ is continuous. Hence $\gaussmix{K}$ is the continuous image of the set $(\gaussset)^K\times B_1^+$, which is compact since $\gaussset$ is compact, and therefore $\gaussmix{K}$ is compact.
\end{proof}

\subsection{Covering numbers of Gaussians}\label{sec:proofs_covnum_gaussians}

\begin{proof}[Proof of Theorem \ref{thm:covnum}]
Consider the embedding $\embd: \thetaset \rightarrow \gaussset$ defined as $\embd(\mtheta)= \pp_\mtheta$, which is surjective by definition of $\gaussset$. We show that $\embd$ is Lipschitz continuous, for the Euclidean norm on $\thetaset \subset \mathbb{R}^{2n}$ and total variation norm on $\gaussset \subset E$.

We begin by the classical Pinsker's inequality \cite{Fedotov2003}:
\begin{equation}
\label{eq:pinsker}
\|\pp-\qq\|_{TV} \leq \sqrt{2D_{KL}(\pp||\qq)},
\end{equation}
where $D_{KL}$ is the Kullback-Leibler divergence. By symmetry, we have:
\begin{equation}
\label{eq:pinsker2}
\|\pp-\qq\|^2_{TV} \leq D_{KL}(\pp||\qq)+D_{KL}(\qq||\pp).
\end{equation}

The Kullback-Leibler divergence has a closed-form expression in the case of multivariate Gaussians \cite{Duchi2007}:
\begin{equation}
\label{eq:kl_gauss}
D_{KL}(\pp_{\mtheta_1}||\pp_{\mtheta_2}) = \frac{1}{2}\left[\log\frac{|\mSigma_2|}{|\mSigma_1|} + \mathrm{tr}\left(\mSigma_2^{-1}\mSigma_1\right)-d + \left(\mmu_2-\mmu_1\right)^T \mSigma_2^{-1}\left(\mmu_2-\mmu_1\right) \right].
\end{equation}
In our case, with diagonal Gaussians and bounded parameters, we have
\begin{align*}
D_{KL}(\pp_{\mtheta_1}||\pp_{\mtheta_2}) + D_{KL}(\pp_{\mtheta_1}||\pp_{\mtheta_2}) =& \frac{\mathrm{tr}\left(\mSigma_2^{-1}\mSigma_1\right) + \mathrm{tr}\left(\mSigma_1^{-1}\mSigma_2\right)}{2}-d + \left(\mmu_2-\mmu_1\right)^T \frac{\mSigma_2^{-1}+\mSigma_1^{-1}}{2}\left(\mmu_2-\mmu_1\right) \\
=& \frac12 \sum_{\ell=1}^d \left(\frac{\sigma^2_{1,\ell}}{\sigma^2_{2,\ell}}+\frac{\sigma^2_{2,\ell}}{\sigma^2_{1,\ell}}-2\right) + \sum_{\ell=1}^d \frac{\sigma_{1,\ell}^{-2}+\sigma_{2,\ell}^{-2}}{2}(\mu_{2,\ell}-\mu_{1,\ell})^2,  \\
\leq& \frac12 \sum_{\ell=1}^d \left(\frac{\sigma^4_{2,\ell}+\sigma^4_{1,\ell}-2\sigma^2_{1,\ell}\sigma^2_{2,\ell}}{\sigma^2_{1,\ell}\sigma^2_{2,\ell}} \right)+ \frac{1}{\sigma^2_{\min}}\|\mmu_1-\mmu_2\|_2^2,  \\
\leq& \frac{1}{2\sigma_{\min}^4}\sum_{\ell=1}^d\left(\sigma^2_{1,\ell}-\sigma^2_{2,\ell}\right)^2+ \frac{1}{\sigma^2_{\min}}\|\mmu_1-\mmu_2\|_2^2,  \\
\leq& \frac{1}{2\sigma_{\min}^4}\|\msigma_2-\msigma_1\|_2^2 + \frac{1}{\sigma^2_{\min}}\|\mmu_1-\mmu_2\|_2^2 \leq L^2\|\mtheta_1-\mtheta_2\|^2_2,
\end{align*}
where $L:=\max\left(\sigma_{\min}^{-1},\sigma_{\min}^{-2}/\sqrt{2}\right)$. Hence
\begin{equation}\label{eq:lipg}
\|\pp_{\mtheta_1}-\pp_{\mtheta_2}\|_{TV} \leq L\|\mtheta_1-\mtheta_2\|_2,
\end{equation}
and the embedding $\embd$ is $L$-Lipschitz. Hence $\gaussset$ is the continuous image of a compact set and is compact.

Since $\embd$ is also surjective, we can apply Lemma \ref{lem:covnumlipschitz} and conclude: denote $\mathcal{B} \subseteq \mathbb{R}^{2d}$ a ball of radius $\rad{\thetaset}$ for the Euclidean norm such that $\thetaset \subseteq \mathcal{B}$, and we have

\begin{equation}
N_{\gaussset,\|.\|_{TV}}(\epsilon) \stackrel{\text{Lemma \ref{lem:covnumlipschitz}}}{\leq} N_{\thetaset,\|.\|_2}(\epsilon/L) \stackrel{\text{Lemma \ref{lem:covnumsub}}}{\leq} N_{\mathcal{B},\|.\|_2}\left(\frac{\epsilon}{2L}\right) \stackrel{\text{Lemma \ref{lem:covnumball}}}{\leq} \left(\frac{8L\rad{\thetaset}}{\epsilon}\right)^{2d} = \left(\frac{B}{\epsilon}\right)^{2d}, \label{eq:covnumgauss}
\end{equation}
which is the desired result.

\end{proof}

\begin{proof}[Proof of Corollary \ref{cor:covnumgmm}]
Combining Theorem \ref{thm:covnum} and Lemma \ref{lem:covnummix} proves that the set of GMMs $\gaussmix{K}$ is compact. Furthermore, we obtain the following bound, for all $0<\tau<1$:
\begin{equation*}
N_{\gaussmix{K},\|.\|_{TV}}(\epsilon) \leq \frac{B^{2dK}8^K}{\tau^{2dK}(1-\tau)^K \epsilon^{(2d+1)K}},
\end{equation*}
where $B$ is defined as in Theorem \ref{thm:covnum}. By choosing\footnote{Note that the choice of $\tau$ is not optimal (indeed the minimum is attained for $\tau=\frac{2d}{2d+1}$), however we choose this value for the simplicity and clarity of the resulting bound.} $\tau=\frac{B}{B+1}$, we obtain
\begin{equation*}
N_{\gaussmix{K},\|.\|_{TV}}(\epsilon) \leq \left(\frac{B+1}{\epsilon}\right)^{(2d+1)K}8^K \leq \left(\frac{2(B+1)}{\epsilon}\right)^{(2d+1)K},
\end{equation*}
since $8^{\frac{1}{2d+1}} \leq 8^{1/3} = 2$.
\end{proof}

\subsection{Domination between metrics on the set of Gaussians}

\newcommand{\four}{\mathcal{F}}
\newcommand{\invfour}{\mathcal{F}^{-1}}

In this section, we aim at proving Theorem \ref{thm:toy}. We begin by an intermediate result.

\begin{lemma}\label{lem:domination}
Suppose that $\thetaset \subset \mathbb{R}^{2d}$ is such that $\|\mmu\|_2\leq M$ and $0<\sigma_{\min}^2\leq \sigma^2_i \leq \sigma_{\max}^2$ for all $[\mmu,\msigma]\in\thetaset$. For all $\pp_{\mtheta_1},~\pp_{\mtheta_2} \in \gaussset$,
\begin{equation}
\|\mtheta_1-\mtheta_2\|_2^2\leq D'\|\dens_{\mtheta_1}-\dens_{\mtheta_2}\|_{L^2(\mathbb{R}^d)}^2
\end{equation}
where
\[
D':=\frac{(4\pi\sigma_{\max}^2)^{d/2+1}D_1}{\pi(1-e^{-D_1})} \text{ with } D_1=\frac{M^2}{\sigma_{\min}^2} +  \frac{d}{2}\log\left(\frac{\sigma_{\max}^2}{\sigma_{\min}^2}\right)
\]
\end{lemma}

\begin{proof}
We use a property from \cite{Ahrendt2005} on product of Gaussians:
\begin{equation}
\int \mathcal{N}(\bfx;\mmu_a,\mSigma_a)\mathcal{N}(\bfx;\mmu_b,\mSigma_b)d\bfx=\frac{1}{(2\pi)^{d/2}|\mSigma_a+\mSigma_b|^{1/2}}\exp\left(-\frac{1}{2}(\mmu_a-\mmu_b)^T(\mSigma_a+\mSigma_b)^{-1}(\mmu_a-\mmu_b)\right)
\end{equation}

Hence we have
\begin{align}
\|\dens_{\mtheta_1}-\dens_{\mtheta_2}\|_{L^2(\mathbb{R}^d)}^2&=\int (\dens_{\mtheta_1}(\bfx)-\dens_{\mtheta_2}(\bfx))^2d\bfx \notag \\
&=\int \dens_{\mtheta_1}(\bfx)^2d\bfx+\int \dens_{\mtheta_2}(\bfx)^2d\bfx -2\int \dens_{\mtheta_1}(\bfx)\dens_{\mtheta_2}(\bfx)d\bfx \notag \\
&=\frac{1}{(2\pi)^{d/2}}\left[|2\mSigma_1|^{-\frac{1}{2}}+|2\mSigma_2|^{-\frac{1}{2}}-2|\mSigma_1+\mSigma_2|^{-\frac{1}{2}}e^{-\frac12 (\mmu_1-\mmu_2)^T(\mSigma_1+\mSigma_2)^{-1}(\mmu_1-\mmu_2)}\right] \notag \\
&=\frac{|2\mSigma_1|^{-\frac{1}{2}}+|2\mSigma_2|^{-\frac{1}{2}}}{(2\pi)^{d/2}}\left[1-e^{-\left(\frac12 (\mmu_1-\mmu_2)^T(\mSigma_1+\mSigma_2)^{-1}(\mmu_1-\mmu_2)+\log\left(\frac{|2\mSigma_1|^{-\frac{1}{2}}+|2\mSigma_2|^{-\frac{1}{2}}}{2|\mSigma_1+\mSigma_2|^{-\frac{1}{2}}}\right)\right)}\right] \notag \\
&\geq \frac{2}{(4\pi\sigma_{\max}^2)^{d/2}}\left[1-e^{-\left(\frac12 (\mmu_1-\mmu_2)^T(\mSigma_1+\mSigma_2)^{-1}(\mmu_1-\mmu_2)+\log\left(\frac{|2\mSigma_1|^{-\frac{1}{2}}+|2\mSigma_2|^{-\frac{1}{2}}}{2|\mSigma_1+\mSigma_2|^{-\frac{1}{2}}}\right)\right)}\right] \label{eq:toy_gauss_inter}
\end{align}

On the one hand, we have
\begin{equation}\label{eq:toy_gauss_mu_bound}
0 \leq \frac{1}{4\sigma^2_{\max}}\|\mmu_1-\mmu_2\|^2_2 \leq \frac12 (\mmu_1-\mmu_2)^T(\mSigma_1+\mSigma_2)^{-1}(\mmu_1-\mmu_2) \leq \frac12 \frac{4M^2}{2\sigma_{\min}^2}  = \frac{M^2}{\sigma_{\min}^2} 
\end{equation}

On the other hand,
\begin{align*}
\log\left(\frac{|2\mSigma_1|^{-\frac{1}{2}}+|2\mSigma_2|^{-\frac{1}{2}}}{2|\mSigma_1+\mSigma_2|^{-\frac{1}{2}}}\right)&=\log\left(\frac{|2\mSigma_1|^{-\frac{1}{2}}+|2\mSigma_2|^{-\frac{1}{2}}}{2}\right) + \frac12 \log|\mSigma_1+\mSigma_2| \\
\eqcomment{by concavity of the log}&\geq \frac12\log|2\mSigma_1|^{-\frac{1}{2}}+\frac12 \log|2\mSigma_2|^{-\frac{1}{2}} + \frac12 \log|\mSigma_1+\mSigma_2| \\
&=\frac{1}{4}\left(\sum_{\ell=1}^d \log(\sigma_{1,\ell}^2+\sigma_{2,\ell}^2)^2-\sum_{\ell=1}^d \log(2\sigma_{1,\ell}^2) -\sum_{\ell=1}^d \log(2\sigma_{2,\ell}^2)\right) \\
&=\frac{1}{4}\sum_{\ell=1}^d \log\left(\frac{(\sigma_{1,\ell}^2+\sigma_{2,\ell}^2)^2}{4\sigma_{1,\ell}^2\sigma_{2,\ell}^2}\right)\\
\eqcomment{since $-\log(x)\geq 1-x$} &\geq \frac{1}{4}\sum_{\ell=1}^d\left(1-\frac{4\sigma_{1,\ell}^2\sigma_{2,\ell}^2}{(\sigma_{1,\ell}^2+\sigma_{2,\ell}^2)^2}\right)\\
&= \sum_{\ell=1}^d \frac{(\sigma_{1,\ell}^2-\sigma_{2,\ell}^2)^2}{4(\sigma_{1,\ell}^2+\sigma_{2,\ell}^2)} \geq \frac{1}{8\sigma_{\max}^2}\|\msigma_1-\msigma_2\|_2^2
\end{align*}
and therefore

\begin{equation}\label{eq:toy_gauss_sig_bound}
0 \leq \frac{1}{8\sigma_{\max}^2}\|\msigma_1-\msigma_2\|_2^2 \leq \log\left(\frac{|2\mSigma_1|^{-\frac{1}{2}}+|2\mSigma_2|^{-\frac{1}{2}}}{2|\mSigma_1+\mSigma_2|^{-\frac{1}{2}}}\right) \leq \log\left(\left(\frac{2\sigma_{\max}^2}{2\sigma_{\min}^2}\right)^{d/2}\right) = \frac{d}{2}\log\left(\frac{\sigma_{\max}^2}{\sigma_{\min}^2}\right).
\end{equation}

We can now bound 
\begin{equation}\label{eq:toy_gauss_bounds}
0\leq \frac{1}{8\sigma_{\max}^2}\|\mtheta_1-\mtheta_2\|_2^2 \leq \left[\frac12 (\mmu_1-\mmu_2)^T(\mSigma_1+\mSigma_2)^{-1}(\mmu_1-\mmu_2)+\log\left(\frac{|2\mSigma_1|^{-\frac{1}{2}}+|2\mSigma_2|^{-\frac{1}{2}}}{2|\mSigma_1+\mSigma_2|^{-\frac{1}{2}}}\right)\right] \leq D_1,
\end{equation}
where
\begin{equation*}
D_1=\frac{M^2}{\sigma_{\min}^2} +  \frac{d}{2}\log\left(\frac{\sigma_{\max}^2}{\sigma_{\min}^2}\right).
\end{equation*}

By concavity of $x\mapsto 1-e^{-x}$, the function $\varphi: x\mapsto \frac{1-e^{-x}}{x}$ is decreasing. Hence we have:
\begin{equation*}
\forall x\in[0;D_1],~1-e^{-x}\geq \frac{1-e^{-D_1}}{D_1}x.
\end{equation*}
Therefore, given \eqref{eq:toy_gauss_inter} and \eqref{eq:toy_gauss_bounds}, we have
\begin{equation*}
\|\dens_{\mtheta_1}-\dens_{\mtheta_2}\|_{L^2(\mathbb{R}^d)}^2 \geq D_2 \left(\frac12 (\mmu_1-\mmu_2)^T(\mSigma_1+\mSigma_2)^{-1}(\mmu_1-\mmu_2)+\log\left(\frac{|2\mSigma_1|^{-\frac{1}{2}}+|2\mSigma_2|^{-\frac{1}{2}}}{2|\mSigma_1+\mSigma_2|^{-\frac{1}{2}}}\right)\right),
\end{equation*}
where $D_2:=\frac{2(1-e^{-D_1})}{D_1(4\pi\sigma_{\max}^2)^{d/2}}$.

And we have, using \eqref{eq:toy_gauss_bounds} again:
\begin{align*}
\|\dens_{\mtheta_1}-\dens_{\mtheta_2}\|_{L^2(\mathbb{R}^d)}^2 \geq \frac{D_2}{8\sigma_{\max}^2}\|\mtheta_1-\mtheta_2\|_2^2,
\end{align*}
which leads to the desired result.
\end{proof}


\begin{proof}[Proof of Theorem \ref{thm:toy}]
We denote $\four$ and $\invfour$ the Fourier and inverse Fourier transform:
\begin{align*}
\four(f)(\freq)&=\int_{\mathbb{R}^d}e^{-\imaginaryi \freq^T\bfx}f(\bfx)d\bfx, \\
\invfour(F)(\bfx)&=\frac{1}{(2\pi)^d}\int_{\mathbb{R}^d}e^{\imaginaryi \freq^T\bfx}F(\freq)d\freq.
\end{align*}

We recall the classical Plancherel's Theorem:
\begin{equation}
\label{eq:plancherel}
\|f\|_2^2=\frac{1}{(2\pi)^d}\|\four(f)\|_2^2.
\end{equation}

Let $\pp_1,\pp_2\in \gaussset$. Recall that $\freqdist=\mathcal{N}\left(0,\frac{\sigkersmall}{n}\bfI\right)$, denote $\sigmaker^2:=\frac{\sigkersmall}{d}$. The MMD is expressed:
\begin{align*}
\normKsq{\pp_1}{\pp_2}&=\int |\chrc_1(\freq)-\chrc_2(\freq)|^2\mathcal{N}(\freq;0,\sigmaker^2\bfI)d\freq \\
&=\int |\chrc_1(\freq)-\chrc_2(\freq)|^2\frac{1}{(2\pi\sigmaker^2)^{d/2}}e^{-\frac{\|\freq\|_2^2}{2\sigmaker^2}}d\freq \\
&=\frac{1}{(2\pi\sigmaker^2)^{d/2}}\int \left\lvert(e^{-\imaginaryi\freq^T\mmu_1}e^{-\frac{1}{2}\freq^T\left(\mSigma_1+\frac{\bfI}{2\sigmaker^2}\right)\freq}-e^{-\imaginaryi\freq^T\mmu_2}e^{-\frac{1}{2}\freq^T\left(\mSigma_2+\frac{\bfI}{2\sigmaker^2}\right)\freq}\right\rvert^2d\freq \\
&=\frac{1}{(2\pi\sigmaker^2)^{d/2}}\int \left\lvert\chrc_{\pp'_1}(\freq)-\chrc_{\pp'_2}(\freq)\right\rvert^2d\freq,
\end{align*}
where $\pp'_i$ is a Gaussian with the same mean than $\pp_i$ and dilated variance $\mSigma_i'=\mSigma_i+\frac{\bfI}{2\sigmaker^2}$. Since $\chrc_{\pp'_i}=\four(\dens'_i)$, by Plancherel's Theorem we have:
\begin{equation*}
\normKsq{\pp_1}{\pp_2}=\left(\frac{2\pi}{\sigmaker^2}\right)^{d/2}\int \left\lvert \dens'_1(\bfx)-\dens'_2(\bfx)\right\rvert^2d\bfx=\left(\frac{2\pi}{\sigmaker^2}\right)^{d/2}\|\dens'_1-\dens'_2\|_2^2.
\end{equation*}

The parameters of the Gaussians $\pp'_i$ belong to a compact set $\thetaset'=\left\lbrace\left(\mmu,\mSigma+\frac{\bfI}{2\sigmaker^2}\right); (\mmu,\mSigma)\in \thetaset\right\rbrace$. We can therefore apply Lemma \ref{lem:domination}, such that:
\begin{equation}
D'\|\dens'_1-\dens'_2\|_2^2 \geq \|\mtheta'_1-\mtheta'_2\|_2^2=\|\mtheta_1-\mtheta_2\|_2^2.
\end{equation}
The last equality comes from the fact that the variance between $\mtheta_i$ and $\mtheta'_i$ are just translated. The constant $D'$ is:
\begin{equation*}
D'=\frac{(4\pi(\sigma_{\max}^2+1/(2\sigmaker^2)))^{d/2+1}D_1}{\pi(1-e^{-D_1})} \text{ with } D_1=\frac{M^2}{\sigma_{\min}^2+1/(2\sigmaker^2)} +  \frac{d}{2}\log\left(\frac{\sigma_{\max}^2+1/(2\sigmaker^2)}{\sigma_{\min}^2+1/(2\sigmaker^2)}\right)
\end{equation*}
and thus
\begin{equation*}
\left(\frac{\sigmaker^2}{2\pi}\right)^{d/4}D'\normK{\pp_1}{\pp_2}\geq \|\mtheta_1-\mtheta_2\|_2.
\end{equation*}

Considering eq. \eqref{eq:lipg} in the proof of Theorem \ref{thm:covnum}, we have
\begin{equation}
\|\pp_1-\pp_2\|_{TV}\leq \max(\sigma_{\min}^{-1},\sigma_{\min}^{-2}/\sqrt{2})\|\mtheta_1-\mtheta_2\|_2 \leq \bar D\normK{\pp_1}{\pp_2},
\end{equation}
with $\bar D:=\max(\sigma_{\min}^{-1},\sigma_{\min}^{-2}/\sqrt{2})\left(\frac{\sigmaker^2}{2\pi}\right)^{d/4}D'=\max(\sigma_{\min}^{-1},\sigma_{\min}^{-2}/\sqrt{2})\sqrt{\frac{2D_1\cdot (2\sigmaker^2\sigma_{\max}^2+1)^{d/2+1}}{\sigmaker^2(1-e^{-D_1})}}$.

We now use the fact that $\sigmaker^2=\frac{\sigkersmall}{d}$. We have
\begin{align}
D_1=&\frac{M^2}{\sigma_{\min}^2+d/(2\sigkersmall)} +  \frac{d}{2}\log\left(\frac{2\sigkersmall\sigma_{\max}^2/d+1}{2\sigkersmall\sigma_{\min}^2/d+1}\right) \notag\\
\leq & \frac{M^2}{\sigma_{\min}^2+d/(2\sigkersmall)} +  \frac{d}{2}\log\left(2\sigkersmall\sigma_{\max}^2/d+1\right) \notag\\
\eqcomment{since $\log(1+x)\leq x$}\leq& \frac{M^2}{\sigma_{\min}^2+d/(2\sigkersmall)} +  \frac{d}{2}2\sigkersmall\sigma_{\max}^2/d = \frac{M^2}{\sigma_{\min}^2+d/(2\sigkersmall)} +  \sigkersmall\sigma_{\max}^2\notag \\
\leq& \sigma_{\max}^2\sigkersmall\left(1+\frac{2M^2}{d}\right) := D_2 \label{eq:final1}
\end{align}
and similarly 
\begin{align}
(2\sigmaker^2\sigma_{\max}^2+1)^{d/2+1}=&(2\sigkersmall\sigma_{\max}^2/d+1)\exp\left(\frac{d}{2}\log\left(2\sigkersmall\sigma_{\max}^2/d+1\right)\right) \notag\\
\leq& (2\sigkersmall\sigma_{\max}^2/d+1)e^{\sigkersmall\sigma_{\max}^2} \notag \\
\leq& e^{\sigkersmall\sigma_{\max}^2\left(1+\frac{1}{d}\right)} \leq e^{3\sigkersmall\sigma_{\max}^2}. \label{eq:final2}
\end{align}
Using \eqref{eq:final1} with the fact that the function $\varphi:x\mapsto \frac{1-e^{-x}}{x}$ is decreasing, and \eqref{eq:final2}, we obtain
\begin{equation}
\|\pp_1-\pp_2\|_{TV} \leq D\normK{\pp_1}{\pp_2},
\end{equation}
with $D:=\max(\sigma_{\min}^{-1},\sigma_{\min}^{-2}/\sqrt{2})\sqrt{\frac{2dD_2\cdot e^{3\sigkersmall\sigma_{\max}^2}}{\sigkersmall(1-e^{-D_2})}}$.
\end{proof}

\end{document}